\documentclass[dvipsnames,svgnames]{article} %
\usepackage{iclr2025_conference,times}

\usepackage[utf8]{inputenc} %
\usepackage[T1]{fontenc}    %
\usepackage[colorlinks = true,
            linkcolor = blue,
            urlcolor  = blue,
            citecolor = blue]{hyperref}       %
\usepackage{url}            %
\usepackage{booktabs}       %
\usepackage{amsfonts}       %
\usepackage{nicefrac}       %
\usepackage{microtype}      %
\usepackage{xcolor}  
\usepackage{amsmath}
\usepackage{amssymb}
\usepackage{amsthm}
\usepackage{graphicx}
\usepackage{caption}
\usepackage{subcaption}
\usepackage{relsize}
\usepackage{mathrsfs}
\usepackage{multicol}
\usepackage{wrapfig}
\usepackage{sidecap}
\usepackage{selectp}
\usepackage{comment}
\usepackage{enumitem}

\allowdisplaybreaks %

\newcommand{\diff}{\mathrm{d}}

\newcommand{\cR}{{\mathcal{R}}}
\newcommand{\cM}{{\mathcal{M}}}
\newcommand{\cN}{{\mathcal{N}}}

\newcommand{\R}{{\mathbb{R}}}
\newcommand{\N}{{\mathbb{N}}}

\newcommand{\Var}{\text{Var}}

\newcommand{\Prob}{\mathbb{P}}
\newcommand{\Esp}{\mathbb{E}}
\newcommand{\erf}{\mathrm{erf}}
\newcommand{\length}{L}
\renewcommand{\leq}{\leqslant}
\renewcommand{\geq}{\geqslant}

\definecolor{ForestGreen}{cmyk}{0.91,0,0.88,0.12}
\colorlet{pierrem}{ForestGreen}

\DeclareMathOperator*{\argmax}{arg\,max}
\DeclareMathOperator*{\argmin}{arg\,min}

\DeclareMathOperator*{\Cov}{Cov}

\DeclareMathAlphabet{\mathbbb}{U}{bbold}{m}{n}

\newtheorem{theorem}{Theorem}
\newtheorem{corollary}[theorem]{Corollary}
\newtheorem{lemma}[theorem]{Lemma}
\newtheorem{proposition}[theorem]{Proposition}
\newtheorem{definition}{Definition}
\newtheorem{assumption}{Assumption}

\usepackage[many]{tcolorbox}  
\definecolor{main}{HTML}{5989cf}    %
\definecolor{sub}{HTML}{cde4ff}     %

\tcbset{
    sharp corners,
    colback = white,
    before skip = 0.2cm,    %
    after skip = 0.5cm      %
}                           %

\newtcolorbox{boxH}{
    colback = sub, 
    colframe = main, 
    boxrule = 0pt, 
    leftrule = 6pt %
}
\newtcolorbox{boxI}{
    colback = sub, 
    colframe = main, 
    boxrule = 0pt, 
    toprule = 6pt %
}

\usepackage{ulem}

\title{Attention~Layers~Provably~Solve \\Single-Location Regression}

\author{Pierre Marion \\ %
Institute of Mathematics\\
EPFL\\
Lausanne, Switzerland \\
\texttt{pierre.marion@epfl.ch} \\
\And
Rapha\"el Berthier \\
Sorbonne Université, Inria \\
Centre Inria de Sorbonne Université \\
Paris, France \\
\texttt{raphael.berthier@inria.fr} \\
\AND
Gérard Biau \\
Sorbonne Université \\
Institut Universitaire de France \\
Paris, France \\
\texttt{gerard.biau@upmc.fr} \\
\And
Claire Boyer \\
Université Paris-Saclay \\
Institut Universitaire de France \\
Orsay, France \\
}

\iclrfinalcopy %
\begin{document}

\maketitle

\begin{abstract}
Attention-based models, such as Transformer, excel across various tasks but lack a comprehensive theoretical understanding, especially regarding token-wise sparsity and internal linear representations. To address this gap, we introduce the \textit{single-location regression} task, where only one token in a sequence determines the output, and its position is a latent random variable, retrievable via a linear projection of the input. To solve this task, we propose a dedicated predictor, which turns out to be a simplified version of a non-linear self-attention layer. We study its theoretical properties, by showing its asymptotic Bayes optimality and analyzing its training dynamics. In particular, despite the non-convex nature of the problem, the predictor effectively learns the underlying structure. This work highlights the capacity of attention mechanisms to handle sparse token information and internal linear structures.
\end{abstract}

\section{Introduction}

Attention-based models \citep{bahdanau2015neural}, such as Transformer \citep{vaswani2017attention}, have achieved unprecedented performance in various learning tasks, including natural language processing (NLP), e.g., text generation \citep{bubeck2023sparks}, translation \citep{luong-etal-2015-effective}, sentiment analysis \citep{song2019attentional,sun-etal-2019-utilizing,xu-etal-2019-bert}, and audio/speech analysis \citep{bahdanau2016endtoend}. These developments have led to many architectural and algorithmic variants of attention-based models \citep[see the review by][]{lin2022survey}. At a high level, the success of attention has been linked to its ability to manage long-range dependencies in input sequences \citep{bahdanau2015neural,vaswani2017attention}, since it consists in computing pairwise dependence between input tokens according to their projection in learned directions, independently of their location in the sequence. 

On the theoretical front, however, a deeper understanding of attention-based neural networks is still in its infancy. This limited progress is due both to the complexity of the architectures and to the disturbing diversity of relevant tasks. A common approach to tackle these challenges is to introduce a simplified task that models certain features of real-world tasks, followed by demonstrating a simplified version of the attention mechanism capable of solving the task. Prominent examples of this pattern include studying in-context learning with linearized attention \citep{ahn2023transformers,vonoswald2023transformers,zhang2024trained}, topic understanding with single-layer attention and alternate minimization scheme \citep{li2023how}, learning spatial structure with positional attention \citep{jelassi2022vision}, next-token prediction with latent bigram \citep{bietti2023birth,tian2023scan} or causal graph \citep{nichani2024how} structures, and sparse token selection \citep{wang2024transformers}. We refer to Appendix \ref{app:related-works} for additional discussion on some of these related works.
While these works shed light on some abilities of Transformer, they do not encompass all the characteristics of tasks where Transformer performs well, in particular in NLP. Two features of particular interest, which to our knowledge have not been addressed in previous theoretical studies on Transformer, are token-wise sparsity, where relevant information is contained in a limited number of tokens, and internal linear representations, which are interpretable representations of the input constructed by the model.

\paragraph{Contributions.}
To understand why attention is a suitable architecture for addressing these features, we introduce \textit{single-location regression}, a novel statistical task where attention-based predictors excel (Section \ref{sec:regressiontask}). In a nutshell, this task is a regression problem with a sequence of tokens as input. The key novelty is that only one token determines the prediction, and the location of this token is a latent random variable that changes based on the input sequence. Consequently, solving the task requires first identifying the location of the relevant token, which can be done by learning a latent linear projection, followed by performing regression on that token.

To tackle this problem, we propose a dedicated predictor, which turns out to be a simplified version of a non-linear self-attention layer. We show that this attention-based predictor is asymptotically Bayes optimal, whereas more standard linear regressors fail to perform better than the null predictor.
We then analyze the training dynamics of the proposed predictor, when trained to minimize the theoretical risk by projected gradient descent. Despite the non-convexity of the problem and the non-linearity of this transformer-based method, we show that the learned predictor successfully retrieves the underlying structure of the task and thus solves single-location regression. 

\paragraph{Organization.}
Section \ref{sec:regressiontask} presents the mathematical framework of single-location regression, followed by motivations from language processing.
Section \ref{sec:solving} is dedicated to defining our predictor and explaining its connection with attention. We then move on to the mathematical study, from both statistical (Section \ref{sec:risk-optimal}) and optimization (Section \ref{sec:optim}) points of view. Section \ref{sec:conclusion} concludes the paper.

\section{Single-location regression task}
\label{sec:regressiontask}
In this section, we describe our statistical task, and connect it to language processing motivations.

\subsection{Statistical setting} 
We consider a regression scenario where the inputs are sequences of $L$ random \textit{tokens}\footnote{For the sake of simplicity, we interchangeably use the terms ``token'' and ``embedding'', although they have slightly different meanings in the NLP community.} 
$(X_1, \hdots, X_L)$ taking values in $\mathbb{R}^d$.
The output $Y \in \mathbb{R}$ is assumed to be given by
\begin{align}
   Y &= X_{J_0}^\top v^\star + \xi,  
   \tag{$P_\mathrm{learn}$}
   \label{pb:learning_pb}
\end{align}
where $J_0$ is a latent discrete random variable on $\{1, \hdots, L\}$ and,
conditionally on $J_0$,
\[
\left\{
\begin{array}{lll}
    X_{J_0} &\sim&\cN\Big(\sqrt{\frac{d}{2}}k^\star, \gamma^2 I_d\Big)   \\
    X_\ell &\sim& \cN(0, I_d) \quad \textnormal{for} \quad \ell \neq J_0 \, .
\end{array}
\right.
\]
In the above formulation, $\cN(\mu, \Sigma)$ denotes the normal distribution with expectation $\mu$ and covariance matrix $\Sigma$, and $I_d$ is the identity matrix of size $d \times d$. All vectors are considered as column matrices, and the noise term $\xi$ is assumed to be a centered random variable independent of~$X$ and~$J_0$, with finite second-order moment $\varepsilon^2$. Conditionally on $J_0$, the tokens $(X_j)_{1 \leq j \leq L}$ are assumed to be independent.

The parameters of the regression problem \eqref{pb:learning_pb} are the \textit{unknown} vectors $k^\star$ and $v^\star$, both assumed to be on the unit sphere $\mathbb{S}^{d-1}$ in dimension~$d$, i.e., $\|k^\star\|_2 = \|v^\star\|_2 = 1$. The output is determined by a specific token in the sentence, indexed by the discrete random variable $J_0$ on $\{1, \hdots, L\}$. This token can be detected via its mean, which is proportional to $k^\star$, contrarily to the others which have zero mean. Once $X_{J_0}$ is identified, the prediction is formed as a linear projection in the direction~$v^\star$. Therefore, the originality and difficulty of this task lies in the fact that the response $Y$ is linearly related to a single informative token $X_{J_0}$, whose location varies from sequence to sequence---in this sense, the problem is sparse, but with a random support. %

A knee-jerk reaction would be to fit a linear model to the pair $(X_1^\top,\hdots, X_L^\top,Y)$. One might also consider tackling the problem with classical statistical approaches dedicated to sparsity, such as a Lasso estimator or a group-Lasso technique \citep{hastie2009}. %
However, as we will see (in Section~\ref{sec:risk-optimal}), all linear predictors fail due to the unknown and changing location of $J_0$.
We note in addition that $\mathbb{E}[ \| X_{\ell} \|_2^2] = d$ when $\ell\neq J_0$, while  $\mathbb{E}[ \| X_{J_0} \|_2^2] = d/2 + \gamma^2 d$. Therefore, choosing $\gamma^2=1/2$ implies that
tokens have the same squared norm in expectation, whether they are discriminatory of not. This shows that any approach based on comparing the magnitude of the tokens does not yield meaningful results. Ultimately, it is necessary to implement a more sophisticated approach, capable of taking into account the characteristics of the problem.

\subsection{Language processing motivation}
The structure of the task \eqref{pb:learning_pb} is motivated by natural language processing (NLP), and more specifically by two features, token-wise sparsity and internal linear representations, as we detail next.

\begin{figure}[ht]
    \begin{subfigure}{0.49\textwidth}
    \centering
    \includegraphics[width=\textwidth]{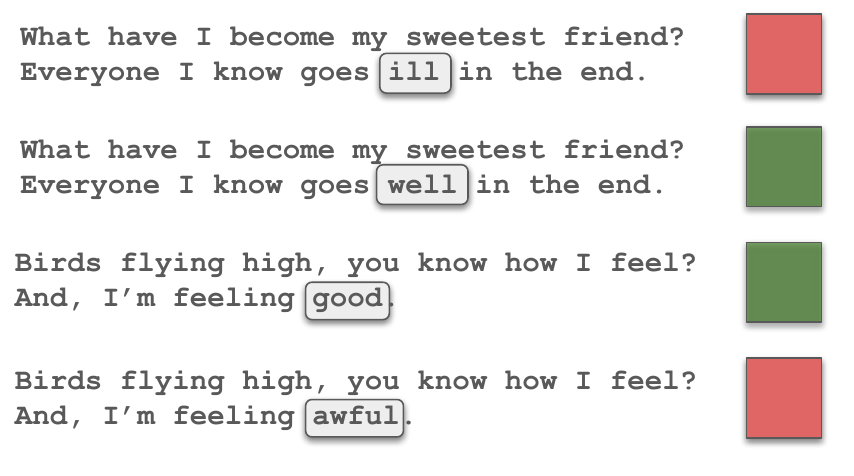}
    \vspace{5mm}
    \caption{Examples of input-output pairs. The input is a text containing two sentences (e.g., a question and an answer), and the task is to perform sentiment analysis only for the second sentence. The $Y$ output is symbolized here by a color code, where green (resp. red) corresponds to positive (resp. negative) feelings. The relevant information is sparse, typically concentrated in a single token: changing the grey token flips the output.}
    \label{fig:examples-nlp}
    \end{subfigure}
     \hfill
    \begin{subfigure}{0.49\textwidth}
     \centering
    \includegraphics[width=\textwidth]{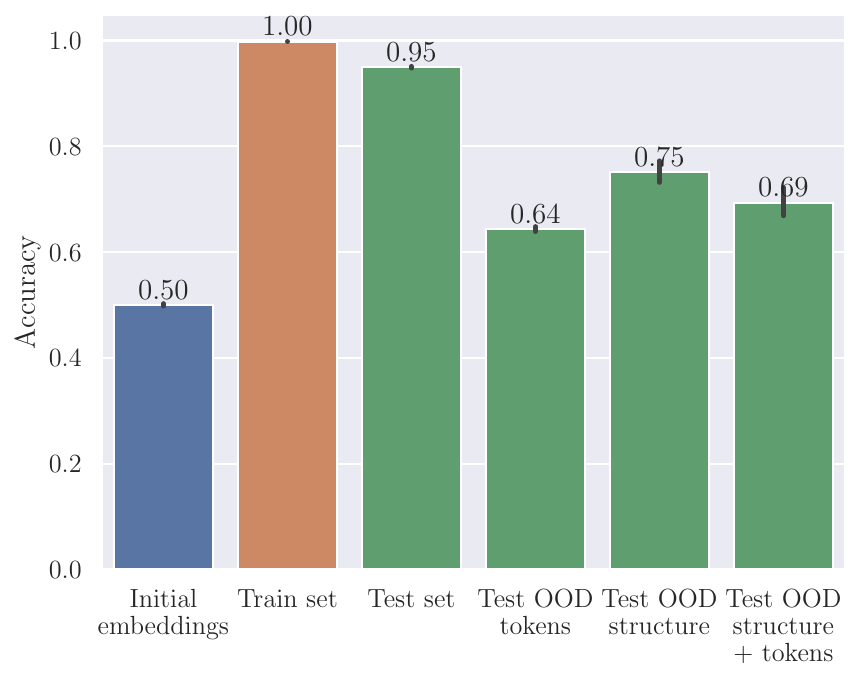}
    \caption{Accuracy of logistic regression on embeddings of [CLS] tokens in the hidden layers of a pretrained Transformer model. Initial embeddings of [CLS] (at layer $0$) are not context-aware, so they have a pure-chance accuracy of $50\%$. In hidden layers, the [CLS] token contains a representation of the sentence that achieves high scores and 
    is robust to out-of-distribution changes in token distribution and sentence structure.}
    \label{fig:exp-transformers}
    \end{subfigure}
    \caption{A simple sentiment analysis task with synthetic data, which exemplifies (a) token-wise sparsity and (b) internal linear representations. We refer to Appendix \ref{app:experimental-details} for details on the experiment.}
    \label{fig1}
\end{figure}

\paragraph{Token-wise sparsity.} In language tasks, the relevant information is often contained in few tokens, where we recall that tokens correspond to small text units (typically, words or subwords), which are embedded in $\R^d$ using a learned dictionary. This sparsity is revealed by the success of sparse attention \citep{martins2016from,blondel2017regularized,correia-etal-2019-adaptively,child2019generating,jasczcur2021sparse,kim2022learned,farina2024sparsity}, which is competitive with full attention while attending to fewer tokens. 
As an illustration, we consider a simple sentiment analysis task in Figure \ref{fig:examples-nlp}, and observe that changing one token flips the output.
This is modeled in \eqref{pb:learning_pb} by having the output $Y$ depend on a single token~$J_0$, whose location furthermore varies with the input.

\paragraph{Internal linear representations.} Linear projections of internal representations of Transformer (a.k.a.~linear probing) contain interpretable information \citep{bolukbasi2021interpretability,burns2023discovering,li2023inference}. Such a linear structure is also present in the learned token embeddings that are fed as input to language models \citep{mikolov2013exploiting,mikolov-etal-2013-linguistic,bolukbasi2016man,nanda-etal-2023-emergent,wenyi2023hyperpolyglot}. In our task \eqref{pb:learning_pb}, the two directions $k^\star$ and~$v^\star$ have to be learned by the model in order to solve the task. Figure~\ref{fig:examples-nlp-2} gives an example of possible such directions for the toy task described above. While this illustration relies on initial embeddings, similar structures also appear in the intermediate representations of Transformer. This is shown in Figure~\ref{fig:exp-transformers}, where we observe that pretrained Transformer architectures indeed build internal representations that are sufficient to solve the task with a linear classifier.

\begin{figure}[ht]
    \centering
    \includegraphics[width=0.9\textwidth]{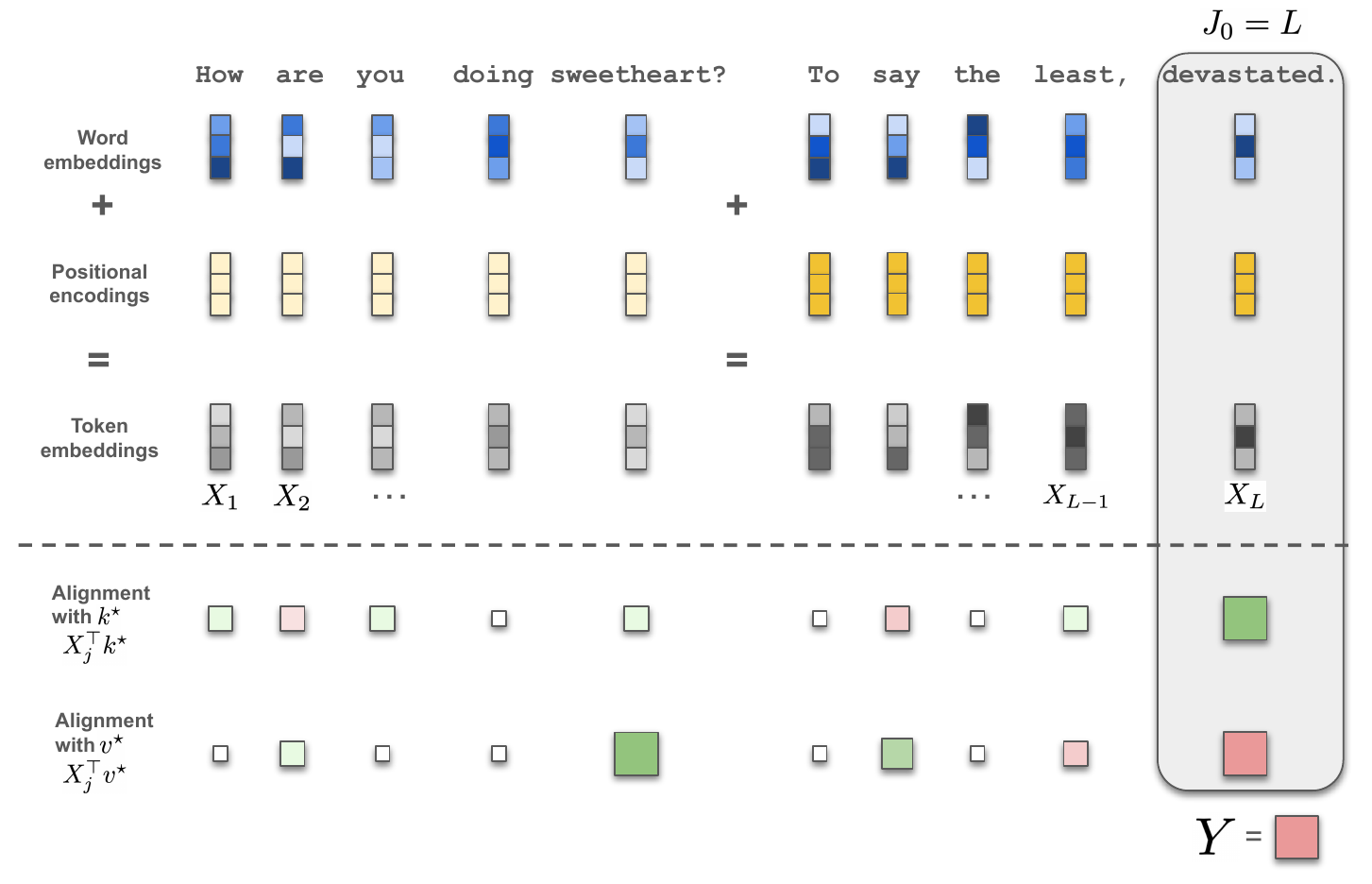}
    \caption{Modeling of an NLP task within our statistical setting \eqref{pb:learning_pb}. The token embeddings $X_1, \dots, X_L$ are constructed by adding the embeddings of each word and a positional encoding. For illustration purposes, we assume that each token corresponds to a word, and that the positional encoding solely depends on the part of the sentence (before or after the question mark), which differs from usual practice. Then, let the direction~$k^\star$ encode both the notion of sentiment and the position in the second part of the sentence. Thus only the last token of the sentence is aligned (positively) with $k^\star$, and we have $J_0=L$. As for~$v^\star$, it encodes whether the word is associated with a positive or negative sentiment. Note that several tokens are positively or negatively aligned with $v^\star$, but the output $Y$ only depends on the token $J_0$. This illustrates the interest of having two latent directions $k^\star$ and $v^\star$, one that filters the informative token and one that aligns with the output $Y$.}
    \label{fig:examples-nlp-2}
\end{figure}

We acknowledge that our statistical task presents limitations such as fixed sequence length, independent tokens, and output depending only on a single token. 
More complex models could be considered, but at significant technical cost. Moreover, as argued above, our problem \eqref{pb:learning_pb} preserves interesting aspects of NLP tasks, which makes it relevant for theoretical study of Transformer. Furthermore, it is an original statistical task requiring the implementation of a customized estimation strategy. It is precisely in this context that attention models prove their effectiveness, as we show next.

\section{An attention-based predictor to solve the regression task}
\label{sec:solving}
In this section, we propose a predictor adapted to the problem \eqref{pb:learning_pb} and discuss its connection with attention. In order to make our point as clear as possible, the construction is divided into three steps. We represent the input sequence in a matrix format $\mathbb{X} \in \R^{L \times d}$, where $\mathbb{X} = (X_1 | X_2 | \cdots | X_L)^\top$. %

\paragraph{Step 1: An oracle non-differentiable predictor.}

If the vectors $(k^\star, v^\star) \in (\mathbb{S}^{d-1})^2$ were known, then a natural procedure to solve the task \eqref{pb:learning_pb} would be to predict $Y$ from $\mathbb{X}$ via
\begin{equation}
\label{eq:true-task}
T(\mathbb{X}) = (\mathbb{X} v^\star)_{j_0(\mathbb{X})} = X_{j_0(\mathbb{X})}^\top v^\star \, , \quad \textnormal{where} \quad j_0(\mathbb{X}) = \argmax_{1 \leq \ell \leq L} \, (\mathbb{X} k^\star)_\ell %
\, .
\end{equation}
The $\argmax$ part detects the location $J_0$ by exploiting the fact that all $X_{\ell}$ have zero mean except $X_{J_0}$, while the $\mathbb{X} v^\star$ part exploits the linear relationship $Y = X_{J_0}^\top v^\star + \xi$.
In a more compact format, this ideal predictor can be rewritten as
$ %
    T(\mathbb{X})= \sum_{\ell=1}^L \mathbbb{1}_{\argmax(\mathbb{X} k^\star) = \ell} (\mathbb{X} v^\star)_\ell \, ,
$
which is a linear regression in the direction $v^\star$ with non-differentiable weights depending on $k^\star$. 

\paragraph{Step 2: A trainable predictor.} In practice, the vectors $k^{\star}$ and $v^{\star}$ are unknown and must be estimated from the data. In addition, the non-differentiability of the $\argmax$ function poses significant optimization challenges. To solve this problem, the most common approach in machine learning is to replace $\argmax$ with a softmax function with inverse temperature $\lambda > 0$, i.e., for $z=(z_1, \hdots, z_L)\in~\mathbb R^L$, $[\mathrm{softmax}(\lambda z)]_j={e^{\lambda z_j}} / {\sum_{\ell=1}^L e^{\lambda z_\ell}}$. This leads us to the model
\begin{equation}    \label{def:simplified-transformer}
   T_\lambda^{(\text{soft},k,v)}(\mathbb{X}) = \sum_{\ell=1}^L [\mathrm{softmax}(\lambda \mathbb{X} k)]_\ell (\mathbb{X} v)_\ell = \mathrm{softmax}\big(\lambda \underbrace{ \mathbb{X} k}_{L\times 1} \big)^\top  \underbrace{\mathbb{X}v}_{L\times 1}, 
\end{equation}
where $k,v \in \mathbb{S}^{d-1}$, and the superscript `soft' is used to indicate the presence of the softmax function.

\paragraph{Step 3: The final predictor.} The softmax nonlinearity, by inducing a coupling between all tokens, significantly complicates the mathematical analysis. To alleviate this difficulty, we replace it by the component-wise nonlinear function $\erf(z)=\frac{2}{\sqrt \pi}\int_0^z e^{-t^2}dt$, which is differentiable, increasing on $\R$, and such that $\erf(-\infty) = -1$ and $\erf(\infty) = 1$. We are therefore led to our operational model
\begin{align}
\label{def:our_transformer}
T_\lambda^{(k,v)}(\mathbb{X}) = \erf\big(\lambda \mathbb{X} k \big)^\top \mathbb{X}v = \sum_{\ell=1}^L \erf\big( \lambda {X}_\ell^\top k \big) X_\ell^\top v \, , 
\end{align}
where the $\erf$ function is applied component-wise. 
The choice of this activation function enables closed-form expectations for functions of Gaussian random variables (see, e.g., Lemma~\ref{lem:technical_results}). Note that the role of softmax in attention is an open question in the community. Several empirical papers investigate simplifying softmax into a component-wise nonlinearity \citep{qin2022cosformer, shen2023study, wortsman2023replacing,ramapuram2024theory}, and have observed a similar performance. These works emphasize the importance of the normalization $\lambda$ when replacing softmax, which we also find out to play an important role (see Corollary \ref{cor:oracle-asymptotic-bayes} and Section \ref{sec:optim}).

\paragraph{Connection to attention.}
It turns out that our estimation method finds a natural interpretation in terms of attention models. To see this, consider a model consisting of a single attention layer with a single head \citep{vaswani2017attention}
\begin{align}   \label{eq:def-transformer}
    T_\lambda^{(Q, K, V, O)}(\mathbb{X}) &= \mathrm{softmax}\Big(\lambda \underbrace{\mathbb{X} Q}_{L\times p} \underbrace{\vphantom{q}K^\top \mathbb{X}^\top}_{p\times L} \Big) \underbrace{\mathbb{X}V}_{L\times p} \underbrace{O^\top}_{p \times o} \, ,
\end{align}
where the dimensions $p,o\in \mathbb N^*$ are hyperparameters of the model, the softmax function is applied row by row, $Q,K,V \in\mathbb{R}^{d \times p}$ and $O \in \mathbb{R}^{o \times p}$ are the regular query, key, value, and output matrices, and $\lambda$ is usually taken to be $1/\sqrt{p}$.
In practice, the attention head is added to~$\mathbb{X}$ via a skip connection, which enforces $o = d$. 
In a nutshell, $K$ detects which tokens are relevant in the sentence, $V$ encodes the regression coefficient, and $Q$ encodes where to store the information.

In a supervised context, it is classical in practice to concatenate in first position an additional token [CLS] to the tokenized sentence $\mathbb{X}$ \citep[see, e.g.,][]{devlin2019bert}. In this context, only the first coordinate of the output is used for the prediction task. Thus, we focus on the first row of~\eqref{eq:def-transformer}, corresponding to the embedding of [CLS], namely %
\begin{align}   \label{eq:def-transformer-first-line}
    T_\lambda^{(Q, K, V, O)}(\mathbb{X})_1 = \mathrm{softmax}\left(\lambda \, a \, K^\top \mathbb{X}^\top \right) \mathbb{X} V O^\top,
\end{align}
with $a=X_{\textnormal{[CLS]}}^\top Q \in \R^{1 \times p}$, where $X_{\textnormal{[CLS]}} \in \mathbb R^d$ denotes the embedding of the [CLS] token. 

It is important to note that only considering the first output coordinate is a mathematically valid simplification for a single attention layer, but not when multiple layers are stacked, as all coordinates of the attention output contribute. Nevertheless, even in this latter more realistic case, the [CLS] token---or the similar concepts of attention sinks and registers---has been empirically shown to play a crucial role \citep{clark2019bert,darcet2024vision,xiao2024efficient}. This is also confirmed by our experiment in Figure~\ref{fig:exp-transformers}, where we show that the [CLS] token in pretrained Transformer architectures stores an internal representation of the sentence that is sufficient to solve simple NLP tasks with a linear classifier. This further motivates the need to understand how information is stored in this token.

It turns out that there is a direct connection between the model $T_\lambda^{(\text{soft},k,v)}(\mathbb{X})$ defined in \eqref{def:simplified-transformer} and the attention model $T_\lambda^{(Q, K, V, O)}(\mathbb{X})_1$ described in \eqref{eq:def-transformer-first-line}. To see this, take $o=1$, to adapt the model \eqref{eq:def-transformer-first-line} for univariate regression, and set $p=1$, a reasonable assumption given both empirical and theoretical evidence suggesting that Transformer parameter matrices are low-rank \citep{aghajanyan2021intrinsic,kajitsuka2024are}. Then, let $Q \in \R^{d \times 1}$ be any vector with positive correlation with $X_{\textnormal{[CLS]}}$ (for instance it suffices to take $Q = X_{\textnormal{[CLS]}}$), and $O = 1$. We then deduce that
\[
T^{(Q, K, V, O)} _{\lambda}(\mathbb{X})_1 = 
T_{\lambda X^{\top}_{\textnormal{[CLS]}}Q}^{(\text{soft},K,V)}(\mathbb{X}) \, .
\]
In other words, \textit{the attention layer \eqref{eq:def-transformer-first-line} matches the considered predictor in \eqref{def:simplified-transformer}} with a softmax inverse temperature proportional to the scalar product between $X_{\textnormal{[CLS]}}$ and $Q$. 
Thus, our results, in particular the study of the training dynamics in Section \ref{sec:optim}, can be seen as a model of how Transformer builds internal representations of the input during training.
This is also supported by numerical experiments showing that Transformer layers behave similarly to our predictor (see Appendix \ref{app:experimental-details}).

\section{Risk of the oracle and of the linear predictors}    \label{sec:risk-optimal}
Now that we have constructed our predictor $T_\lambda^{(k,v)}$ (see Eq.~\eqref{def:our_transformer}), a first key question is to assess its statistical performance. Recall that $k,v \in \mathbb{S}^{d-1}$ are the two parameters of the model, and their purpose is to approximate their theoretical counterparts $k^{\star}$ and $v^{\star}$ defined in \eqref{eq:true-task}. This begs in particular the question of the performance of the \textit{oracle predictor} $T_\lambda^{(k^\star,v^\star)}$. 
To answer these questions, we introduce the risk of the predictor, which is measured by the mean squared error
\begin{align}
\label{def:risk}
\cR_\lambda(k, v) &= \Esp \Big[ \Big(Y - T_\lambda^{(k,v)}(\mathbb{X}) \Big)^2 \Big] \, .
\end{align}
To proceed with the analysis, we make the following assumption.
\begin{assumption}
\label{hyp:orthogonality} The vectors $k^\star, v^\star \in  \mathbb{S}^{d-1}$ are orthogonal, i.e., $k^{\star \top}v^{\star}=0$.
\end{assumption}
This assumption is made everywhere in the remainder of the paper, even though it is not reminded explicitly at each result. It is a relatively mild assumption in a high-dimensional setting where any two independent vectors uniformly distributed on the sphere are close to being orthogonal. 
\paragraph{Oracle predictor.}Our first result characterizes the risk of the proposed transformer model  \eqref{def:our_transformer} with oracle parameters $(k^\star,v^\star)$. 
All the proofs of the paper are deferred to the Appendix.
\begin{theorem}
\label{theo:optimalrisk}
There exists a function $\mathcal{R}_\lambda^<: \R^5 \to \R$ such that, for any $(k,v) \in (\mathbb{S}^{d-1})^2$,
\begin{equation*}
\mathcal{R}_\lambda(k, v) = \mathcal{R}_\lambda^<(\kappa, \nu, \theta, \eta, \rho) \, ,
\end{equation*}
where $\kappa:= k^\top k^\star$, $\nu := v^\top v^\star$, $\theta := v^\top k^\star$, $\eta := k^\top v^\star$, and $\rho := k^\top v$. A closed-form expression of $\mathcal{R}_\lambda^<$ is given in Appendix \ref{proof:risk_on_sphere}.
In particular,
    \begin{align*}
        \mathcal{R}_\lambda(k^\star, v^\star) &= \mathcal{R}_\lambda^<(1, 1, 0, 0, 0) \\
        &= \gamma^2- 2 \gamma^2 \, \erf \bigg(\lambda \sqrt{\frac{d}{2(1+2\lambda^2\gamma^2)}} \bigg)  +  \gamma^2 \zeta\Big(\lambda \sqrt{\frac{d}{2}} ,\lambda^2 \gamma^2\Big)   +   (\length -1)  \zeta(0, \lambda^2)  + \varepsilon^2 \, ,
    \end{align*}
    where, for $t,\gamma\in \mathbb{R}$,
    \begin{equation}
    \label{eq:def-zeta}
        \zeta(t,\gamma^2) := \Esp \left[ \erf^2(t+G)\right] \, , \qquad  G\sim\mathcal{N}(0,\gamma^2) \, .
    \end{equation}
\end{theorem}
This result is fundamental for the analysis of gradient descent studied in the next section since it reduces the dimension of the dynamical system defined by the optimization dynamics. Before delving into the optimization analysis, we study below the statistical optimality of the estimator $\mathcal{R}_\lambda(k^\star, v^\star)$ and its comparison with linear regression.

\paragraph{Asymptotic Bayes optimality.} 
Let us start by observing that the Bayes risk associated with problem \eqref{pb:learning_pb} is larger than $\varepsilon^2$, which follows from elementary properties of the conditional expectation \citep[][Chapter 11]{legall2022measure}. Indeed, using the Pythagorean theorem, one easily shows that
\begin{equation}    \label{eq:bayes-min}
\Esp[(Y - \Esp[Y|X])^2] \geq \Esp[(Y - \Esp[Y|X,J_0])^2] = \Esp[\xi^2] = \varepsilon^2 \, .   
\end{equation}
Then, the following corollary to Theorem \ref{theo:optimalrisk} shows that the oracle predictor achieves the Bayes-optimal risk in the asymptotic scaling $L \ll 1/\lambda^2 \ll d$.
\begin{corollary}   \label{cor:oracle-asymptotic-bayes}
Assume a joint asymptotic scaling where $d\to\infty$ and $L = o(d)$. Taking $\lambda$ such that $\lambda \sqrt{d} \to \infty$ and $\lambda \sqrt{L} \to 0$, we have
\[
\mathcal{R}_\lambda(k^\star, v^\star) \xrightarrow{} \varepsilon^2 \, .
\]
Thus, in this asymptotic regime, the oracle predictor $T_\lambda^{(k^\star,v^\star)}$ is asymptotically Bayes optimal.
\end{corollary}
Note that Corollary \ref{cor:oracle-asymptotic-bayes} holds for any finite $L \in \N_{>0}$, but $L$ may also tend to infinity, as long as $L = o(d)$. 
Let us give an intuition on why this result holds and where the scalings of $L$ and $\lambda$ intervene. The oracle predictor can be decomposed as
\begin{align}   \label{eq:exp-risk-oracle}
T_\lambda^{(k^\star,v^\star)}(\mathbb{X}) = \underbrace{X_{J_0}^\top v^\star}_{= \Esp[Y|\mathbb{X}, J_0]} \mathrm{erf}(\underbrace{\lambda X_{J_0}^\top k^\star}_{= \Theta(\lambda \sqrt{d})}) + \sum_{j \neq J_0} \underbrace{X_j^\top v^\star}_{= \Theta(1)} \mathrm{erf}(\underbrace{\lambda X_j^\top k^\star}_{= \Theta(\lambda)}) 
\end{align}
With the scaling $\lambda \sqrt{d} \to \infty$, the argument of the first $\mathrm{erf}$ nonlinearity diverges to infinity with $d$. Thus it reaches the saturating part of $\mathrm{erf}$, so the first term in \eqref{eq:exp-risk-oracle} converges to $\Esp[Y|\mathbb{X}, J_0]$. On the other hand, the argument of the $\mathrm{erf}$ nonlinearities inside the sum are of order~$\lambda = o(1)$. Thus they are in the linear part of $\mathrm{erf}$. Therefore, the sum consists of $L-1$ independent terms, each of magnitude~$\lambda$. As a consequence, by the central limit theorem,
the whole sum is of order $\Theta(\lambda \sqrt{L})$, and we get
\begin{align*}
T_\lambda^{(k^\star,v^\star)}(\mathbb{X}) \approx \Esp [Y|\mathbb{X}, J_0] + \Theta(\lambda \sqrt{L}) \, .    
\end{align*}
Due to the scaling $\lambda \sqrt{L} \to 0$, the second term decays to zero, and the oracle predictor implements the conditional expectation of $Y$ given $\mathbb{X}$ and $J_0$. This is the best that we can hope for: the predictor succeeds in inferring the latent variable $J_0$, then gives the best possible prediction of $Y$ given $\mathbb{X}$ and $J_0$.
We also see the crucial role played by the nonlinearity of $\mathrm{erf}$, whose linear part acts for $j \neq J_0$ and saturating part for $j=J_0$. In particular, the reasoning would not hold for a linear activation.

\paragraph{Linear model.} The asymptotic optimality of our oracle predictor is particularly striking in comparison to the risk of the optimal linear predictor. More precisely, let 
\[
\beta^\star \in \argmin_{\beta \in \R^{dL}}\Esp\Big[(Y-(X_1^\top, \hdots , X_L^\top) \beta)^2\Big]
\]
be the optimal linear predictor for the regression task \eqref{pb:learning_pb}. Its associated risk is
$\mathcal{R}(\beta^\star)=\Esp\big[(Y-(X_1^\top, \hdots , X_L^\top)\beta^{\star})^2\big]$. Both the optimal predictor and its risk can be made explicit as follows.
\begin{proposition}
\label{prop:bayes_risk_linear_model}
Let $p_j = \Prob(J_0 = j)$ for $j \in \{1, \dots, L\}$. Then the optimal linear predictor is parameterized by
    $
    \beta^\star =  (b_1 v^\star, \dots,  b_L v^\star)
    $,
   with
    $b_j = \frac{\gamma^2 p_j}{1 + p_j(\gamma^2-1)}$,
    and its risk is 
\[
    \mathcal{R}(\beta^\star) = \varepsilon^2 + \gamma^2 - \gamma^4 \sum_{j=1}^L \frac{p_j^2}{1 + p_j(\gamma^2 - 1)} \, .
\]
In particular, 
\[
    \mathcal{R}(\beta^\star) \geq \varepsilon^2 + \gamma^2 - \gamma^2 (\gamma^2 + 1) \max_{j=1, \dots, L} p_j \, .
\]
\end{proposition}

This result calls for a few comments. If the number of tokens is $L=1$ or if $J_0$ is a constant location (meaning that one $p_j$ is equal to $1$ while the others are equal to $0$), then the learning problem \eqref{pb:learning_pb} corresponds to a standard linear regression. In this case, $\mathcal R(\beta^{\star})=\varepsilon^2$, and the linear predictor $(X_1, \hdots , X_L)\mapsto (X_1^\top, \hdots , X_L^\top)\beta^{\star}$ achieves the Bayes risk. 
\sidecaptionvpos{figure}{c}
\begin{SCfigure}[][ht]
\centering
\includegraphics[width=0.49\textwidth]{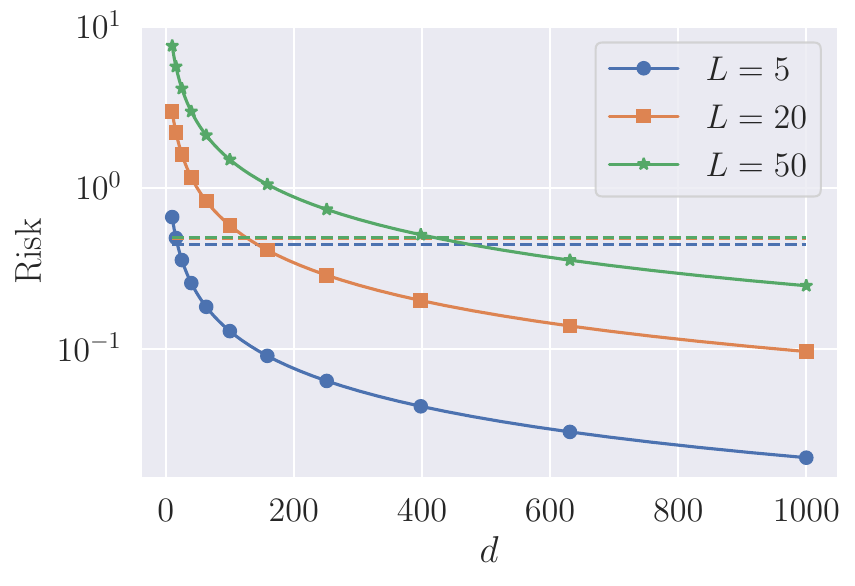}
\caption{Risk of the oracle predictor (Theorem~\ref{theo:optimalrisk}, solid lines) and of the best linear predictor (Proposition~\ref{prop:bayes_risk_linear_model}, dashed lines), depending on the dimensions $d$ and $L$. The oracle predictor outperforms the linear predictor when scaling $d$. We take $\varepsilon^2=0$, $\gamma=1/\sqrt{2}$, $\lambda=1/d^{0.4}$, and all $p_j$ equal to $1/L$.}
\label{fig:risk_comparison}
\end{SCfigure}
At the other end of the spectrum, in the case where $J_0$ is uniform over $\{1, \dots, L\}$, the formula for the risk of the linear predictor simplifies to $\mathcal{R}(\beta^\star) = \varepsilon^2 + \gamma^2 - \frac{\gamma^4}{\gamma^2+\length -1}$. When  $L\to \infty$, this risk tends to $\varepsilon^2+\gamma^2$, that is, the performance of the null predictor. In other words, the optimal linear predictor performs no better than always predicting zero. More generally, this conclusion is true in any limit where $L\to \infty$ and $\max p_j \to 0$. This can be explained by the fact that the location of the relevant token for prediction is random, varying from sentence to sentence. Unable to leverage this latent information, the linear regressor balances all its coefficients, resulting in poor prediction performance. This stands in sharp contrast to Corollary \ref{cor:oracle-asymptotic-bayes}, which shows that the oracle predictor $T_\lambda^{(k^{\star},v^{\star})}$ is able to account for the complexity of the task, at least asymptotically. This is also illustrated by Figure \ref{fig:risk_comparison}, which compares the value of the risks given by Theorem \ref{theo:optimalrisk} and Proposition \ref{prop:bayes_risk_linear_model}. 

Naturally, implementing the attention-based oracle predictor $T_\lambda^{(k^{\star},v^{\star})}$ requires knowledge of the parameters $k^{\star}$ and $v^{\star}$. Our goal in the next section is therefore to show that gradient descent is able to recover these parameters.

\section{Gradient descent provably recovers the oracle predictor}   \label{sec:optim}
This section is devoted to the analysis of the optimization dynamics in $(k,v) \in (\mathbb{S}^{d-1})^2$ of the risk
\begin{equation*}
\cR_\lambda(k, v) = \Esp \Big[ \Big(Y - T_\lambda^{(k,v)}(\mathbb{X}) \Big)^2 \Big] = \Esp \Big[ \Big(Y -\erf\big(\lambda \mathbb{X} k\big)^\top  \mathbb{X}v
\Big)^2 \Big] \, . 
\end{equation*}
We emphasize that $\cR_\lambda(k, v)$ is a theoretical risk, which depends on the distribution of the pair $(\mathbb X,Y)$ (defined in Section \ref{sec:regressiontask}). In practice, an empirical version of this risk is minimized. As we show experimentally (see Figure \ref{fig:experiment-sgd}), the stochastic dynamics induced by the empirical version of the risk are qualitatively similar to the deterministic dynamics of the theoretical risk. In the remainder of the article, we focus on the theoretical risk for simplicity, and leave the empirical risk for future research. 

Our optimization method is the Projected (Riemannian) Gradient Descent (PGD), described below.
\begin{definition}[PGD]
    Given an initialization $(k_0, v_0) \in (\mathbb{S}^{d-1})^2$, a step size $\alpha>0$, and an inverse temperature sequence $(\lambda_t)_{t \geq 0}$, the sequence $(k_t,v_t)_{t \geqslant 0} \in  (\mathbb{S}^{d-1})^2$ is recursively defined by
    \begin{align}   \label{eq:pgd-iterations}
        \begin{split}
        k_{t+1} &= \mathrm{Proj}_{\mathbb{S}^{d-1}}(k_t - \alpha (I_d-k_t k_t^\top) \nabla_k \mathcal{R}_{\lambda_t}(k_t,v_t)) %
        \, , \\
        v_{t+1} &= \mathrm{Proj}_{\mathbb{S}^{d-1}}(v_t - \alpha (I_d-v_t v_t^\top) \nabla_v \mathcal{R}_{\lambda_t}(k_t,v_t)) %
        \, ,
        \end{split}
\end{align}
where $\mathrm{Proj}_{\mathbb{S}^{d-1}}: x \mapsto x/\|x\|_2$ denotes the Euclidean projection on the unit sphere of $\R^d$.  
\end{definition}
The operators $(I_d-k_t k_t^\top)$ and $(I_d-v_t v_t^\top)$ correspond to Riemannian gradient descent \citep[][Section 4.3]{boumal2023introduction}, meaning that we compute the gradient of the risk on the Riemannian manifold $(\mathbb{S}^{d-1})^2$. In other words, the gradient step is performed on the tangent space to the sphere at the current iterate. This is a precaution we are taking because, in the analysis of the dynamics, we rely on an expression of the risk \eqref{def:risk} that is valid only on this manifold.
In addition, this ensures that the subsequent projection on $\mathbb{S}^{d-1}$ is always well-defined, despite the fact that the sphere is a non-convex set, because iterates always avoid the pathological cases $k=0$ or $v=0$.

Experimentally, we observe in Figure \ref{fig:init-sphere} that PGD is able to recover the oracle parameters $(k^\star, v^\star)$. Note that running the PGD iterates \eqref{eq:pgd-iterations} involves computing the gradients $\nabla_k \mathcal{R}_{\lambda_t}(k_t,v_t)$ and $\nabla_v \mathcal{R}_{\lambda_t}(k_t,v_t)$, which is non-trivial a priori. A direct approach using Monte Carlo simulations would require a large number of sample points to reduce variance, which is computationally intractable in particular in high-dimension, and in any case gives an approximate result. Instead, we leverage our closed form formula for $\mathcal{R}_\lambda^<$ from Theorem~\ref{theo:optimalrisk} to get exact values for the gradients  (up to numerical errors). Interestingly, we also observe in Figure \ref{fig:init-sphere} that $v$ aligns with $v^\star$ much faster than $k$ aligns with $k^\star$. This is typical of two-timescale dynamics, which is a common framework in analysis of non-convex learning dynamics \citep{heusel2017gans,dagreou2022framework,hong2023two,marion2023leveraging,berthier2024learning,marion2024implicit}.

\begin{figure}[ht]
    \begin{subfigure}[b]{0.99\textwidth}
    \centering
    \includegraphics[width=\textwidth]{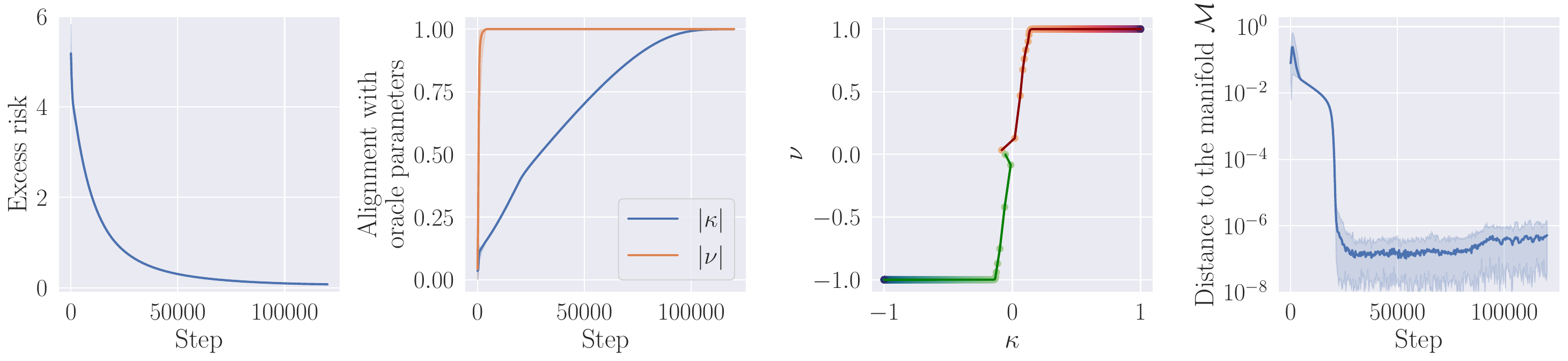}
    \caption{From a random initialization on $(\mathbb{S}^{d-1})^2$.}
    \label{fig:init-sphere}
    \end{subfigure}
     \hfill
    \begin{subfigure}[b]{0.99\textwidth}
     \centering
    \includegraphics[width=\textwidth]{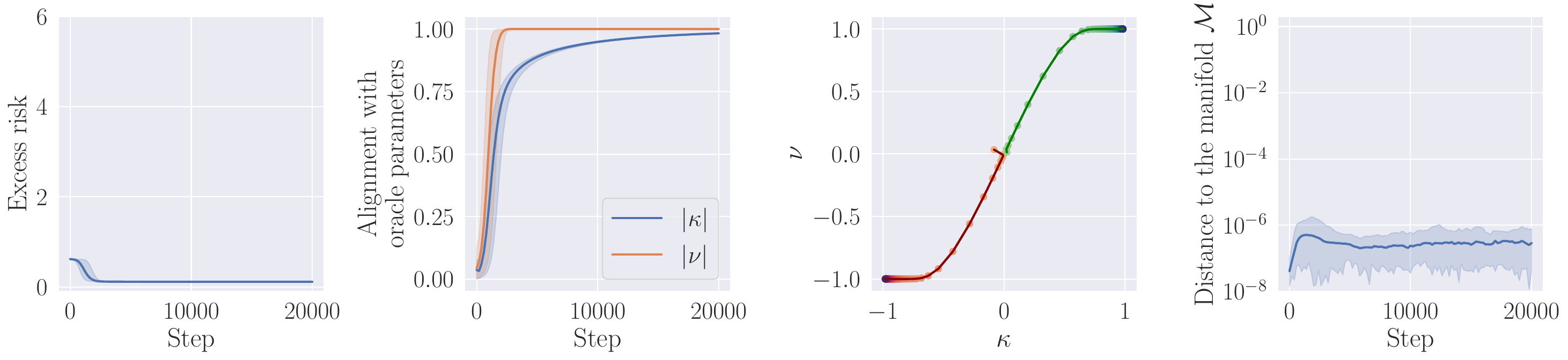}
    \caption{From a random initialization on $\cM$ (see Eq. \eqref{eq:def-M}).}
    \label{fig:init-manifold}
    \end{subfigure}
    \caption{Convergence of PGD to the oracle parameters. \textbf{Left:} Excess risk as a function of the number of steps. \textbf{Middle left:} Alignment $|\kappa| = |k^\top k^\star|$ and $|\nu| = |v^\top v^\star|$ with the oracle parameters. \textbf{Middle right:} Trajectories of $\kappa$ and $\nu$ in two repetitions of the experiments. Each repetition corresponds to a color, the trajectory starts in the middle and ends at a corner of the plot. \textbf{Right:} Distance to the invariant manifold $\cM$. In all plots except the middle right ones, the experiment is repeated $30$ times with independent random initializations, and $95\%$ percentile intervals are plotted (but are not visible when the variance is too small). Parameters are $d=400$, $L=10$, $\gamma = \sqrt{1/2}$, and (a) $\lambda_t = 1/(1 + 10^{-4}t)$, (b) $\lambda_t = 0.1$. More details are given in Appendix \ref{app:experimental-details}.}
    \label{fig:experiment-gd}
\end{figure}

Moving on to the mathematical study, even with the formula for $\mathcal{R}_\lambda^<$, a full analysis of the dynamics~\eqref{eq:pgd-iterations} is difficult. For instance, the dynamics~\eqref{eq:pgd-iterations} can be formulated in terms of the five variables of $\mathcal{R}_\lambda^<$, but then one needs to study a $5$-dimensional highly nonlinear dynamical system.
In the following, we consider the case where the parameters are initialized on the submanifold of $(\mathbb{S}^{d-1})^2$
\begin{equation}    \label{eq:def-M}
    \mathcal{M} = \{(k,v) \in \mathbb{S}^{d-1} \times \mathbb{S}^{d-1}, k^\top v^\star = 0, v^\top k^\star = 0, k^\top v = 0\} \, .
\end{equation}
We introduce this manifold on the one hand owing to the observation in Figure \ref{fig:init-sphere} (right) that the dynamics converge to this manifold even when initialized on the sphere, and on the other hand because this allows to reduce the problem to a lower-dimensional subspace and to simplify the expression of the risk. Clearly, due to Assumption \ref{hyp:orthogonality}, the oracle parameters $(k^\star, v^\star)$ belong to $\cM$. A first key property of this manifold is invariance under the PGD dynamics.
\begin{lemma}
\label{lem:invariant_manifold}
 The manifold $\mathcal{M}$ is invariant under the PGD dynamics \eqref{eq:pgd-iterations}, in the sense that if $(k_t,v_t)\in \mathcal{M}$, then $(k_{t+1},v_{t+1})\in \mathcal{M}$.
 \end{lemma}
This lemma shows that, if the initialization is taken on the manifold, then it is enough to understand the dynamics on the manifold to conclude. Such analysis on the manifold is tractable. This yields Theorem \ref{thm:main}, our main result, which shows that
the sequence $(k_t, v_t)_{t \geqslant 0}$ converges to the oracle values $(k^\star,v^\star)$ (up to a sign) as $t\to \infty$, for any small enough step size, and a constant inverse temperature.
\begin{theorem}
    \label{thm:main}
    Take a constant inverse temperature $\lambda_t \equiv \lambda > 0$. Then there exists $\overline{\alpha} >0$ such that, for any step size $\alpha \leq \overline{\alpha}$, and for a generic initialization $(k_0,v_0)\in \mathcal{M}$,
    $
    (k_t,v_t) \xrightarrow[t\to\infty]{} \pm (k^\star,v^\star).
    $
\end{theorem}
This result shows that, despite the non-convexity of the risk, the attention layer trained by PGD can recover the underlying structure of the problem. %
Convergence to $(k^\star,v^\star)$ \emph{or} $(-k^\star,-v^\star)$ is not at all problematic, since  $T_\lambda^{(k^\star,v^\star)}=T_\lambda^{(-k^\star,-v^\star)}$ by symmetry of the $\mathrm{erf}$ function. Furthermore, recovery is guaranteed for a generic initialization on $\cM$, in the sense that the pathological pairs $(k_0,v_0) \in \cM$ such that PGD fails to recover the oracle parameters are of Lebesgue measure zero.
The results of Theorem \ref{thm:main} are illustrated by Figure \ref{fig:init-manifold}. We observe that, due to roundoff errors, the dynamics are not exactly on the manifold but stay very close to the manifold.

We emphasize that the manifold $\cM$ depends on the unknown parameters $k^\star$ and $v^\star$, making it impractical to initialize directly on the manifold. If the initialization is not on $\cM$, more diverse phenomena are possible. As already pointed out in Figure \ref{fig:init-sphere}, it is possible to obtain recovery of $(k^\star,v^\star)$ and convergence to the manifold $\cM$ from a general initialization on the sphere. This suggests that our analysis on the manifold is relevant, and completing the analysis for a general initialization is left for future work. However, we note that using a decreasing inverse temperature sequence $\lambda_t$ is crucial for the recovery of $(k^\star,v^\star)$ when initialized out of $\cM$. Indeed, to the best of our experiments, an iteration-independent choice of $\lambda$ does not consistently lead to the recovery of $k^\star$ and $v^\star$ in this case (see Appendix \ref{app:experimental-details}). This contrasts with the dynamics on the manifold proven in Theorem \ref{thm:main}.

To investigate these behaviors, a fruitful direction would be to investigate the (local) stability of the manifold $\cM$ for the PGD dynamics. If the manifold is indeed stable, one can hope to transfer the analysis on the manifold to dynamics initialized close to the manifold. Furthermore, recall that, in high dimension, random vectors on the sphere are close to being orthogonal. Thus, with high probability, a uniform initialization in $(\mathbb{S}^{d-1})^2$ falls in the neighborhood of the manifold $\cM$, so that the local analysis should allow to conclude.

The proof of the theorem relies on a detailed analysis of the dynamics of the PGD algorithm on the invariant manifold $\cM$, in particular the properties of its stationary points. These arguments, which lie at the intersection of dynamical systems and topology, are of independent interest. A key idea is to reduce the problem to a two-dimensional system depending only on $\kappa = k^\top k^\star$ and $\nu = v^\top v^\star$. 

Finally, numerical experiments show that a full Transformer layer is able to solve the single-location regression task. Similarly to our simplified predictor, the weights align with the oracle parameters $k^\star$ and $v^\star$. This supports the connection drawn in Section \ref{sec:solving} between our predictor and attention layers. We refer to Appendix \ref{app:experimental-details} for details and plots, as well as experiments on multiple-location regression, a variant of single-location regression where the output depends on several tokens.

\begin{figure}[ht]
    \centering
    \includegraphics[width=\textwidth]{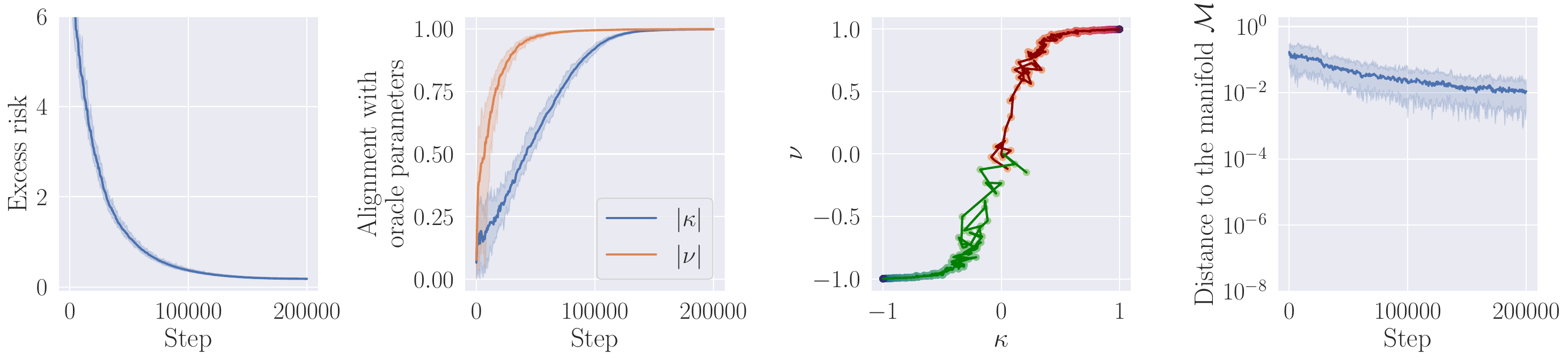}
    \caption{Convergence of online stochastic PGD to the oracle parameters from a random initialization on $(\mathbb{S}^{d-1})^2$. \textbf{Left:} Excess risk as a function of the number of steps. \textbf{Middle left:} Alignment $|\kappa| = |k^\top k^\star|$ and $|\nu| = |v^\top v^\star|$ with the oracle parameters. \textbf{Middle right:} Trajectories of $\kappa$ and $\nu$ in two repetitions of the experiment. Each repetition corresponds to a color, the trajectory starts in the middle and ends at a corner of the plot. \textbf{Right:} Distance to the invariant manifold $\cM$. In all plots except the middle right one, the experiment is repeated $30$ times with independent random initializations, and $95\%$ percentile intervals are plotted. Parameters are $d=80$, $L=10$, $\gamma = \sqrt{1/2}$, $\lambda_t = 2/(1 + 10^{-4}t)$, and a batch size of $5$. More details are given in Appendix \ref{app:experimental-details}.}
    \label{fig:experiment-sgd}
\end{figure}

\section{Conclusion}
\label{sec:conclusion}
This paper introduced \textit{single-location regression}, a novel statistical task where the relevant information in the input sequence is supported by a single token. We analyzed the statistical properties and optimization dynamics of a natural estimator for this task, which can be seen as a basic attention layer. We hope this work encourages further research into how Transformer architectures address sparsity and long-range dependencies, while simultaneously constructing internal linear representations of their input—--an aspect with significant implications for interpretability. Beyond NLP, potential applications include problems connected to sparse sequential modeling such as anomaly detection in time series. A natural extension of our framework is when relevant information is spread across a few input tokens rather than just one, which relates to multi-head attention. Future mathematical analyses should also consider extensions to general initialization schemes and stochastic dynamics. Our  experiments (Figures \ref{fig:init-sphere}, \ref{fig:experiment-sgd}, and Appendix \ref{app:experimental-details}) yield encouraging results in all these directions.

\subsubsection*{Acknowledgments}
Authors thank Peter Bartlett, Linus Bleistein, Alex Damian, Spencer Frei, and Clément Mantoux for fruitful discussions and feedback. P.M.~is supported by a Google PhD Fellowship.

\bibliography{references}

\begin{thebibliography}{66}
\providecommand{\natexlab}[1]{#1}
\providecommand{\url}[1]{\texttt{#1}}
\expandafter\ifx\csname urlstyle\endcsname\relax
  \providecommand{\doi}[1]{doi: #1}\else
  \providecommand{\doi}{doi: \begingroup \urlstyle{rm}\Url}\fi

\bibitem[Aghajanyan et~al.(2021)Aghajanyan, Gupta, and
  Zettlemoyer]{aghajanyan2021intrinsic}
Armen Aghajanyan, Sonal Gupta, and Luke Zettlemoyer.
\newblock Intrinsic dimensionality explains the effectiveness of language model
  fine-tuning.
\newblock In C.~Zong, F.~Xia, W.~Li, and R.~Navigli (eds.), \emph{Proceedings
  of the 59th Annual Meeting of the Association for Computational Linguistics
  and the 11th International Joint Conference on Natural Language Processing
  (Volume 1: Long Papers)}, pp.\  7319--7328. Association for Computational
  Linguistics, 2021.

\bibitem[Ahn et~al.(2023)Ahn, Cheng, Daneshmand, and Sra]{ahn2023transformers}
Kwangjun Ahn, Xiang Cheng, Hadi Daneshmand, and Suvrit Sra.
\newblock Transformers learn to implement preconditioned gradient descent for
  in-context learning.
\newblock In A.~Oh, T.~Naumann, A.~Globerson, K.~Saenko, M.~Hardt, and
  S.~Levine (eds.), \emph{Advances in Neural Information Processing Systems},
  volume~36, pp.\  45614--45650. Curran Associates, Inc., 2023.

\bibitem[Bahdanau et~al.(2015)Bahdanau, Cho, and Bengio]{bahdanau2015neural}
Dzmitry Bahdanau, Kyunghyun Cho, and Yoshua Bengio.
\newblock Neural machine translation by jointly learning to align and
  translate.
\newblock In Y.~Bengio and Y.~LeCun (eds.), \emph{3rd International Conference
  on Learning Representations}, 2015.

\bibitem[Bahdanau et~al.(2016)Bahdanau, Chorowski, Serdyuk, Brakel, and
  Bengio]{bahdanau2016endtoend}
Dzmitry Bahdanau, Jan Chorowski, Dmitriy Serdyuk, Philémon Brakel, and Yoshua
  Bengio.
\newblock End-to-end attention-based large vocabulary speech recognition.
\newblock In \emph{2016 IEEE International Conference on Acoustics, Speech and
  Signal Processing (ICASSP)}, pp.\  4945--4949, 2016.

\bibitem[Berthier et~al.(2024)Berthier, Montanari, and
  Zhou]{berthier2024learning}
Rapha{\"e}l Berthier, Andrea Montanari, and Kangjie Zhou.
\newblock Learning time-scales in two-layers neural networks.
\newblock \emph{Foundations of Computational Mathematics}, pp.\  1--84, 2024.

\bibitem[Bietti et~al.(2023)Bietti, Cabannes, Bouchacourt, J{\'e}gou, and
  Bottou]{bietti2023birth}
Alberto Bietti, Vivien Cabannes, Diane Bouchacourt, Herv{\'e} J{\'e}gou, and
  L{\'e}on Bottou.
\newblock Birth of a transformer: A memory viewpoint.
\newblock In A.~Oh, T.~Naumann, A.~Globerson, K.~Saenko, M.~Hardt, and
  S.~Levine (eds.), \emph{Advances in Neural Information Processing Systems},
  volume~36, pp.\  1560--1588. Curran Associates, Inc., 2023.

\bibitem[Bolukbasi et~al.(2016)Bolukbasi, Chang, Zou, Saligrama, and
  Kalai]{bolukbasi2016man}
Tolga Bolukbasi, Kai-Wei Chang, James Zou, Venkatesh Saligrama, and Adam Kalai.
\newblock Man is to computer programmer as woman is to homemaker? {D}ebiasing
  word embeddings.
\newblock In D.~Lee, M.~Sugiyama, U.~{von Luxburg}, I.~Guyon, and R.~Garnett
  (eds.), \emph{Advances in Neural Information Processing Systems}, volume~29,
  pp.\  4356--4364. Curran Associates, Inc., 2016.

\bibitem[Bolukbasi et~al.(2021)Bolukbasi, Pearce, Yuan, Coenen, Reif,
  Vi{\'e}gas, and Wattenberg]{bolukbasi2021interpretability}
Tolga Bolukbasi, Adam Pearce, Ann Yuan, Andy Coenen, Emily Reif, Fernanda
  Vi{\'e}gas, and Martin Wattenberg.
\newblock An interpretability illusion for {BERT}.
\newblock \emph{arXiv:2104.07143}, 2021.

\bibitem[Boumal(2023)]{boumal2023introduction}
Nicolas Boumal.
\newblock \emph{An Introduction to Optimization on Smooth Manifolds}.
\newblock Cambridge University Press, Cambridge, 2023.

\bibitem[Bradbury et~al.(2018)Bradbury, Frostig, Hawkins, Johnson, Leary,
  Maclaurin, Necula, Paszke, Vander{P}las, Wanderman-{M}ilne, and
  Zhang]{jax2018github}
James Bradbury, Roy Frostig, Peter Hawkins, Matthew~James Johnson, Chris Leary,
  Dougal Maclaurin, George Necula, Adam Paszke, Jake Vander{P}las, Skye
  Wanderman-{M}ilne, and Qiao Zhang.
\newblock {JAX}: composable transformations of {P}ython+{N}um{P}y programs,
  2018.
\newblock URL \url{http://github.com/jax-ml/jax}.

\bibitem[Bubeck et~al.(2023)Bubeck, Chandrasekaran, Eldan, Gehrke, Horvitz,
  Kamar, Lee, Lee, Li, Lundberg, et~al.]{bubeck2023sparks}
S{\'e}bastien Bubeck, Varun Chandrasekaran, Ronen Eldan, Johannes Gehrke, Eric
  Horvitz, Ece Kamar, Peter Lee, Yin~Tat Lee, Yuanzhi Li, Scott Lundberg,
  et~al.
\newblock Sparks of artificial general intelligence: Early experiments with
  {GPT-4}.
\newblock \emph{arXiv:2303.12712}, 2023.

\bibitem[Burns et~al.(2023)Burns, Ye, Klein, and
  Steinhardt]{burns2023discovering}
Collin Burns, Haotian Ye, Dan Klein, and Jacob Steinhardt.
\newblock Discovering latent knowledge in language models without supervision.
\newblock In \emph{The Eleventh International Conference on Learning
  Representations}, 2023.

\bibitem[Child et~al.(2019)Child, Gray, Radford, and
  Sutskever]{child2019generating}
Rewon Child, Scott Gray, Alec Radford, and Ilya Sutskever.
\newblock Generating long sequences with sparse transformers.
\newblock \emph{arXiv:1904.10509}, 2019.

\bibitem[Chizat et~al.(2019)Chizat, Oyallon, and Bach]{chizat2019lazy}
L\'{e}na\"{\i}c Chizat, Edouard Oyallon, and Francis Bach.
\newblock On lazy training in differentiable programming.
\newblock In H.~Wallach, H.~Larochelle, A.~Beygelzimer, F.~d\textquotesingle
  Alch\'{e}-Buc, E.~Fox, and R.~Garnett (eds.), \emph{Advances in Neural
  Information Processing Systems}, volume~32. Curran Associates, Inc., 2019.

\bibitem[Clark et~al.(2019)Clark, Khandelwal, Levy, and Manning]{clark2019bert}
Kevin Clark, Urvashi Khandelwal, Omer Levy, and Christopher~D. Manning.
\newblock What does {BERT} look at? {A}n analysis of {BERT}{'}s attention.
\newblock In T.~Linzen, G.~Chrupa{\l}a, Y.~Belinkov, and D.~Hupkes (eds.),
  \emph{Proceedings of the 2019 ACL Workshop BlackboxNLP: Analyzing and
  Interpreting Neural Networks for NLP}, pp.\  276--286. Association for
  Computational Linguistics, 2019.

\bibitem[Correia et~al.(2019)Correia, Niculae, and
  Martins]{correia-etal-2019-adaptively}
Gon{\c{c}}alo~M. Correia, Vlad Niculae, and Andr{\'e}~F.T. Martins.
\newblock Adaptively sparse transformers.
\newblock In K.~Inui, J.~Jiang, V.~Ng, and X.~Wan (eds.), \emph{Proceedings of
  the 2019 Conference on Empirical Methods in Natural Language Processing and
  the 9th International Joint Conference on Natural Language Processing
  (EMNLP-IJCNLP)}, pp.\  2174--2184. Association for Computational Linguistics,
  2019.

\bibitem[Dagr\'{e}ou et~al.(2022)Dagr\'{e}ou, Ablin, Vaiter, and
  Moreau]{dagreou2022framework}
Mathieu Dagr\'{e}ou, Pierre Ablin, Samuel Vaiter, and Thomas Moreau.
\newblock A framework for bilevel optimization that enables stochastic and
  global variance reduction algorithms.
\newblock In S.~Koyejo, S.~Mohamed, A.~Agarwal, D.~Belgrave, K.~Cho, and A.~Oh
  (eds.), \emph{Advances in Neural Information Processing Systems}, volume~35,
  pp.\  26698--26710. Curran Associates, Inc., 2022.

\bibitem[Darcet et~al.(2024)Darcet, Oquab, Mairal, and
  Bojanowski]{darcet2024vision}
Timoth{\'e}e Darcet, Maxime Oquab, Julien Mairal, and Piotr Bojanowski.
\newblock Vision transformers need registers.
\newblock In \emph{The Twelfth International Conference on Learning
  Representations}, 2024.

\bibitem[{De Veaux}(1989)]{deveaux1989mixtures}
Richard~D. {De Veaux}.
\newblock Mixtures of linear regressions.
\newblock \emph{Computational Statistics \& Data Analysis}, 8:\penalty0
  227--245, 1989.

\bibitem[Devlin et~al.(2019)Devlin, Chang, Lee, and Toutanova]{devlin2019bert}
Jacob Devlin, Ming-Wei Chang, Kenton Lee, and Kristina Toutanova.
\newblock {BERT}: Pre-training of deep bidirectional transformers for language
  understanding.
\newblock In \emph{Proceedings of the 2019 Conference of the North {A}merican
  Chapter of the Association for Computational Linguistics: Human Language
  Technologies, Volume 1 (Long and Short Papers)}, pp.\  4171--4186.
  Association for Computational Linguistics, 2019.

\bibitem[Farina et~al.(2024)Farina, Ahmad, Taha, Younes, Mesbah, Yu, and
  Pedrycz]{farina2024sparsity}
Mirko Farina, Usman Ahmad, Ahmad Taha, Hussein Younes, Yusuf Mesbah, Xiao Yu,
  and Witold Pedrycz.
\newblock Sparsity in transformers: A systematic literature review.
\newblock \emph{Neurocomputing}, 582:\penalty0 127468, 2024.

\bibitem[Hastie et~al.(2009)Hastie, Tibshirani, and Friedman]{hastie2009}
Trevor Hastie, Robert Tibshirani, and Jerome Friedman.
\newblock \emph{The Elements of Statistical Learning. Data Mining, Inference,
  and Prediction}.
\newblock Springer, New York, 2 edition, 2009.

\bibitem[He \& Hofmann(2024)He and Hofmann]{he2024simplifying}
Bobby He and Thomas Hofmann.
\newblock Simplifying transformer blocks.
\newblock In \emph{The Twelfth International Conference on Learning
  Representations}, 2024.

\bibitem[He et~al.(2023)He, Martens, Zhang, Botev, Brock, Smith, and
  Teh]{he2023deep}
Bobby He, James Martens, Guodong Zhang, Aleksandar Botev, Andrew Brock,
  Samuel~L Smith, and Yee~Whye Teh.
\newblock Deep transformers without shortcuts: Modifying self-attention for
  faithful signal propagation.
\newblock In \emph{The Eleventh International Conference on Learning
  Representations}, 2023.

\bibitem[Heusel et~al.(2017)Heusel, Ramsauer, Unterthiner, Nessler, and
  Hochreiter]{heusel2017gans}
Martin Heusel, Hubert Ramsauer, Thomas Unterthiner, Bernhard Nessler, and Sepp
  Hochreiter.
\newblock {GANs} trained by a two time-scale update rule converge to a local
  nash equilibrium.
\newblock In I.~Guyon, U.~{von Luxburg}, S.~Bengio, H.~Wallach, R.~Fergus,
  S.~Vishwanathan, and R.~Garnett (eds.), \emph{Advances in Neural Information
  Processing Systems}, volume~30. Curran Associates, Inc., 2017.

\bibitem[Hong et~al.(2023)Hong, Wai, Wang, and Yang]{hong2023two}
Mingyi Hong, Hoi-To Wai, Zhaoran Wang, and Zhuoran Yang.
\newblock A two-timescale stochastic algorithm framework for bilevel
  optimization: Complexity analysis and application to actor-critic.
\newblock \emph{SIAM Journal on Optimization}, 33:\penalty0 147--180, 2023.

\bibitem[Jaszczur et~al.(2021)Jaszczur, Chowdhery, Mohiuddin, Kaiser, Gajewski,
  Michalewski, and Kanerva]{jasczcur2021sparse}
Sebastian Jaszczur, Aakanksha Chowdhery, Afroz Mohiuddin, Lukasz Kaiser,
  Wojciech Gajewski, Henryk Michalewski, and Jonni Kanerva.
\newblock Sparse is enough in scaling transformers.
\newblock In M.~Ranzato, A.~Beygelzimer, Y.~Dauphin, P.S. Liang, and J.~Wortman
  Vaughan (eds.), \emph{Advances in Neural Information Processing Systems},
  volume~34, pp.\  9895--9907. Curran Associates, Inc., 2021.

\bibitem[Jelassi et~al.(2022)Jelassi, Sander, and Li]{jelassi2022vision}
Samy Jelassi, Michael Sander, and Yuanzhi Li.
\newblock Vision transformers provably learn spatial structure.
\newblock In S.~Koyejo, S.~Mohamed, A.~Agarwal, D.~Belgrave, K.~Cho, and A.~Oh
  (eds.), \emph{Advances in Neural Information Processing Systems}, volume~35,
  pp.\  37822--37836. Curran Associates, Inc., 2022.

\bibitem[Kajitsuka \& Sato(2024)Kajitsuka and Sato]{kajitsuka2024are}
Tokio Kajitsuka and Issei Sato.
\newblock Are transformers with one layer self-attention using low-rank weight
  matrices universal approximators?
\newblock In \emph{The Twelfth International Conference on Learning
  Representations}, 2024.

\bibitem[Kim et~al.(2022)Kim, Shen, Thorsley, Gholami, Kwon, Hassoun, and
  Keutzer]{kim2022learned}
Sehoon Kim, Sheng Shen, David Thorsley, Amir Gholami, Woosuk Kwon, Joseph
  Hassoun, and Kurt Keutzer.
\newblock Learned token pruning for transformers.
\newblock In \emph{Proceedings of the 28th ACM SIGKDD Conference on Knowledge
  Discovery and Data Mining}, pp.\  784–794. Association for Computing
  Machinery, 2022.

\bibitem[Lange(2013)]{lange2013optimization}
Kenneth Lange.
\newblock \emph{Optimization}.
\newblock Springer, New York, 2 edition, 2013.

\bibitem[Le~Gall(2022)]{legall2022measure}
Jean-François Le~Gall.
\newblock \emph{Measure Theory, Probability, and Stochastic Processes}.
\newblock Springer Cham, 2022.

\bibitem[Li et~al.(2023{\natexlab{a}})Li, Patel, Vi\'{e}gas, Pfister, and
  Wattenberg]{li2023inference}
Kenneth Li, Oam Patel, Fernanda Vi\'{e}gas, Hanspeter Pfister, and Martin
  Wattenberg.
\newblock Inference-time intervention: Eliciting truthful answers from a
  language model.
\newblock In A.~Oh, T.~Naumann, A.~Globerson, K.~Saenko, M.~Hardt, and
  S.~Levine (eds.), \emph{Advances in Neural Information Processing Systems},
  volume~36, pp.\  41451--41530. Curran Associates, Inc., 2023{\natexlab{a}}.

\bibitem[Li et~al.(2023{\natexlab{b}})Li, Li, and Risteski]{li2023how}
Yuchen Li, Yuanzhi Li, and Andrej Risteski.
\newblock How do transformers learn topic structure: Towards a mechanistic
  understanding.
\newblock In A.~Krause, E.~Brunskill, K.~Cho, B.~Engelhardt, S.~Sabato, and
  J.~Scarlett (eds.), \emph{Proceedings of the 40th International Conference on
  Machine Learning}, volume 202 of \emph{Proceedings of Machine Learning
  Research}, pp.\  19689--19729. PMLR, 2023{\natexlab{b}}.

\bibitem[Lin et~al.(2022)Lin, Wang, Liu, and Qiu]{lin2022survey}
Tianyang Lin, Yuxin Wang, Xiangyang Liu, and Xipeng Qiu.
\newblock A survey of transformers.
\newblock \emph{AI Open}, 3:\penalty0 111--132, 2022.

\bibitem[Luong et~al.(2015)Luong, Pham, and Manning]{luong-etal-2015-effective}
Thang Luong, Hieu Pham, and Christopher~D. Manning.
\newblock Effective approaches to attention-based neural machine translation.
\newblock In L.~M{\`a}rquez, C.~Callison-Burch, and J.~Su (eds.),
  \emph{Proceedings of the 2015 Conference on Empirical Methods in Natural
  Language Processing}, pp.\  1412--1421. Association for Computational
  Linguistics, 2015.

\bibitem[Marion \& Berthier(2023)Marion and Berthier]{marion2023leveraging}
Pierre Marion and Rapha{\"e}l Berthier.
\newblock Leveraging the two timescale regime to demonstrate convergence of
  neural networks.
\newblock In A.~Oh, T.~Naumann, A.~Globerson, K.~Saenko, M.~Hardt, and
  S.~Levine (eds.), \emph{Advances in Neural Information Processing Systems},
  volume~36, pp.\  64996--65029. Curran Associates, Inc., 2023.

\bibitem[Marion et~al.(2024)Marion, Korba, Bartlett, Blondel, Bortoli, Doucet,
  Llinares-López, Paquette, and Berthet]{marion2024implicit}
Pierre Marion, Anna Korba, Peter Bartlett, Mathieu Blondel, Valentin~De
  Bortoli, Arnaud Doucet, Felipe Llinares-López, Courtney Paquette, and
  Quentin Berthet.
\newblock Implicit diffusion: Efficient optimization through stochastic
  sampling.
\newblock \emph{arXiv:2402.05468}, 2024.

\bibitem[Martins \& Astudillo(2016)Martins and Astudillo]{martins2016from}
Andre Martins and Ramon Astudillo.
\newblock From softmax to sparsemax: A sparse model of attention and
  multi-label classification.
\newblock In M.F. Balcan and K.Q. Weinberger (eds.), \emph{Proceedings of The
  33rd International Conference on Machine Learning}, volume~48 of
  \emph{Proceedings of Machine Learning Research}, pp.\  1614--1623. PMLR,
  2016.

\bibitem[McCullagh \& Nelder(1983)McCullagh and Nelder]{mccullagh83}
Peter McCullagh and John~A. Nelder.
\newblock \emph{Generalized Linear Models}.
\newblock Chapman \& Hall, London, 2 edition, 1983.

\bibitem[Mikolov et~al.(2013{\natexlab{a}})Mikolov, Le, and
  Sutskever]{mikolov2013exploiting}
Tomas Mikolov, Quoc~V. Le, and Ilya Sutskever.
\newblock Exploiting similarities among languages for machine translation.
\newblock \emph{arXiv:1309.4168}, 2013{\natexlab{a}}.

\bibitem[Mikolov et~al.(2013{\natexlab{b}})Mikolov, Yih, and
  Zweig]{mikolov-etal-2013-linguistic}
Tomas Mikolov, Wen-tau Yih, and Geoffrey Zweig.
\newblock Linguistic regularities in continuous space word representations.
\newblock In L.~Vanderwende, H.~Daum{\'e}~III, and K.~Kirchhoff (eds.),
  \emph{Proceedings of the 2013 Conference of the North {A}merican Chapter of
  the Association for Computational Linguistics: Human Language Technologies},
  pp.\  746--751. Association for Computational Linguistics,
  2013{\natexlab{b}}.

\bibitem[Nanda et~al.(2023)Nanda, Lee, and
  Wattenberg]{nanda-etal-2023-emergent}
Neel Nanda, Andrew Lee, and Martin Wattenberg.
\newblock Emergent linear representations in world models of self-supervised
  sequence models.
\newblock In Y.~Belinkov, S.~Hao, J.~Jumelet, N.~Kim, A.~McCarthy, and
  H.~Mohebbi (eds.), \emph{Proceedings of the 6th BlackboxNLP Workshop:
  Analyzing and Interpreting Neural Networks for NLP}, pp.\  16--30.
  Association for Computational Linguistics, 2023.

\bibitem[Nichani et~al.(2024)Nichani, Damian, and Lee]{nichani2024how}
Eshaan Nichani, Alex Damian, and Jason~D. Lee.
\newblock How transformers learn causal structure with gradient descent.
\newblock In \emph{International Conference on Machine Learning}. PLMR, 2024.

\bibitem[Niculae \& Blondel(2017)Niculae and Blondel]{blondel2017regularized}
Vlad Niculae and Mathieu Blondel.
\newblock A regularized framework for sparse and structured neural attention.
\newblock In I.~Guyon, U.~{von Luxburg}, S.~Bengio, H.~Wallach, R.~Fergus,
  S.~Vishwanathan, and R.~Garnett (eds.), \emph{Advances in Neural Information
  Processing Systems}, volume~30. Curran Associates, Inc., 2017.

\bibitem[Pedregosa et~al.(2011)Pedregosa, Varoquaux, Gramfort, Michel, Thirion,
  Grisel, Blondel, Prettenhofer, Weiss, Dubourg, Vanderplas, Passos,
  Cournapeau, Brucher, Perrot, and Duchesnay]{pedregosa2011scikitlearn}
F.~Pedregosa, G.~Varoquaux, A.~Gramfort, V.~Michel, B.~Thirion, O.~Grisel,
  M.~Blondel, P.~Prettenhofer, R.~Weiss, V.~Dubourg, J.~Vanderplas, A.~Passos,
  D.~Cournapeau, M.~Brucher, M.~Perrot, and E.~Duchesnay.
\newblock Scikit-learn: Machine learning in {P}ython.
\newblock \emph{Journal of Machine Learning Research}, 12\penalty0
  (85):\penalty0 2825--2830, 2011.

\bibitem[Phuong \& Hutter(2022)Phuong and Hutter]{phuong2022formal}
Mary Phuong and Marcus Hutter.
\newblock Formal algorithms for transformers.
\newblock \emph{arXiv:2207.09238}, 2022.

\bibitem[Press et~al.(2007)Press, Teukolsky, Vetterling, and
  Flannery]{press2007numerical}
William Press, Saul Teukolsky, William Vetterling, and Brian Flannery.
\newblock \emph{Numerical Recipes: The Art of Scientific Computing}.
\newblock Cambridge University Press, Cambridge, 3 edition, 2007.

\bibitem[Qin et~al.(2022)Qin, Sun, Deng, Li, Wei, Lv, Yan, Kong, and
  Zhong]{qin2022cosformer}
Zhen Qin, Weixuan Sun, Hui Deng, Dongxu Li, Yunshen Wei, Baohong Lv, Junjie
  Yan, Lingpeng Kong, and Yiran Zhong.
\newblock cosformer: Rethinking softmax in attention.
\newblock \emph{arXiv:2202.08791}, 2022.

\bibitem[Ramapuram et~al.(2024)Ramapuram, Danieli, Dhekane, Weers, Busbridge,
  Ablin, Likhomanenko, Digani, Gu, Shidani, et~al.]{ramapuram2024theory}
Jason Ramapuram, Federico Danieli, Eeshan Dhekane, Floris Weers, Dan Busbridge,
  Pierre Ablin, Tatiana Likhomanenko, Jagrit Digani, Zijin Gu, Amitis Shidani,
  et~al.
\newblock Theory, analysis, and best practices for sigmoid self-attention.
\newblock \emph{arXiv:2409.04431}, 2024.

\bibitem[Shen et~al.(2023)Shen, Guo, Tan, Tang, Wang, and Bian]{shen2023study}
Kai Shen, Junliang Guo, Xu~Tan, Siliang Tang, Rui Wang, and Jiang Bian.
\newblock A study on {ReLU} and {S}oftmax in {T}ransformer.
\newblock \emph{arXiv:2302.06461}, 2023.

\bibitem[Shub(1987)]{shub1987global}
Michael Shub.
\newblock \emph{Global Stability of Dynamical Systems}.
\newblock Springer, New York, 1987.

\bibitem[Song et~al.(2019)Song, Wang, Jiang, Liu, and Rao]{song2019attentional}
Youwei Song, Jiahai Wang, Tao Jiang, Zhiyue Liu, and Yanghui Rao.
\newblock Attentional encoder network for targeted sentiment classification.
\newblock \emph{arXiv:1902.09314}, 2019.

\bibitem[Stein(1981)]{stein1981estimation}
Charles~M. Stein.
\newblock Estimation of the mean of a multivariate normal distribution.
\newblock \emph{The Annals of Statistics}, 9:\penalty0 1135--1151, 1981.

\bibitem[Sun et~al.(2019)Sun, Huang, and Qiu]{sun-etal-2019-utilizing}
Chi Sun, Luyao Huang, and Xipeng Qiu.
\newblock Utilizing {BERT} for aspect-based sentiment analysis via constructing
  auxiliary sentence.
\newblock In J.~Burstein, C.~Doran, and T.~Solorio (eds.), \emph{Proceedings of
  the 2019 Conference of the North {A}merican Chapter of the Association for
  Computational Linguistics: Human Language Technologies, Volume 1 (Long and
  Short Papers)}, pp.\  380--385. Association for Computational Linguistics,
  2019.

\bibitem[Tian et~al.(2023)Tian, Wang, Chen, and Du]{tian2023scan}
Yuandong Tian, Yiping Wang, Beidi Chen, and Simon~S. Du.
\newblock Scan and snap: Understanding training dynamics and token composition
  in 1-layer transformer.
\newblock In A.~Oh, T.~Naumann, A.~Globerson, K.~Saenko, M.~Hardt, and
  S.~Levine (eds.), \emph{Advances in Neural Information Processing Systems},
  volume~36, pp.\  71911--71947. Curran Associates, Inc., 2023.

\bibitem[Vaswani et~al.(2017)Vaswani, Shazeer, Parmar, Uszkoreit, Jones, Gomez,
  Kaiser, and Polosukhin]{vaswani2017attention}
Ashish Vaswani, Noam Shazeer, Niki Parmar, Jakob Uszkoreit, Llion Jones,
  Aidan~N. Gomez, {\L}ukasz Kaiser, and Illia Polosukhin.
\newblock Attention is all you need.
\newblock In I.~Guyon, U.~{von Luxburg}, S.~Bengio, H.~Wallach, R.~Fergus,
  S.~Vishwanathan, and R.~Garnett (eds.), \emph{Advances in Neural Information
  Processing Systems}, volume~30, pp.\  6000--6010. Curran Associates, Inc.,
  2017.

\bibitem[{von Oswald} et~al.(2023){von Oswald}, Niklasson, Randazzo,
  Sacramento, Mordvintsev, Zhmoginov, and
  Vladymyrov]{vonoswald2023transformers}
Johannes {von Oswald}, Eyvind Niklasson, Ettore Randazzo, Joao Sacramento,
  Alexander Mordvintsev, Andrey Zhmoginov, and Max Vladymyrov.
\newblock Transformers learn in-context by gradient descent.
\newblock In A.~Krause, E.~Brunskill, K.~Cho, B.~Engelhardt, S.~Sabato, and
  J.~Scarlett (eds.), \emph{Proceedings of the 40th International Conference on
  Machine Learning}, volume 202 of \emph{Proceedings of Machine Learning
  Research}, pp.\  35151--35174. PMLR, 2023.

\bibitem[Wang et~al.(2024)Wang, Wei, Hsu, and Lee]{wang2024transformers}
Zixuan Wang, Stanley Wei, Daniel Hsu, and Jason~D. Lee.
\newblock Transformers provably learn sparse token selection while
  fully-connected nets cannot.
\newblock In \emph{International Conference on Machine Learning}. PLMR, 2024.

\bibitem[Wen-Yi \& Mimno(2023)Wen-Yi and Mimno]{wenyi2023hyperpolyglot}
Andrea~W Wen-Yi and David Mimno.
\newblock Hyperpolyglot {LLM}s: Cross-lingual interpretability in token
  embeddings.
\newblock In H.~Bouamor, J.~Pino, and K.~Bali (eds.), \emph{Proceedings of the
  2023 Conference on Empirical Methods in Natural Language Processing}, pp.\
  1124--1131. Association for Computational Linguistics, 2023.

\bibitem[Wolf et~al.(2020)Wolf, Debut, Sanh, Chaumond, Delangue, Moi, Cistac,
  Rault, Louf, Funtowicz, Davison, Shleifer, von Platen, Ma, Jernite, Plu, Xu,
  Scao, Gugger, Drame, Lhoest, and Rush]{wolf2020transformers}
Thomas Wolf, Lysandre Debut, Victor Sanh, Julien Chaumond, Clement Delangue,
  Anthony Moi, Pierric Cistac, Tim Rault, Rémi Louf, Morgan Funtowicz, Joe
  Davison, Sam Shleifer, Patrick von Platen, Clara Ma, Yacine Jernite, Julien
  Plu, Canwen Xu, Teven~Le Scao, Sylvain Gugger, Mariama Drame, Quentin Lhoest,
  and Alexander~M. Rush.
\newblock Transformers: State-of-the-art natural language processing.
\newblock In \emph{Proceedings of the 2020 Conference on Empirical Methods in
  Natural Language Processing: System Demonstrations}, pp.\  38--45.
  Association for Computational Linguistics, 2020.

\bibitem[Wortsman et~al.(2023)Wortsman, Lee, Gilmer, and
  Kornblith]{wortsman2023replacing}
Mitchell Wortsman, Jaehoon Lee, Justin Gilmer, and Simon Kornblith.
\newblock Replacing softmax with {ReLU} in vision transformers.
\newblock \emph{arXiv:2309.08586}, 2023.

\bibitem[Xiao et~al.(2024)Xiao, Tian, Chen, Han, and Lewis]{xiao2024efficient}
Guangxuan Xiao, Yuandong Tian, Beidi Chen, Song Han, and Mike Lewis.
\newblock Efficient streaming language models with attention sinks.
\newblock In \emph{The Twelfth International Conference on Learning
  Representations}, 2024.

\bibitem[Xu et~al.(2019)Xu, Liu, Shu, and Yu]{xu-etal-2019-bert}
Hu~Xu, Bing Liu, Lei Shu, and Philip Yu.
\newblock {BERT} post-training for review reading comprehension and
  aspect-based sentiment analysis.
\newblock In J.~Burstein, C.~Doran, and T.~Solorio (eds.), \emph{Proceedings of
  the 2019 Conference of the North {A}merican Chapter of the Association for
  Computational Linguistics: Human Language Technologies, Volume 1 (Long and
  Short Papers)}, pp.\  2324--2335. Association for Computational Linguistics,
  2019.

\bibitem[Yang et~al.(2021)Yang, Hu, Babuschkin, Sidor, Liu, Farhi, Ryder,
  Pachocki, Chen, and Gao]{yang2021tuning}
Ge~Yang, Edward Hu, Igor Babuschkin, Szymon Sidor, Xiaodong Liu, David Farhi,
  Nick Ryder, Jakub Pachocki, Weizhu Chen, and Jianfeng Gao.
\newblock Tuning large neural networks via zero-shot hyperparameter transfer.
\newblock \emph{Advances in Neural Information Processing Systems},
  34:\penalty0 17084--17097, 2021.

\bibitem[Zhang et~al.(2024)Zhang, Frei, and Bartlett]{zhang2024trained}
Ruiqi Zhang, Spencer Frei, and Peter~L. Bartlett.
\newblock Trained transformers learn linear models in-context.
\newblock \emph{Journal of Machine Learning Research}, 25\penalty0
  (49):\penalty0 1--55, 2024.

\end{thebibliography}
\bibliographystyle{iclr2025_conference}

\newpage

\appendix

\begin{center}
    \LARGE \bf Appendix
\end{center}

\paragraph{Organization of the Appendix.} Section \ref{sec:outline-proof-main-thm} presents the main steps of Theorem \ref{thm:main}. The intermediate results of this proof, as well as the other statements of the main text, are proven in Section \ref{app:proof_main}. Section~\ref{proof:risk_on_sphere} provides an expression of the risk $\cR$ beyond the manifold $\cM$ that extends the one provided in Lemma~\ref{lem:risk_on_M} on $\cM$. Section \ref{app:tech-results} gives some useful technical lemmas. Experimental details and additional results are in Section \ref{app:experimental-details}. Finally, Section \ref{app:related-works} discusses additional related models.

\paragraph{Notation.} In the whole Appendix, we consider a constant inverse temperature schedule $\lambda_t \equiv \lambda > 0$, as in Theorem \ref{thm:main}. For this reason, it is not necessary to make explicit the dependence of $\cR_\lambda$ and $\cR_\lambda^<$ on $\lambda$, and we use the lighter notations  $\cR$ and $\cR^<$ instead.

\section{Outline of the proof of Theorem \ref{thm:main}}    \label{sec:outline-proof-main-thm}

This section outlines the essential steps for the proof of Theorem \ref{thm:main}. For clarity, the proofs are to be found in Appendix \ref{app:proof_main}, except the proof of Proposition \ref{prop:minima_maxima_saddle}.

\paragraph{Step 1: Invariant manifold \& reparameterization.}
We first show that the risk $\mathcal{R}(k,v)$ has a simpler expression when considered on the manifold $\mathcal{M}$. 
 \begin{lemma}
     \label{lem:risk_on_M}
     The risk $\cR(k, v)$ restricted to $\mathcal M$ has the form
\begin{align*}
{\cR(k, v) }
 &= \gamma^2- 2 \gamma^2 v^\top v^\star \, \erf \left(\lambda \sqrt{\frac{d}{2(1+2\lambda^2\gamma^2)}} k^\top k^\star \right)  +  \gamma^2 \zeta\Big(\lambda \sqrt{\frac{d}{2}} \, k^\top k^\star,\lambda^2 \gamma^2\Big) \\
 &\qquad +   (\length -1)  \zeta(0, \lambda^2)  + \varepsilon^2 ,
 \end{align*}
where, for $t,\gamma\in \mathbb{R}$,
\[
        \zeta(t,\gamma^2) := \Esp \left[ \erf^2(t+G)\right] \, , \qquad  G\sim\mathcal{N}(0,\gamma^2) \, .
\]
 \end{lemma}
This expression has two main consequences. First, we use it to prove that the manifold $\cM$ is invariant by PGD, according to Lemma \ref{lem:invariant_manifold}. 
Second, we observe that the risk on the manifold depends on the variables $(k,v) \in \mathbb{S}^{d-1} \times \mathbb{S}^{d-1}$ only through the two scalar quantities 
\[
\kappa = k^\top k^\star \quad \textnormal{and} \quad 
\nu = v^\top v^\star  \, .
\]
This suggests studying the dynamics in terms of the reduced variables $(\kappa, \nu) \in [-1,1]^2$. More precisely, in the following, we denote by $\mathcal{R}^<$ the risk function $\mathcal{R}$ reparameterized as a function of $(\kappa,\nu)$, i.e., we let
\[
{\cR^<(\kappa, \nu) }
 = \gamma^2- 2 \gamma^2 \nu \, \erf \left(\lambda \sqrt{\frac{d}{2(1+2\lambda^2\gamma^2)}} \kappa \right)  +  \gamma^2 \zeta\Big(\lambda \sqrt{\frac{d}{2}} \, \kappa,\lambda^2 \gamma^2\Big)  +   (\length -1)  \zeta(0, \lambda^2)  + \varepsilon^2  \, .
\]
Note that, with a slight abuse of notation, we use  $\cR^<$  to denote both the function of five variables $(\kappa, \nu, \theta, \rho, \eta)$ (as in Theorem 1) and the function of only the first two variables $(\kappa, \nu)$. There should be no confusion, as both functions coincide on the manifold $\cM$ where $\theta = \rho = \eta = 0$. 
We also denote the corresponding PGD iterates using this reparameterization by $(\kappa_t, \nu_t) := (k_t^\top k^\star, v_t^\top v^\star)$.  With this notation, the following lemma reformulates the PGD iterations as an autonomous discrete dynamical system in terms of $(\kappa_t, \nu_t)$.
\begin{lemma}   \label{lemma:2d-reformulation}
When initialized on the manifold $\cM$, the PGD iterations \eqref{eq:pgd-iterations} can be reformulated in terms of the autonomous discrete dynamical system
\begin{align}
\label{eq:dynamical_system_kappa_nu-1}
(\kappa_{t+1}, \nu_{t+1}) &= g(\kappa_t, \nu_t) \, ,
\end{align}
where the mapping $g: [-1, 1]^2 \to [-1, 1]^2$ is given by
\begin{align}
\label{eq:dynamical_system_kappa_nu-2}
g(\kappa, \nu) = \left(
   \frac{\kappa - \alpha (\partial_\kappa \mathcal{R}^<(\kappa,\nu))(1-\kappa^2)}{\sqrt{1+\alpha^2(\partial_\kappa \mathcal{R}^<(\kappa,\nu))^2 (1-\kappa^2)}} \, , \quad 
   \frac{\nu - \alpha (\partial_\nu \mathcal{R}^<(\kappa,\nu))(1-\nu^2)}{\sqrt{1+\alpha^2(\partial_\nu \mathcal{R}^<(\kappa,\nu))^2 (1-\nu^2)}}
   \right).
\end{align}    
\end{lemma}

\paragraph{Step 2: Analysis of the stationary points.} Regarding the dynamics restricted to the invariant manifold $\mathcal{M}$, we can characterize the limit points of the PGD iterates as follows.
\begin{proposition}
\label{prop:analysis-pgd}
For a sufficiently small step size $\alpha$ and for any $(k_0,v_0)\in \mathcal{M}$, the risk $\cR^<$ is decreasing along the PGD iterates.
Furthermore, the distance between successive PGD iterates tends to zero, and,
if $(\kappa, \nu)$ is an accumulation point of the sequence of iterates $(\kappa_t, \nu_t)_{t \geqslant 0}$, then
\begin{equation}    \label{eq:formula-accumulation-points}
(1-\kappa^2) \partial_\kappa \mathcal{R}^<(\kappa, \nu) = 0 \quad \textnormal{and} \quad (1-\nu^2) \partial_\nu \mathcal{R}^<(\kappa, \nu) = 0  \, .    
\end{equation}
\end{proposition}
We stress that the system \eqref{eq:formula-accumulation-points} of equations corresponds to fixed points of the dynamics \eqref{eq:dynamical_system_kappa_nu-1}--\eqref{eq:dynamical_system_kappa_nu-2}. We next solve this system of equations.
\begin{proposition} \label{prop:critical-points}
    The points $(\kappa,\nu) \in [-1,1]^2$ satisfying \eqref{eq:formula-accumulation-points} are $(\kappa,\nu)=(\pm 1, \pm 1)$\footnote{This notation is used to designate any extreme point of the square $[-1,1]^2$, i.e., $(\kappa,\nu)=(1,1)$, $(1,-1)$, $(-1,1)$, and $(-1,-1)$.} and $(\kappa,\nu)=(0, 0)$. 
\end{proposition}
The identity $(\kappa,\nu)=(\pm 1,\pm 1)$ corresponds to the situation where the variables $(k,v)$ are aligned (up to sign) with the targets $(k^\star,v^\star)$. As the next proposition shows, these are the only global minima of $\cR^<$.
\begin{proposition} 
\label{prop:minima_maxima_saddle}
The fixed points of the dynamics can be classified as follows:
\begin{enumerate}%
    \item[$(i)$] The points $(\kappa,\nu)=(-1,1)$ and $(1,-1)$ are global maxima of $\mathcal{R}^<$ on $[-1, 1]^2$.
    \item[$(ii)$] The points $(\kappa,\nu)=(1,1)$ and $(-1,-1)$ are global minima of $\mathcal{R}^<$ on $[-1, 1]^2$.
    \item[$(iii)$] The point $(\kappa,\nu)=(0,0)$ is a saddle point of $\mathcal{R}^<$ on $[-1, 1]^2$.
\end{enumerate}
\end{proposition}
The fixed points of the dynamics as well as the vector field 
\[
(\kappa,\nu) \mapsto -(\partial_\kappa \mathcal{R}^< (\kappa,\nu)(1-\kappa^2), \partial_\nu \mathcal{R}^< (\kappa,\nu)(1-\nu^2))
\]
are displayed in Figure \ref{fig:nu_kappa_vector_field}.

\begin{figure}
\begin{center}
\begin{tabular}{cc}
\begin{tikzpicture}[scale=1.9]
  \draw[thin,dotted] (-1.5,-1.5) grid (1.5,1.5);
    \draw[->] (-1.5,0) -- (1.5,0);
    \draw[->] (0,-1.5) -- (0,1.5);
    \draw[thick] (-1,-1) -- (-1,1)
    -- (1,1) -- (1,-1) -- cycle;
    \node at (1.8,0) {$\kappa$};
    \node at (0,1.8) {$\nu$};

    \node [Gold] at (1,1) {\textbullet};
    \node [Gold] at (-1,-1) {\textbullet};

    \node [RoyalBlue] at (0,0) {\textbullet};
    
    \node [OrangeRed] at (-1,1) {\textbullet};
    \node [OrangeRed] at (1,-1) {\textbullet};

    \foreach \Point in {(1,-1), (-1,-1), (-1,1), (1,1), (0,0)}{
    \node at \Point {$\circ$};

}
\end{tikzpicture}
&
\includegraphics[width=7cm]{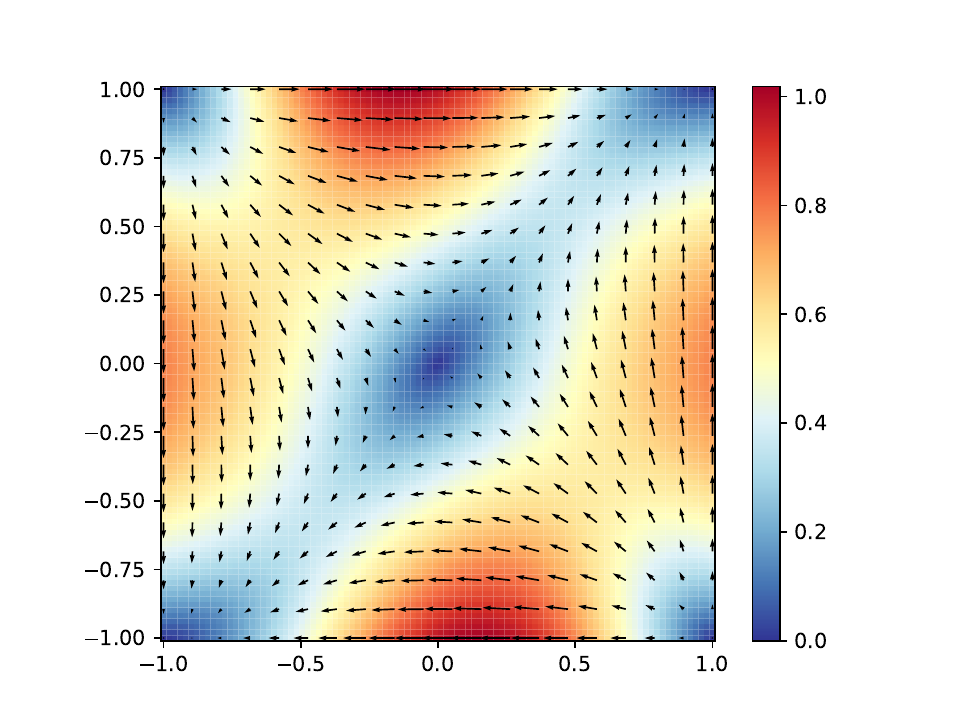} \\
{(a)} &{(b)}
\end{tabular}
\caption{Dynamics in $(\kappa, \nu)$ on the manifold $\mathcal{M}$. In (a), the fixed points of the dynamics are represented;  the  minimizers, saddle point, and maximizers are respectively depicted in yellow, blue and red. In (b), the vector field $(\kappa,\nu) \mapsto -(\partial_\kappa \mathcal{R}^< (\kappa,\nu)(1-\kappa^2), \partial_\nu \mathcal{R}^< (\kappa,\nu)(1-\nu^2))$ is displayed (the colormap corresponds to the magnitude of the vector field).}
\label{fig:nu_kappa_vector_field}
\end{center}
\end{figure}

\paragraph{Step 3: Convergence to global minima.} The convergence of the sequence of iterates $(\kappa_t,\nu_t)_{t \geqslant 0}$ to a global minimum is shown in two stages. First, we show that the iterates converge to one of the five fixed points described in Proposition \ref{prop:minima_maxima_saddle}.
\begin{proposition}
\label{prop:5points}
For a sufficiently small step size $\alpha$, the sequence of iterates $(\kappa_t, \nu_t)_{t \geqslant 0}$ converges to one of the five fixed points $\{(\pm 1, \pm 1), (0, 0)\}$. 
\end{proposition}
\begin{proof}
 According to Proposition \ref{prop:analysis-pgd}, the distance between successive iterates $(\kappa_t, \nu_t)$ tends to zero. Therefore, the set of accumulation points of the sequence $(\kappa_t, \nu_t)_{t \geqslant 0}$ is connected \citep[][Proposition 12.4.1]{lange2013optimization}. Since there is a finite number of possible accumulation points (by Proposition~\ref{prop:critical-points}), 
 we deduce that the sequence has a unique accumulation point. Furthermore, the sequence belongs to a compact. Thus, it converges, and its limit is one of the five fixed points.
\end{proof}
It remains to precisely characterize the limit of the sequence $(\kappa_t,\nu_t)_{t \geqslant 0}$. To this aim, we begin by showing key properties of the gradient mapping $g$.
\begin{proposition} 
\label{prop:properties-g}
For a sufficiently small step size $\alpha$, the mapping $g$ is a local diffeomorphism around $(0,0)$, whose Jacobian matrix has one eigenvalue in $(0, 1)$ and one eigenvalue in $(1, \infty)$. Furthermore, it is injective on $[-1, 1]^2$, differentiable, and its Jacobian is non-degenerate.
\end{proposition}
These properties enable us to apply the Center-Stable Manifold theorem \citep[][Theorem III.7]{shub1987global}, a tool from dynamical systems theory, to deduce the next proposition.
\begin{proposition}  \label{prop:no-conv-to-saddle}
For a sufficiently small step size $\alpha$, the set of initializations such that the sequence $(\kappa_t, \nu_t)_{t \geqslant 0}$ converges to $(-1, 1)$, $(1, -1)$, or $(0, 0)$ has Lebesgue measure zero (with respect to the Lebesgue measure on the manifold $\cM$). 
\end{proposition} 
Combining Proposition \ref{prop:5points} and Proposition  \ref{prop:no-conv-to-saddle}, we conclude that, provided the step size $\alpha$ is chosen small enough, the sequence $(\kappa_t, \nu_t)_{t \geqslant 0}$ almost surely converges to one of the minimizers, $(1,1)$ or $(-1,-1)$. This convergence is almost sure with respect to the Lebesgue measure on the manifold $\cM$. Indeed, Proposition \ref{prop:no-conv-to-saddle} ensures that the pathological initializations converging towards a maximizer or a saddle point are of Lebesgue measure zero. This concludes the proof of Theorem \ref{thm:main}.

The use of the Center-Stable Manifold theorem is crucial to our proof. Unfortunately, this tool does not provide quantitative rates of convergence. Obtaining a rate is a challenging task as it would require quantifying the distance of the iterates to the saddle points of the risk (the dynamics is indeed slower near saddle points), which in turn requires other tools of analysis and potentially additional assumptions.

\section{Proofs of the main results}
\label{app:proof_main}

\subsection{Proof of Lemma \ref{lem:risk_on_M} and Theorem \ref{theo:optimalrisk}}
\label{proof:risk_5_parameters}

We recall the formula for the risk
\begin{equation*}
\cR(k, v) = \Esp \Big[ \Big(Y - \sum_{\ell=1}^\length \erf (\lambda X_\ell^\top k)   X_\ell^\top v   \Big)^2 \Big]
\end{equation*}
and the data model 
\begin{align*}
   Y &= X_{J_0}^\top v^\star + \xi,  
\end{align*}
where 
\[
J_0 \in \mathcal{P}(\{1, \hdots,\length\})
\quad \text{and}\quad
\left\{
\begin{array}{lll}
    X_{J_0} &\sim&\cN\Big(\sqrt{\frac{d}{2}}k^\star, \gamma^2 I_d\Big)   \\
    X_\ell &\sim& \cN(0, I_d) \quad \textnormal{for} \quad \ell \neq J_0.
\end{array}
\right.
\]
In the above expression for the risk, we can condition on the value of $J_0$. Actually, the conditioned risk is independent of $J_0$. Thus in this section, we assume without loss of generality that $J_0 = 1$ a.s.: 
\begin{equation}    \label{eq:intermediate-eq-risk}
\cR(k, v) = \Esp \Big[ \Big(X_1^\top v^\star + \xi - \erf (\lambda X_1^\top k)   X_1^\top v - \sum_{\ell=2}^\length \erf (\lambda X_\ell^\top k)   X_\ell^\top v   \Big)^2 \Big]  \, ,
\end{equation}
where 
\[
\left\{
\begin{array}{lll}
    X_{1} &\sim&\cN\Big(\sqrt{\frac{d}{2}}k^\star, \gamma^2 I_d\Big)   \\
    X_\ell &\sim& \cN(0, I_d) \quad \textnormal{for} \quad \ell \geq 2.
\end{array}
\right.
\]
We rewrite this quantity in terms of multivariate standard Gaussian random variables. Using Assumption \ref{hyp:orthogonality}, we get
\begin{align*}
\cR(k, v) &= \Esp \Big[ \Big(\gamma \widetilde{X}_1^\top v^\star + \xi - \erf \Big(\lambda \Big(\sqrt{\frac{d}{2}}k_*^\top k + \gamma \widetilde{X}_1^\top k\Big)\Big)   \Big(\sqrt{\frac{d}{2}}k_*^\top v + \gamma \widetilde{X}_1^\top v\Big) \\
&\qquad- \sum_{\ell=2}^\length \erf (\lambda X_\ell^\top k)   X_\ell^\top v   \Big)^2 \Big]  \, ,
\end{align*}
where $\widetilde{X}_1, X_2, \dots, X_\length \sim \cN(0, I_d)$. 
This can be formulated in terms of the five scalar quantities $\kappa= k^\top k^\star$, $\nu = v^\top v^\star$, $\theta = v^\top k^\star$, $\eta = k^\top v^\star$, and $\rho = k^\top v$. Indeed, we have 
\begin{align}
\begin{split}
\label{eq:risk-5-params-1}
    &\cR(k,v) = \cR^<(\kappa, \nu, \theta, \eta, \rho) \\
    &\quad:= \Esp\Big[\Big(\gamma G_1^{v^\star} + \xi - \Big(\sqrt{\frac{d}{2}}\theta + \gamma G_1^v\Big)\erf\Big(\lambda\Big(\sqrt{\frac{d}{2}}\kappa + \gamma G_1^k\Big)\Big)-\sum_{\ell=2}^L G_{\ell}^v \erf(\lambda G_{\ell}^k)\Big)^2\Big] \, ,
    \end{split}
\end{align}
where
\begin{align}
\label{eq:risk-5-params-2}
    \begin{pmatrix}
        G_1^{v^\star} \\ G_1^v \\ G_1^k
    \end{pmatrix}
    , \dots, 
    \begin{pmatrix}
        G_L^{v^\star} \\ G_L^v \\ G_L^k
    \end{pmatrix}
    \underset{\text{i.i.d.}}{\sim} \cN\left(0, \begin{pmatrix}
        1 & \nu & \eta \\
        \nu & 1 & \rho \\
        \eta & \rho & 1 
    \end{pmatrix}\right) \, .
\end{align}
This last expression only involves the five parameters $\kappa, \nu, \theta, \eta, \rho$, which play a role either explicitly in the function or as parameters of the covariance of the random variables. This proves the first statement of Theorem \ref{theo:optimalrisk}. A computation of a closed-form formula for this expectation is given in Appendix \ref{proof:risk_on_sphere}.

On the manifold $\cM$ defined by $\theta = \eta = \rho = 0$, we can simplify the expressions \eqref{eq:risk-5-params-1}--\eqref{eq:risk-5-params-2}
\begin{align*}
    \cR^<(\kappa, \nu, 0,0,0) = \Esp\Big[\Big(\gamma G_1^{v^\star} + \xi - \gamma G_1^v\erf\Big(\lambda\Big(\sqrt{\frac{d}{2}}\kappa + \gamma G_1^k\Big)\Big)-\sum_{\ell=2}^L G_{\ell}^v \erf(\lambda G_{\ell}^k)\Big)^2\Big] 
\end{align*}
where
$\begin{pmatrix}
        G_1^{v^\star} \\ G_1^v 
    \end{pmatrix}
    , \dots, 
    \begin{pmatrix}
        G_L^{v^\star} \\ G_L^v 
    \end{pmatrix}
    \underset{\text{i.i.d.}}{\sim} \cN\left(0, \begin{pmatrix}
        1 & \nu  \\
        \nu & 1  \\
    \end{pmatrix}\right)$, $G_1^k, \dots, G_L^k \underset{\text{i.i.d.}}{\sim} \cN(0,1)$, and $\xi \sim \cN(0, \varepsilon^2)$ are independent. 

    We first expand in $\xi$ and obtain
    \begin{align*}
    \cR^<(\kappa, \nu, 0,0,0) = \varepsilon^2 + \Esp\Big[\Big(\gamma G_1^{v^\star} - \gamma G_1^v\erf\Big(\lambda\Big(\sqrt{\frac{d}{2}}\kappa + \gamma G_1^k\Big)\Big)-\sum_{\ell=2}^L G_{\ell}^v \erf(\lambda G_{\ell}^k)\Big)^2\Big]  \, .
\end{align*}
We now expand the square, as follows:
\begin{align*}
    \cR^<(\kappa, \nu, 0,0,0) &= \varepsilon^2 + \gamma^2 \Esp \left[(G_1^{v^\star})^2\right] - 2 \gamma^2\Esp \Big[G_1^{v^\star} G_1^v \erf \Big(\lambda \Big(\sqrt{\frac{d}{2}}\kappa + \gamma G_1^k\Big)\Big)\Big] \\
    &\quad+ \gamma^2\Esp \Big[ (G_1^v)^2 \erf^2 \Big(\lambda \Big(\sqrt{\frac{d}{2}}\kappa + \gamma G_1^k\Big)\Big)\Big] 
\\
&\quad-2\sum_{\ell=2}^L\gamma \Esp\Big[\Big(G_1^{v^\star}-G_1^v \erf \Big(\lambda \Big(\sqrt{\frac{d}{2}}\kappa + \gamma G_1^k\Big)\Big)\Big)G_{\ell}^v \erf(\lambda G_{\ell}^k)\Big] \\
&\quad+ \sum_{\ell,m =2}^L \Esp[G_{\ell}^v \erf(\lambda G_{\ell}^k) G_m^v \erf(\lambda G_m^k)] \, .
\end{align*}
We address each term in this sum separately. 
\begin{itemize}
    \item Since $G_1^{v^\star} \sim \cN(0,1)$, $\gamma^2 \Esp \left[(G_1^{v^\star})^2\right] = \gamma^2$. 
    \item Since $\begin{pmatrix}
        G_1^{v^\star} \\ G_1^v 
    \end{pmatrix}
    \sim \cN\left(0, \begin{pmatrix}
        1 & \nu  \\
        \nu & 1  \\
    \end{pmatrix}\right)$ is independent from $G_1^k \sim \cN(0,1)$, we have 
    \begin{align*}
         - 2 \gamma^2\Esp \Big[G_1^{v^\star} G_1^v \erf \Big(\lambda \Big(\sqrt{\frac{d}{2}}\kappa + \gamma G_1^k\Big)\Big)\Big] &=  - 2 \gamma^2\Esp \Big[G_1^{v^\star} G_1^v \Big]\Esp \Big[\erf \Big(\lambda \Big(\sqrt{\frac{d}{2}}\kappa + \gamma G_1^k\Big)\Big)\Big] \\
         &= - 2 \gamma^2 \nu \Esp \Big[\erf \Big(\lambda \Big(\sqrt{\frac{d}{2}}\kappa + \gamma G_1^k\Big)\Big)\Big] \ .
    \end{align*}
    Finally, using Lemma \ref{lem:technical_results}$(ii)$, we obtain 
    \begin{align*}
         - 2 \gamma^2\Esp \Big[G_1^{v^\star} G_1^v \erf \Big(\lambda \Big(\sqrt{\frac{d}{2}}\kappa + \gamma G_1^k\Big)\Big)\Big] = - 2 \gamma^2 \nu \erf\Big(\lambda \sqrt{\frac{d}{2}} \frac{\kappa}{\sqrt{1 +2\lambda^2 \gamma^2}}\Big) \ .
         \end{align*}
         \item Since $G_1^v, G_1^k \sim_{\text{i.i.d.}} \cN(0,1)$, we have 
         \begin{align*}
             \gamma^2\Esp \Big[ (G_1^v)^2 \erf^2 \Big(\lambda \Big(\sqrt{\frac{d}{2}}\kappa + \gamma G_1^k\Big)\Big)\Big] &= \gamma^2\Esp \Big[ (G_1^v)^2 \Big] \Esp \Big[ \erf^2 \Big(\lambda \Big(\sqrt{\frac{d}{2}}\kappa + \gamma G_1^k\Big)\Big)\Big] \\
             &=\gamma^2 \Esp \Big[ \erf^2 \Big(\lambda \Big(\sqrt{\frac{d}{2}}\kappa + \gamma G_1^k\Big)\Big)\Big] \, .
         \end{align*}
         Using the definition of $\zeta$ in Eq.~\eqref{eq:def-zeta}, we have 
         \begin{align*}
             \gamma^2\Esp \Big[ (G_1^v)^2 \erf^2 \Big(\lambda \Big(\sqrt{\frac{d}{2}}\kappa + \gamma G_1^k\Big)\Big)\Big] &= \gamma^2\Esp \Big[ (G_1^v)^2 \Big] \Esp \Big[ \erf^2 \Big(\lambda \Big(\sqrt{\frac{d}{2}}\kappa + \gamma G_1^k\Big)\Big)\Big] \\
             &=\gamma^2 \zeta\Big(\lambda \sqrt{\frac{d}{2}}\kappa, \lambda^2 \gamma^2 \Big) \, .
         \end{align*}
         \item For $\ell=2, \dots, L$, $(G_1^{v^\star}, G_1^v, G_1^k)$, $G_{\ell}^v$, and $G_{\ell}^k$ are independent. Thus 
         \begin{align*}
             &\Esp\Big[\Big(G_1^{v^\star}-G_1^v \erf \Big(\lambda \Big(\sqrt{\frac{d}{2}}\kappa + \gamma G_1^k\Big)\Big)\Big)G_{\ell}^v \erf(\lambda G_{\ell}^k)\Big] \\
             &\qquad= \Esp\Big[\Big(G_1^{v^\star}-G_1^v \erf \Big(\lambda \Big(\sqrt{\frac{d}{2}}\kappa + \gamma G_1^k\Big)\Big)\Big) \Big] \Esp\Big[G_{\ell}^v\Big] \Esp \Big[ \erf(\lambda G_{\ell}^k)\Big] = 0 \, ,
         \end{align*}
         where in the last step we use $\Esp[G_{\ell}^v] = 0$. 
         \item Finally, to tackle the last term, we address the cases $\ell \neq m$ and $\ell = m$ separately. If $\ell \neq m$, as $G_{\ell}^v, G_{\ell}^k, G_m^v$, and $G_m^k$ are independent, we have 
         \begin{align*}
              \Esp[G_{\ell}^v \erf(\lambda G_{\ell}^k) G_m^v \erf(\lambda G_m^k)] =  \Esp[G_{\ell}^v] \Esp[ \erf(\lambda G_{\ell}^k)]\Esp[ G_m^v] \Esp[ \erf(\lambda G_m^k)] = 0 \, . 
         \end{align*}
         If $\ell = m$, as $G_{\ell}^v, G_{\ell}^k \sim_{\text{i.i.d.}} \cN(0,1)$, we have
         \begin{align*}
              \Esp[(G_{\ell}^v)^2 \erf^2(\lambda G_{\ell}^k) ] = \Esp[(G_{\ell}^v)^2] \Esp[ \erf^2(\lambda G_{\ell}^k) ] = \zeta(0, \lambda^2) \, . 
         \end{align*}
\end{itemize}
Putting together these computations, we obtain
\begin{align*}
    \cR^<(\kappa, \nu, 0,0,0) &= \varepsilon^2 + \gamma^2  - 2 \gamma^2 \nu \erf\Big(\lambda \sqrt{\frac{d}{2}} \frac{\kappa}{\sqrt{1 +2\lambda^2 \gamma^2}}\Big) + \gamma^2 \zeta\Big(\lambda \sqrt{\frac{d}{2}}\kappa, \lambda^2 \gamma^2 \Big)
\\
&\quad+ (L-1)\zeta(0, \lambda^2) \, .
\end{align*}
This proves Lemma \ref{lem:risk_on_M}. Taking $\kappa=\nu=1$ proves Theorem \ref{theo:optimalrisk}.

\subsection{Proof of Corollary \ref{cor:oracle-asymptotic-bayes}}

Recall that, according to Theorem \ref{theo:optimalrisk},
\begin{align*}
    \mathcal{R}_\lambda(k^\star, v^\star) &= \gamma^2- 2 \gamma^2 \, \erf \bigg(\lambda \sqrt{\frac{d}{2(1+2\lambda^2\gamma^2)}} \bigg)  +  \gamma^2 \zeta\Big(\lambda \sqrt{\frac{d}{2}} ,\lambda^2 \gamma^2\Big)   +   (\length -1)  \zeta(0, \lambda^2)  + \varepsilon^2 \, ,
\end{align*}
where, for $t,\gamma\in \mathbb{R}$,
\begin{equation*}
        \zeta(t,\gamma^2) := \Esp \left[ \erf^2(t+\gamma G)\right] \, , \qquad  G\sim\mathcal{N}(0, 1) \, .
\end{equation*}
We compute the limit of each term separately. First, we have
\begin{equation}    \label{eq:tech-1}
\lambda \sqrt{\frac{d}{2(1+2\lambda^2\gamma^2)}} \sim \frac{\lambda \sqrt{d}}{\sqrt{2}} \xrightarrow{d \to \infty} \infty \, .    
\end{equation}
Therefore, the second term of $\mathcal{R}_\lambda(k^\star, v^\star)$ tends to $-2\gamma^2$. To handle the third term, note by Jensen's inequality that
\[
1\geq \zeta\Big(\lambda \sqrt{\frac{d}{2}} ,\lambda^2 \gamma^2\Big) = \Esp \left[ \erf^2\Big(\lambda \sqrt{\frac{d}{2}}+\lambda\gamma G\Big)\right] \geq \Esp \left[ \erf\Big(\lambda \sqrt{\frac{d}{2}}+\lambda\gamma G\Big)\right]^2 \, .
\]
Thus, by Lemma \ref{lem:technical_results}$(ii)$,
\[
1 \geq \zeta\Big(\lambda \sqrt{\frac{d}{2}} ,\lambda^2 \gamma^2\Big) \geq \erf^2 \Big( \lambda \sqrt{\frac{d}{2(1+2\lambda^2\gamma^2)}} \Big) \to 1 \, ,
\]
where we used \eqref{eq:tech-1}.
Thus the third term of $\mathcal{R}_\lambda(k^\star, v^\star)$ converges to $\gamma^2$.
As for the fourth term, observe by Lemma~\ref{lem:erf} that
\[
\erf^2(u) \leq \frac{4}{\pi} u^2 \, ,
\]
hence
\[
0 \leq \zeta(0, \lambda^2) \leq \frac{4}{\pi} \lambda^2 \Esp[G^2] = \frac{4}{\pi} \lambda^2 \, .
\]
Since $\lambda \sqrt{L} \to 0$, we get
\[
(\length -1)  \zeta(0, \lambda^2) = \mathcal{O}(\lambda^2 L) = o(1) \, .
\]
Putting everything together, we obtain
\[
\mathcal{R}_\lambda(k^\star, v^\star) \xrightarrow{d \to \infty} \gamma^2 - 2\gamma^2 + \gamma^2 + 0 + \varepsilon^2 = \varepsilon^2 \, .
\]
Since we already know by \eqref{eq:bayes-min} that the Bayes risk is lower-bounded by $\varepsilon^2$, this proves that the Bayes risk is asymptotically \textit{equal} to $\varepsilon^2$, and that the oracle predictor is asymptotically Bayes optimal.

\subsection{Proof of Proposition \ref{prop:bayes_risk_linear_model}}

Let us first introduce a useful notation for the proof. If $M$ is a block matrix, we denote by $M_{[ij]}$ its $(i,j)$-th block, and likewise, if $u$ is a block vector, we denote by $u_{[j]}$ its $j$-th block. 
Next, note that 
\begin{align*}
\Esp [Y^2] &= \varepsilon^2 + \Esp \left[ ((v^\star)^\top X_{J_0})^2 \right] \\
&= \varepsilon^2 + \gamma^2 \|v^\star \|_2^2 \\
&= \varepsilon^2 + \gamma^2,
\end{align*}
since $\|v^\star \|_2^2=1$. Recall that
\[
\beta^\star \in \argmin_{\beta \in \R^{dL}}\Esp\Big[(Y-(X_1^\top, \hdots , X_L^\top)\beta)^2\Big]
\]
is the optimal linear predictor. The classical formula for linear regression shows that
\begin{align*}
    \beta^\star &= \left(\Esp \left[ 
    \begin{pmatrix}
        X_1 \\
        \vdots \\
        X_\length
    \end{pmatrix} (X_1^\top,  \hdots, X_\length^\top)
    \right] \right)^{-1} \Esp\left[ ((v^\star)^\top X_{J_0} + \xi)
    \begin{pmatrix}
        X_1 \\
        \vdots \\
        X_\length
    \end{pmatrix}
    \right].
\end{align*}
On the one hand, let 
\[
M = \begin{pmatrix}
        X_1 \\
        \vdots \\
        X_\length
    \end{pmatrix} (X_1^\top,  \hdots, X_\length^\top) \, .
\]
Then $\Esp[M] = \Esp[\Esp[M | J_0]]$, and $\Esp[M | J_0]$ is a block-diagonal matrix, where, for $j,j' \in \{1, \dots, L\}$,
\[
\Esp[M | J_0=j]_{[j', j']} = \delta_{j \neq j'} I_d + \delta_{j = j'} (\gamma^2 I_d + \frac{d}{2} k^\star (k^\star)^\top) \, .
\]
Thus
\[
\Esp[M]_{[j', j']} = (1- p_{j'}) I_d + p_{j'} (\gamma^2 I_d + \frac{d}{2} k^\star (k^\star)^\top) = I_d + p_{j'}(\gamma^2 - 1) I_d  + p_{j'} \frac{d}{2} k^\star (k^\star)^\top \, .
\]
On the other hand, let
\[
u = ((v^\star)^\top X_{J_0} + \xi)
    \begin{pmatrix}
        X_1 \\
        \vdots \\
        X_\length 
    \end{pmatrix} \, .
\]
Then
\begin{align*}
    \Esp\left[u\right] &= \Esp\left[ 
    \begin{pmatrix}
        X_1 \\
        \vdots \\
        X_\length 
    \end{pmatrix} X_{J_0}^\top 
    \right] v^\star = \Esp\left[ \Esp\left[
    \begin{pmatrix}
        X_1 \\
        \vdots \\
        X_\length 
    \end{pmatrix} X_{J_0}^\top 
    \right] \middle\vert J_0 \right]v^\star \\
    &= \begin{pmatrix}
        p_1 (\gamma^2    I_d &+ & \frac{d}{2} k^\star (k^\star)^\top) \\
        &\vdots & \\
        p_L (\gamma^2    I_d &+& \frac{d}{2} k^\star (k^\star)^\top)
    \end{pmatrix} v^\star  = \gamma^2 \begin{pmatrix}
        p_1 v^\star \\
        \vdots \\
        p_L v^\star
    \end{pmatrix},
\end{align*}
since, by Assumption \ref{hyp:orthogonality}, $k^{\star \top}v^{\star}=0$.

Since $\Esp[M]$ is a block-diagonal matrix and $\Esp[u]$ is a block vector, we get by standard computation rules for block matrices
\[
\beta^\star_{[j]} = (\Esp[M]^{-1} \Esp[u])_{[j]} = \Esp[M]_{[j, j]}^{-1} \Esp[u]_{[j]}  = \Big(I_d + p_{j}(\gamma^2 - 1) I_d  + p_{j} \frac{d}{2} k^\star (k^\star)^\top\Big)^{-1} \gamma^2 p_j v^\star \, .
\]
Recall the Sherman-Morrison formula \citep[][Section 2.7.1]{press2007numerical}, which states that for any vectors $u,v \in \mathbb R^d$, $(I_d+uu^\top)^{-1}v = \left(I_d-uu^\top/(1+u^\top u) \right)v$.
Applying this formula with orthogonal vectors, we obtain
\begin{align*}
    \beta^\star_{[j]} &= \left(1 + p_{j}(\gamma^2 - 1) \right)^{-1} \gamma^2 p_j v^\star
    = \frac{\gamma^2 p_j}{1 + p_{j}(\gamma^2 - 1)} v^\star \, ,
\end{align*}
which shows the first formula of the proposition.
Finally, the risk associated with the optimal linear predictor $(X_1^\top , \hdots , X_L^\top) \mapsto (X_1^\top , \hdots , X_L^\top) \beta^{\star}$ is given by
\begin{align}
    \mathcal{R}(\beta^\star) &=
    \Esp [Y^2] - \Esp[Y(X_1^\top \hdots X_L^\top) \beta^\star] \nonumber \\ 
    &= \varepsilon^2 + \gamma^2 - \gamma^2 \cdot  \begin{pmatrix}
        p_1 (v^\star)^\top, &
        \dots, &
        p_L (v^\star)^\top
    \end{pmatrix} 
    \begin{pmatrix}
        \beta^\star_{[1]} \\
        \vdots \\
        \beta^\star_{[L]}
    \end{pmatrix} \nonumber \\
    &= \varepsilon^2 + \gamma^2 - \gamma^4 \sum_{j=1}^L \frac{p_j^2}{1 + p_{j}(\gamma^2 - 1)} \, . \label{eq:aux1} %
\end{align}
This shows the formula for the risk given in the Proposition. To obtain the last bound, observe that, if $\gamma^2 \geq 1$, we have $1 + p_j(\gamma^2 -1) \geq 1$. If $\gamma^2 \leq 1$, since $p_j \leq 1$, we have $1 + p_j(\gamma^2 -1) \geq 1 + (\gamma^2 -1) = \gamma^2$. Thus we obtain $1 + p_j(\gamma^2 -1) \geq \min(1, \gamma^2)$. Therefore,
\begin{align*}
    \mathcal{R}(\beta^\star) &\geq \varepsilon^2 + \gamma^2 - \max(\gamma^4, \gamma^2) \sum_{j=1}^L {p_j^2} \\
    &\geq \varepsilon^2 + \gamma^2 - \max(\gamma^4, \gamma^2) \sum_{j=1}^L {p_j} \cdot  \max_{j=1, \dots, L} {p_j} \\
    &\geq \varepsilon^2 + \gamma^2 - (\gamma^4 + \gamma^2) \max_{j=1, \dots, L} {p_j} \\
    &\geq \varepsilon^2 + \gamma^2 - \gamma^2 (\gamma^2 + 1) \max_{j=1, \dots, L} {p_j} \, .
\end{align*}

When all $p_j$ are equal to $1/L$, all terms in the sum are equal, and Eq.~\eqref{eq:aux1} simplifies to 
\[
\mathcal{R}(\beta^\star) = \varepsilon^2 + \gamma^2 - L \gamma^4 \frac{\frac{1}{L^2}}{1 + \frac{1}{L}(\gamma^2 - 1)} = \varepsilon^2 + \gamma^2 - \frac{\gamma^4}{L + \gamma^2 - 1} \, .
\]

\subsection{Proof of Lemma \ref{lem:invariant_manifold}}

As a first step in the proof, we prove the next lemma, which is the key towards the invariance property we are aiming at, in that it shows that, for a point on the manifold $\cM$ (defined by $\theta = \eta = \rho = 0$), the gradient of the risk does not ``push'' the point outside of the manifold.
Its proof leverages the expression of the risk as a function of five parameters derived in the previous section

\begin{lemma}
\label{lem:vanishing-derivatives}
    At any point $(\kappa, \nu, \theta, \eta, \rho)$ such that $\theta = \eta = \rho = 0$, we have $\partial_\theta \cR^< = \partial_\eta \cR^< = \partial_\rho \cR^< = 0$. 
\end{lemma}

\begin{proof}
    We use Eq.~\eqref{eq:risk-5-params-1}--\eqref{eq:risk-5-params-2} and change signs in the square function: 
    \begin{align*}
    &\cR^<(\kappa, \nu, \theta, \eta, \rho) \\
    &\quad= \Esp\Big[\Big(\gamma G_1^{v^\star} + \xi - \Big(\sqrt{\frac{d}{2}}\theta + \gamma G_1^v\Big)\erf\Big(\lambda\Big(\sqrt{\frac{d}{2}}\kappa + \gamma G_1^k\Big)\Big)-\sum_{\ell=2}^L G_{\ell}^v \erf(\lambda G_{\ell}^k)\Big)^2\Big] \\
    &\quad= \Esp\Big[\Big(\gamma (-G_1^{v^\star}) - \xi - \Big(\sqrt{\frac{d}{2}}(-\theta) + \gamma (-G_1^v)\Big)\erf\Big(\lambda\Big(\sqrt{\frac{d}{2}}\kappa + \gamma G_1^k\Big)\Big)\\
    &\qquad-\sum_{\ell=2}^L (-G_{\ell}^v) \erf(\lambda G_{\ell}^k)\Big)^2\Big] \, ,
\end{align*}
where
\begin{align*}
    \begin{pmatrix}
        G_1^{v^\star} \\ G_1^v \\ G_1^k
    \end{pmatrix}
    , \dots, 
    \begin{pmatrix}
        G_L^{v^\star} \\ G_L^v \\ G_L^k
    \end{pmatrix}
    \underset{\text{i.i.d.}}{\sim} \cN\left(0, \begin{pmatrix}
        1 & \nu & \eta \\
        \nu & 1 & \rho \\
        \eta & \rho & 1 
    \end{pmatrix}\right) \, , \qquad \xi \sim \cN(0, \varepsilon^2) \, .
\end{align*}
Thus 
\begin{align*}
    \begin{pmatrix}
        -G_1^{v^\star} \\ -G_1^v \\ G_1^k
    \end{pmatrix}
    , \dots, 
    \begin{pmatrix}
        -G_L^{v^\star} \\ -G_L^v \\ G_L^k
    \end{pmatrix}
    \underset{\text{i.i.d.}}{\sim} \cN\left(0, \begin{pmatrix}
        1 & \nu & -\eta \\
        \nu & 1 & -\rho \\
        -\eta & -\rho & 1 
    \end{pmatrix}\right) \, , \qquad -\xi \sim \cN(0, \varepsilon^2) \, .
    \end{align*}
    As a consequence,
    \begin{align*}
        \cR^<(\kappa, \nu, \theta, \eta, \rho) = \cR^<(\kappa, \nu, -\theta, -\eta, -\rho) \, .
    \end{align*}
    Taking the partial derivative in $\theta$, we are led to
    \begin{equation*}
        \partial_\theta \cR^<(\kappa, \nu, \theta, \eta, \rho) = - \partial_\theta \cR^<(\kappa, \nu, -\theta, -\eta, -\rho) \, .
    \end{equation*}
    At a point such that $\theta = \eta = \rho = 0$, this gives $\partial_\theta \cR^<(\kappa, \nu, 0,0,0) = - \partial_\theta \cR^<(\kappa, \nu, 0,0,0)$
    and thus $\partial_\theta \cR^<(\kappa, \nu, 0,0,0) = 0$. The proof for the other two derivatives $\partial_\eta \cR$, $\partial_\rho \cR$ is identical. 
\end{proof}

We now complete the proof of Lemma \ref{lem:invariant_manifold}. By the chain rule for total derivatives applied to $R(k,v) = \cR^<(\kappa, \nu, \theta, \eta, \rho)$, and then by Lemma~\ref{lem:vanishing-derivatives}, on the manifold $\cM$, we have
\begin{align}   \label{eq:tech-2}
    \nabla_k \mathcal{R}  &= (\partial_\kappa \mathcal{R}^< ) k^\star +  (\partial_\eta \mathcal{R}^< ) v^\star + (\partial_\rho \mathcal{R}^< ) v = (\partial_\kappa \mathcal{R}^< ) k^\star \, ,
\end{align}
and, similarly,
\begin{align}   \label{eq:tech-3}
    \nabla_v \mathcal{R} &= (\partial_\nu \mathcal{R}^<) v^\star + (\partial_\theta \mathcal{R}^<) k^\star + (\partial_\rho \mathcal{R}) k = (\partial_\nu \mathcal{R}^<) v^\star \, .
\end{align}
Recall the formulas for the PGD updates
    \begin{align*}
         k_{t+1} &= \mathrm{Proj}_{\mathbb{S}^{d-1}}(k_t - \alpha (I-k_t k_t^\top) \nabla_k \mathcal{R}(k_t,v_t)) = \frac{k_t - \alpha (I-k_t k_t^\top) \nabla_k \mathcal{R}(k_t,v_t)}{\left\| k_t - \alpha (I-k_t k_t^\top) \nabla_k \mathcal{R}(k_t,v_t) \right\|_2}, \\
        v_{t+1} &= \mathrm{Proj}_{\mathbb{S}^{d-1}}(v_t - \alpha (I-v_t v_t^\top) \nabla_v \mathcal{R}(k_t,v_t)) = \frac{v_t - \alpha (I-v_t v_t^\top) \nabla_v \mathcal{R}(k_t,v_t)}{\left\|v_t - \alpha (I-v_t v_t^\top) \nabla_v \mathcal{R}(k_t,v_t) \right\|_2} \, .
\end{align*}
Let $c_k = \left\| k_t - \alpha (I-k_t k_t^\top) \nabla_k \mathcal{R}(k_t,v_t) \right\|_2$ and $c_v = \left\|v_t - \alpha (I-v_t v_t^\top) \nabla_v \mathcal{R}(k_t,v_t) \right\|_2$. 
Then, if $(k_t, v_t) \in \mathcal{M}$,
\begin{align*}
(v^\star)^\top k_{t+1} &= \frac{(v^\star)^\top k_t - \alpha  (v^\star)^\top (I-k_t k_t^\top) (\partial_\kappa \mathcal{R}^<(\kappa_t,\nu_t))  k^\star}{c_k} = 0, \\
    (k^\star)^\top v_{t+1} &= \frac{(k^\star)^\top v_t - \alpha  (k^\star)^\top (I-v_t v_t^\top) (\partial_\nu \mathcal{R}^<(\kappa_t,\nu_t)) v^\star}{c_v} = 0,
    \end{align*}
    and
    \begin{align*}
    v_{t+1}^\top k_{t+1} &= \frac{v_t^\top k_t - \alpha (\partial_\nu \mathcal{R}^<) ((I-v_tv_t^\top) v^\star)^\top k_t - \alpha (\partial_\kappa \mathcal{R}^< ) ((I-k_tk_t^\top)k^\star)^\top v_t}{c_v c_k}\\
    &\qquad + \frac{\alpha^2 (\partial_\kappa \mathcal{R}^<) (\partial_\nu \mathcal{R}^<) ((I-k_tk_t^\top)k^\star)^\top (I-v_tv_t^\top)v^\star}{c_v c_k} = 0 \, ,
\end{align*}
where we have omitted the dependence of $(\partial_\kappa \mathcal{R}^<)$ and $(\partial_\nu \mathcal{R}^<)$ in $(\kappa_t,\nu_t)$ in the last expression for the ease of readability.
Note that the last term is equal to zero since
\[
((I-k_tk_t^\top)k^\star)^\top (I-v_tv_t^\top)v^\star  = 
(k^\star -\kappa_t k_t)^\top (v^\star -\nu_t v_t) = 0 \, .
\]
This shows that $(k_{t+1}, v_{t+1}) \in \mathcal{M}$.

\subsection{Proof of Lemma \ref{lemma:2d-reformulation}}

By definition of the PGD iterates and by \eqref{eq:tech-2}--\eqref{eq:tech-3}, one has
\begin{align*}
    \kappa_{t+1} &= k_{t+1}^\top k^\star  = \frac{\kappa_t - \alpha \partial_\kappa \mathcal{R}^<(\kappa_t,\nu_t) (k^\star)^\top (I-k_t k_t^\top) k^\star}{\sqrt{1+\alpha^2 (\partial_\kappa \mathcal{R}^<)^2 \| (I-k_t k_t^\top) k^\star \|_2^2}}= \frac{\kappa_t - \alpha (\partial_\kappa \mathcal{R}^<)(1-\kappa_t^2)}{\sqrt{1+\alpha^2(\partial_\kappa \mathcal{R}^<)^2 (1-\kappa_t^2)}},\\
    \nu_{t+1} &= v_{t+1}^\top v^\star   = \frac{\nu_t - \alpha \partial_\nu \mathcal{R}^<(\kappa_t,\nu_t) (v^\star)^\top (I-v_t v_t^\top) v^\star}{\sqrt{1+\alpha^2 (\partial_\nu \mathcal{R}^<)^2 \| (I-v_tv_t^\top) v^\star \|_2^2}}= \frac{\nu_t - \alpha (\partial_\nu \mathcal{R}^<)(1-\nu_t^2)}{\sqrt{1+\alpha^2(\partial_\nu \mathcal{R}^<)^2 (1-\nu_t^2)}},
\end{align*}
where we have used the Pythagorean theorem and the idempotent property of projection matrices for the denominator.

\subsection{Proof of Proposition \ref{prop:analysis-pgd}}

In this proof, $C$ denotes a constant that does not depend on the step $t$ nor on the step size $\alpha$, and which may vary from line to line.
First note that the risk $\mathcal{R}^<$ is $C^\infty$ on the compact set $[-1, 1]^2$. In particular, it is a $\Lambda$-smooth function for some $\Lambda>0$, in the sense that its gradient is $\Lambda$-Lipschitz continuous. Thus
\begin{align*}
    \mathcal{R}^< (\kappa_{t+1}, \nu_{t+1})  &\leq \mathcal{R}^< (\kappa_{t}, \nu_{t}) + (\nabla \mathcal{R}^<(\kappa_{t}, \nu_{t}))^\top 
    \begin{pmatrix}
        \kappa_{t+1} - \kappa_t \\
        \nu_{t+1} - \nu_t
     \end{pmatrix} + \frac{\Lambda}{2} \left\|
     \begin{pmatrix}
        \kappa_{t+1} - \kappa_t \\
        \nu_{t+1} - \nu_t
     \end{pmatrix}\right\|_2^2 \, ,
\end{align*}
i.e.,
\begin{align}
    \notag
    \mathcal{R}^< &(\kappa_{t+1}, \nu_{t+1}) - \mathcal{R}^< (\kappa_{t}, \nu_{t}) \\ 
    &\leq (\partial_\kappa \mathcal{R}^<) (\kappa_{t+1}-\kappa_t) + (\partial_\nu\mathcal{R}^<) (\nu_{t+1}-\nu_t) + \frac{\Lambda}{2} \left[(\kappa_{t+1}-\kappa_t)^2   +  (\nu_{t+1}-\nu_t)^2\right]. \label{eq:risk_inequality_aux}
\end{align}
Our goal in the following computations is to derive an inequality of the form
\begin{align*}
    \mathcal{R}^< (\kappa_{t+1}, \nu_{t+1}) - \mathcal{R}^< (\kappa_{t}, \nu_{t}) 
    &\leq - \alpha (\partial_\kappa \mathcal{R}^<)^2 (1-\kappa_t^2) - \alpha (\partial_\nu \mathcal{R}^<)^2 (1-\nu_t^2) \\
    &\qquad + C  \alpha^2 (\partial_\kappa \mathcal{R}^<)^2 (1-\kappa_t^2) + C \alpha^2 (\partial_\nu \mathcal{R}^<)^2 (1-\nu_t^2) \, ,
\end{align*}
which shall give us a descent lemma for $\alpha$ small enough.
To this aim, observe that, by definition of the iterates $(\kappa_t,\nu_t)$ given by \eqref{eq:dynamical_system_kappa_nu-1}--\eqref{eq:dynamical_system_kappa_nu-2}, one has
\begin{align}
    \kappa_{t+1}-\kappa_t &= \left[ \frac{1}{\sqrt{1+\alpha^2(\partial_\kappa \mathcal{R}^<)^2 (1-\kappa_t^2)}} -1 \right] \kappa_t -\frac{\alpha (\partial_\kappa \mathcal{R}^<) (1-\kappa_t^2)}{\sqrt{1+\alpha^2(\partial_\kappa \mathcal{R}^<)^2 (1-\kappa_t^2)}} \label{eq:formula-kappa-t+1-kappa-t} \\
    &= - \alpha (\partial_\kappa \mathcal{R}^<) (1-\kappa_t^2) \nonumber \\
    &\qquad + \left[ \frac{1}{\sqrt{1+\alpha^2(\partial_\kappa \mathcal{R}^<)^2 (1-\kappa_t^2)}} -1 \right] (\kappa_t - \alpha (\partial_\kappa \mathcal{R}^<) (1-\kappa_t^2)) \nonumber \, . %
\end{align}
As a consequence,
\begin{align}
    |\kappa_{t+1}-\kappa_t + \alpha (\partial_\kappa \mathcal{R}^<) (1-\kappa_t^2)| &\leq \Big|\frac{1}{\sqrt{1+\alpha^2(\partial_\kappa \mathcal{R}^<)^2 (1-\kappa_t^2)}} -1 \Big| |\kappa_t - \alpha (\partial_\kappa \mathcal{R}^<) (1-\kappa_t^2)| \nonumber \\
    &\leq \alpha^2 (\partial_\kappa \mathcal{R}^<)^2 (1-\kappa_t^2) |\kappa_t - \alpha (\partial_\kappa \mathcal{R}^<) (1-\kappa_t^2)| \nonumber \\
    &\leq C \alpha^2 (\partial_\kappa \mathcal{R}^<)^2 (1-\kappa_t^2) \, , \label{eq:tech-4}
\end{align}
where the second inequality holds by Lemma \ref{lemma:simple-bounds} and the last bound holds since the function $(\kappa, \nu) \mapsto |\kappa - \alpha (\partial_\kappa \mathcal{R}^<(\kappa, \nu)) (1-\kappa^2)|$ is uniformly bounded for all $\alpha \leq 1$.
This bound has two implications. First, 
\begin{align}
(\partial_\kappa \mathcal{R}^<) (\kappa_{t+1}-\kappa_t) + \alpha (\partial_\kappa \mathcal{R}^<)^2 (1-\kappa_t^2)
&= (\partial_\kappa \mathcal{R}^<) ((\kappa_{t+1}-\kappa_t) + \alpha (\partial_\kappa \mathcal{R}^<) (1-\kappa_t^2)) \nonumber \\
&\leq |\partial_\kappa \mathcal{R}^<| |\kappa_{t+1}-\kappa_t + \alpha (\partial_\kappa \mathcal{R}^<) (1-\kappa_t^2)|   \nonumber \\
&\leq C \alpha^2 (\partial_\kappa \mathcal{R}^<)^2 (1-\kappa_t^2) \, , \label{eq:estimate-1}
\end{align}
where we use the fact that $|\partial_\kappa \mathcal{R}^<|$ is bounded, and the bound \eqref{eq:tech-4}.
Second, since the square function is Lipschitz on compact sets, we have
\begin{align*}
|(\kappa_{t+1}-\kappa_t)^2 - (\alpha (\partial_\kappa \mathcal{R}^<) (1-\kappa_t^2))^2| \leq C  \alpha^2 (\partial_\kappa \mathcal{R}^<)^2 (1-\kappa_t^2) \, .
\end{align*}
Thus
\begin{align}   
(\kappa_{t+1}-\kappa_t)^2 &\leq \alpha^2 (\partial_\kappa \mathcal{R}^<)^2 (1-\kappa_t^2)^2 + C  \alpha^2 (\partial_\kappa \mathcal{R}^<)^2 (1-\kappa_t^2) \nonumber \\
&\leq C  \alpha^2 (\partial_\kappa \mathcal{R}^<)^2 (1-\kappa_t^2) \, . \label{eq:estimate-2}
\end{align}
We also obtain analogous bounds to \eqref{eq:estimate-1}–\eqref{eq:estimate-2} for $\nu$, namely
\begin{align} \label{eq:estimate-3}
    (\partial_\nu \mathcal{R}^<) (\nu_{t+1}-\nu_t) + \alpha (\partial_\nu \mathcal{R}^<) (1-\nu_t^2) \leq C \alpha^2 (\partial_\nu \mathcal{R}^<)^2 (1-\nu_t^2) \, ,
\end{align}
and
\begin{equation}    \label{eq:estimate-4}
    (\nu_{t+1}-\nu_t)^2 \leq C  \alpha^2 (\partial_\nu \mathcal{R}^<)^2 (1-\nu_t^2) \, .
\end{equation}
Plugging the bounds \eqref{eq:estimate-1}--\eqref{eq:estimate-4} into Eq.~\eqref{eq:risk_inequality_aux}, we obtain the desired inequality
\begin{align*}
    \mathcal{R}^< (\kappa_{t+1}, \nu_{t+1}) - \mathcal{R}^< (\kappa_{t}, \nu_{t}) 
    &\leq - \alpha (\partial_\kappa \mathcal{R}^<)^2 (1-\kappa_t^2) - \alpha (\partial_\nu \mathcal{R}^<)^2 (1-\nu_t^2) \\
    &\qquad + C  \alpha^2 (\partial_\kappa \mathcal{R}^<)^2 (1-\kappa_t^2) + C \alpha^2 (\partial_\nu \mathcal{R}^<)^2 (1-\nu_t^2) \, .
\end{align*}
By choosing the step size $\alpha \leq \frac{1}{2C}$, this ensures that
\begin{align*}
    \mathcal{R}^< (\kappa_{t+1}, \nu_{t+1}) - \mathcal{R}^< (\kappa_{t}, \nu_{t}) 
    &\leq - \frac{\alpha}{2} (\partial_\kappa \mathcal{R}^<)^2 (1-\kappa_t^2) - \frac{\alpha}{2} (\partial_\nu \mathcal{R}^<)^2 (1-\nu_t^2).
\end{align*}
This shows that the risk is decreasing along the PGD iterates. 
Next, introducing $\cR_{\min}^<=\min_{(\kappa,\nu)\in [0,1]^2} \cR^<(\kappa,\nu)$ and using a telescopic sum, we have,
for all $T\geqslant 0$,
\begin{align*}
    \mathcal{R}^< (\kappa_{0}, \nu_{0}) - \cR^<_{\min} 
    &\geq \mathcal{R}^< (\kappa_{0}, \nu_{0}) - \mathcal{R}^< (\kappa_T, \nu_T) \\
    &\geq \frac{\alpha}{2} \sum_{t=0}^{T-1} \left[(\partial_\kappa \mathcal{R}^<)^2 (1-\kappa_t^2) +(\partial_\nu \mathcal{R}^<)^2 (1-\nu_t^2)\right].
\end{align*}
Since the left-hand side is finite, and the terms of the sum are nonnegative, we conclude that the series converges as $T \to \infty$. In particular, the generic term $(\partial_\kappa \mathcal{R}^<)^2 (1-\kappa_t^2) +(\partial_\nu \mathcal{R}^<)^2 (1-\nu_t^2)$ of the series converges to 0 as $t \to \infty$. Therefore, the accumulation points $(\kappa_\infty, \nu_\infty)$ satisfy
\[
\left\{
\begin{array}{l}
\partial_\kappa \mathcal{R}^<(\kappa_\infty, \nu_\infty) =0 \quad \text{or} \quad \kappa_\infty^2=1 \\
\partial_\nu \mathcal{R}^<(\kappa_\infty, \nu_\infty) =0 \quad \text{or} \quad \nu_\infty^2=1.
\end{array}
\right.
\]
Inspecting identity \eqref{eq:formula-kappa-t+1-kappa-t}, we observe that the convergence of the general term also implies $\kappa_{t+1} - \kappa_t \to 0$. We obtain similarly that $\nu_{t+1} - \nu_t \to 0$.

\subsection{Proof of Proposition \ref{prop:critical-points}}

Recall that the risk in terms of $(\kappa, \nu)$ is given by
\[
{\cR^<(\kappa, \nu) }
 = \gamma^2- 2 \gamma^2 \nu \, \erf \left(\frac{\lambda \sqrt{d/2} \kappa}{\sqrt{1+2\lambda^2 \gamma^2}} \right)  +  \gamma^2 \zeta\Big(\lambda \sqrt{\frac{d}{2}} \, \kappa,\lambda^2 \gamma^2\Big)  +   (\length -1)  \zeta(0, \lambda^2)  + \varepsilon^2  \, .
\]
Then the gradients of $\cR^<$ are given by
\begin{align*}
\partial_\kappa \mathcal{R}^<&(\kappa, \nu) \\
   & = -2\gamma^2 \lambda \sqrt{\frac{d}{2(1+2\lambda^2\gamma^2)}} \erf' \bigg(\frac{\lambda \sqrt{d/2} \kappa}{\sqrt{1+2\lambda^2 \gamma^2}} \bigg)  
    \bigg(
    \nu - 
    \erf\bigg( 
\frac{\lambda \sqrt{d/2} \kappa}{\sqrt{(1+2\lambda^2\gamma^2)(1+ 4\lambda^2\gamma^2)}}
\bigg)
    \bigg) 
\end{align*}
and
\[
   \partial_\nu \mathcal{R}^<(\kappa, \nu) = -2 \gamma^2  \erf \bigg( \frac{\lambda \sqrt{d/2} \kappa}{\sqrt{1+2\lambda^2 \gamma^2}} \bigg) \, . \\
\]
Therefore, the solutions of the system \eqref{eq:formula-accumulation-points} satisfy
\[
\begin{cases}
   -\nu + \erf(c_1 \kappa) = 0 \textnormal{ or } \kappa = \pm 1 \\
   \kappa = 0 \textnormal{ or } \nu = \pm 1 \, , \\
\end{cases}
\]
with $c_1=\frac{\lambda}{\sqrt{4\lambda^2\gamma^2+1}} \sqrt{\frac{d}{2(1+2\lambda^2\gamma^2)}}$.
The solutions of this system are
\[
(\kappa, \nu) = (0,0) \text{ or } (\kappa, \nu) = (\pm 1,\pm 1) \, .
\]

\subsection{Proof of Proposition \ref{prop:minima_maxima_saddle}}

Since $\cR^<$ is a smooth function, the extrema of this function on $[-1,1]^2$ are either critical points (admitting null derivatives) or points on the boundary of the square $[-1,1]^2$. Starting with critical points, the only critical point is $(0, 0)$, and it is a saddle point. Indeed, the Hessian of $\cR^<$ at (0,0) is
\[
    H_{\cR^<}(0,0)
    = - \frac{4}{\sqrt{\pi}} \gamma^2 \lambda \sqrt{\frac{d}{2(1+2\lambda^2\gamma^2)}} \underbrace{\begin{pmatrix}
        c & 1 \\
        1 & 0
    \end{pmatrix}}_{:=M}\,
\]
where $c=-\frac{2\lambda}{\sqrt{\pi(4\lambda^2\gamma^2+1)}}\sqrt{\frac{d}{2(1+2\lambda^2\gamma^2)}}<0$. Then, as $\mathrm{det}(M)=-1$, the two eigenvalues of  $H_{\cR^<}(0,0)$ have opposite signs, $(0,0)$ is thus a saddle point. The extrema of $\cR^<$ must therefore be on the boundary of the square, which we examine next.

For any $(\kappa, \nu) \in (-1, 1)^2$, one has, by inspecting the signs of the gradients given in the proof of Proposition \ref{prop:critical-points},
\[
\cR^<(1,1) < \cR^<(\kappa, 1) < \cR^<(-1, 1) \,  \quad \textnormal{and}  \quad \cR^<(1,1) < \cR^<(1, \nu) < \cR^<(1, -1) \, .
\]
This shows that the minimum of $\cR^<$ on $\{(\kappa, 1), \kappa\in [-1,1]\} \cup \{(1, \nu), \nu\in[-1,1]\}$ is reached at $(1, 1)$, and the maximum is reached both at $(1,-1)$ and $(-1,1)$, since $\cR^<$ is even. Using again evenness of $\cR^<$, we conclude that the extrema of $\cR^<$ on the whole boundary of the square, and thus on the whole square, are the minimizers $(1, 1)$ and $(-1, -1)$, and the maximizers $(1,-1)$ and $(-1,1)$.

\subsection{Proof of Proposition \ref{prop:properties-g}}

We prove the statements of the proposition one by one.

\paragraph{The mapping $g$ is a local diffeomorphism around $(0,0)$, whose Jacobian matrix has one eigenvalue in $(0, 1)$ and one eigenvalue in $(1, \infty)$.} Consider the Taylor expansion of the first component $g(\kappa,\nu)_1$  of $g(\kappa,\nu)$. Since $\partial_\kappa \mathcal{R}^<(0, 0) = 0$, and $\mathcal{R}^<$ is smooth, letting $x = (\kappa, \nu)$, we have $(\partial_\kappa \mathcal{R}^<(\kappa,\nu))^2 = O(\|x\|^2)$. Thus,
\begin{align*}
g(\kappa,\nu)_1 &=
\frac{\kappa - \alpha (\partial_\kappa \mathcal{R}^<(\kappa,\nu))(1-\kappa^2)}{\sqrt{1+\alpha^2(\partial_\kappa \mathcal{R}^<(\kappa,\nu))^2 (1-\kappa^2)}} \\
&=
\frac{\kappa - \alpha (\partial_\kappa \mathcal{R}^<(\kappa,\nu))(1-\kappa^2)}{\sqrt{1+O(\|x\|^2) }} \\
&= (\kappa - \alpha (\partial_\kappa \mathcal{R}^<(\kappa,\nu))(1-\kappa^2)) \left( 1+O(\|x\|^2)\right) \\
&= \kappa - \alpha \partial_\kappa \mathcal{R}^<(\kappa,\nu) + O(\|x\|^2) \, .
\end{align*}
Proceeding similarly with the second component of $g$, we obtain that the Jacobian of $g$ at $(0,0)$ is given by
\[
J_g(0,0) = I_2 - \alpha H_{\cR^<}(0,0) = I_2 + \alpha \cdot \frac{4}{\sqrt{\pi}} \gamma^2 \lambda \sqrt{\frac{d}{2(1+2\lambda^2\gamma^2)}} \underbrace{\begin{pmatrix}
        c & 1 \\
        1 & 0
    \end{pmatrix}}_{=:M} \, ,
\]
where $c=-\frac{2\lambda}{\sqrt{\pi(4\lambda^2\gamma^2+1)}}\sqrt{\frac{d}{2(1+2\lambda^2\gamma^2)}}<0$. Since $\mathrm{det}(M)=-1$, one can choose $\alpha$ small enough so that one eigenvalue of $J_g(0, 0)$ is strictly between $0$ and $1$ and the other one is strictly larger than~$1$. Therefore, $J_g(0, 0)$ is invertible, showing that $g$ is a local diffeomorphism around $(0,0)$.

\paragraph{The mapping $g$ is differentiable on $[-1,1]^2$, and its Jacobian is not degenerate.}
The mapping $g$ is clearly differentiable as a composition of differentiable function. The more delicate part is to show that its Jacobian cannot be degenerate. 
To show this statement, observe first that, for $x \in [-1, 1]^2$, we may write $g(x) = x + \alpha h(x)$, where the first component of $h$ is given by
\begin{align*}
h(\kappa, \nu)_1 &= \frac{1}{\alpha} (g(\kappa, \nu)_1 - \kappa)     \\
&= \frac{1}{\alpha} \Big( \frac{\kappa - \alpha (\partial_\kappa \mathcal{R}^<(\kappa,\nu))(1-\kappa^2)}{\sqrt{1+\alpha^2(\partial_\kappa \mathcal{R}^<(\kappa,\nu))^2 (1-\kappa^2)}} - \kappa\Big)  \\
&= \underbrace{\frac{\kappa}{\alpha} \Big( \frac{1}{\sqrt{1+\alpha^2(\partial_\kappa \mathcal{R}^<(\kappa,\nu))^2 (1-\kappa^2)}} - 1\Big)}_{=:f^{(1)}_\alpha(\kappa, \nu)} - \underbrace{\frac{(\partial_\kappa \mathcal{R}^<(\kappa,\nu))(1-\kappa^2)}{\sqrt{1+\alpha^2(\partial_\kappa \mathcal{R}^<(\kappa,\nu))^2 (1-\kappa^2)}}}_{=:f^{(2)}_\alpha(\kappa, \nu)} \, .
\end{align*}
Let us prove that the gradient of $h(\kappa, \nu)_1$ is bounded uniformly over $\alpha \leq 1$. The uniform boundedness is clear for the gradient of $f^{(2)}_\alpha$, which writes as a composition of functions with uniformly bounded gradients for $\alpha \leq 1$.
Moving on to $f^{(1)}_\alpha$ and letting 
\begin{align*}
g: \left\lbrace
  \begin{array}{ll}
    [-1, 1] \times [0, B] &\to \R \\
    (a, b) &\mapsto \frac{a}{\alpha} \big( \frac{1}{\sqrt{1+\alpha^2 b}} - 1\big) \\
  \end{array}
  \right. \, , \quad
  B = \sup_{(\kappa, \nu) \in [-1, 1]^2} (\partial_\kappa \mathcal{R}^<(\kappa,\nu))^2 (1-\kappa^2) \, ,
\end{align*}
we observe that $f^{(1)}_\alpha$ is the composition of $g$ with a smooth function independent of $\alpha$. In particular, it suffices to show the uniform boundedness of $\nabla g$ to deduce the one of $\nabla f^{(1)}_\alpha$. 
We further have, by Lemma \ref{lemma:simple-bounds}, and for $\alpha \leq 1$,
\[
\big|\partial_a g(a, b)\big| = \frac{1}{\alpha} \Big| \frac{1}{\sqrt{1+\alpha^2 b}} - 1\Big| \leq \alpha b \leq B
\]
and
\[
\big|\partial_b g(a, b)\big| = \Big| - \frac{\alpha a}{2 (1+\alpha^2 b)^{3/2}} \Big| \leq \frac{\alpha}{2} \leq \frac{1}{2} \, .
\]
Therefore, the gradient of $h(\kappa, \nu)_1$ is bounded uniformly over $\alpha \leq 1$. Proceeding similarly with the gradient of $h(\kappa, \nu)_2$, we obtain that the Jacobian of $h(\kappa, \nu)$ is uniformly bounded over $\alpha \leq 1$. Recall now that $J_g(\kappa, \nu) = I_2 + \alpha J_h(\kappa, \nu)$. Therefore, taking $\alpha$ small enough, we obtain that the eigenvalues of $J_g$ have to be bounded away from zero.
 
\paragraph{The mapping $g$ is injective.}
The computation above shows that $h$ is $\beta$-Lipschitz continuous with $\beta$ independent of $\alpha$ (for $\alpha$ small enough). 
In particular we can choose $\alpha$ such that $\alpha<1/\beta$. Now,
 let $x \neq  y \in [-1, 1]^2$ be such that $g(x) = g(y)$. Then
    \[
    \|x-y\| \leq \alpha \|h(x) - h(y) \| \leq \alpha \beta \|x-y\| < \|x-y\| \, .
    \]
    This is a contradiction, showing that $g$ is injective.

\subsection{Proof of Proposition \ref{prop:no-conv-to-saddle}}

Recall that $(1, -1)$ and $(-1, 1)$ are maxima of the risk $\mathcal R^<$ on $[-1, 1]^2$ by Proposition \ref{prop:minima_maxima_saddle}, and that the value of the risk decreases along the iterates of PGD by Proposition \ref{prop:analysis-pgd}. Thus the only possible way to converge to these points is to start the dynamics from them. 

The case of the point $(0, 0)$ is more delicate. We apply the Center-Stable Manifold theorem \citep[][Theorem III.7]{shub1987global} to $g$, which is a local diffeomorphism  around $(0, 0)$ by Proposition \ref{prop:properties-g}. This guarantees the existence of a local center-stable manifold $W_{\text{loc}}^{\text{cs}}$, which verifies the following properties. First, its codimension is equal to the number of eigenvalues of $J_g(0, 0)$ of magnitude larger than $1$, that is, $1$, by Proposition \ref{prop:properties-g}. Hence it has Lebesgue measure zero. Second, there exists a neighborhood $B$ of $0$ such that $\bigcap_{t=0}^\infty g^{-t}(B) \subset W_{\text{loc}}^{\text{cs}}$. Then, let $W^s$ be the set of all $x$ which converge to $(0, 0)$ under the gradient map $g$, and take $x \in W^s$. Then there exists a $T$ such that $g^t(x) \in B$ for all $t \geq T$. This means that $g^T(x) \in \bigcap_{s=0}^\infty g^{-s}(B)$, and thus $g^T(x) \in W_{\text{loc}}^{\text{cs}}$. So, $x \in g^{-T}(W_{\text{loc}}^{\text{cs}})$. We have just shown that
\[
W^s \subset \bigcup_{T \geq 0} g^{-T}(W_{\text{loc}}^{\text{cs}}) \, .
\]
Finally, we prove that the pre-image of sets of measure zero by $g^T$ has measure zero for any $T \geq 0$. This shall conclude the proof of the result since countable unions of sets of measure zero have measure zero. To show this, note that $g$ is injective by Proposition \ref{prop:properties-g}, and therefore $g^T$ is injective too. This allows to define an inverse $g^{-T}$ of $g^T$ defined on the image of $g^T$, and the pre-image by $g^T$ of $W_{\text{loc}}^{\text{cs}}$ is exactly the image by $g^{-T}$ of $W_{\text{loc}}^{\text{cs}}$ (intersected with the domain of definition of $g^{-T}$). Furthermore, by Proposition \ref{prop:properties-g}, the Jacobian of $g^T$ is invertible. This guarantees that $g^{-T}$ is differentiable by the inverse function theorem. The conclusion follows by recalling that differentiable functions map sets of measure zero to sets of measure zero.

\section{Expression of the risk beyond the invariant manifold}
\label{proof:risk_on_sphere}

In this appendix, we provide an expression of the risk $\cR$ beyond the manifold $\cM$ that extends the one provided in Lemma \ref{lem:risk_on_M}. This result is not needed to prove Theorem \ref{thm:main}, and its proof is more involved that the one of Lemma \ref{lem:risk_on_M}. However, we provide it since it might be relevant to follow-up works that would study the dynamics if not initialized on the invariant manifold~$\cM$. It is also useful for the numerical simulations (see Appendix \ref{app:experimental-details}).

\begin{proposition} \label{prop:full-risk-expression}
We have the closed-form expression
\begin{align*} 
&\mathcal{R}_\lambda^<(\kappa, \nu, \theta, \eta, \rho) =\varepsilon^2 +\gamma^2 \\
& - 2 \gamma^2 \nu \erf \bigg(\frac{\lambda \sqrt{d/2} \kappa}{\sqrt{1+2\lambda^2 \gamma^2}} \bigg) -2 \lambda \gamma^2 \sqrt{\frac{d}{2}}  \eta \theta  \frac{1}{\sqrt{1+2\lambda^2\gamma^2}}\erf'\bigg( \frac{\lambda \sqrt{d/2} \kappa}{\sqrt{1+2\lambda^2 \gamma^2}}\bigg) \\
 & - \frac{2\lambda^2 \gamma^4 \eta \rho}{1+2\lambda^2\gamma^2}\erf''\bigg( \frac{\lambda \sqrt{d/2} \kappa}{\sqrt{1+2\lambda^2 \gamma^2}}\bigg)  + (\frac{d}{2} \theta^2 + \gamma^2) \zeta\Big(\lambda\sqrt{\frac{d}{2}} \, \kappa,\lambda^2\gamma^2\Big)  \\
 & + \sqrt{\frac{d}{2}} \left(\theta \rho - \frac{\lambda^2 \gamma^2 \rho^2 \kappa}{1 + 2 \lambda^2 \gamma^2} \right)  \frac{4 \lambda \gamma^2}{\sqrt{1+2\lambda^2\gamma^2}}    \erf\bigg( \frac{\lambda \sqrt{d/2} \, \kappa}{\sqrt{(1+4\lambda^2\gamma^2)(1+2\lambda^2\gamma^2)}}\bigg)   \erf'\bigg( \frac{\lambda \sqrt{d/2} \kappa}{\sqrt{1+2\lambda^2 \gamma^2}}\bigg) \\
& +  \frac{4 \lambda^2 \gamma^4 \rho^2}{\sqrt{\pi} \sqrt{1+4\lambda^2\gamma^2} (1+2\lambda^2\gamma^2)} \erf'\bigg(-\frac{\lambda\sqrt{d}\kappa}{\sqrt{1+4\lambda^2\gamma^2}} \bigg) \\
 &+ (\length -1) \left[\zeta(0, \lambda^2) + \frac{8 \lambda^2}{\pi\sqrt{1+4\lambda^2} (1 + 2 \lambda^2)}  \rho^2 \right] + \frac{4 \lambda^2}{(1+2\lambda^2)\pi} (\length -1)(\length -2) \rho^2\\
 &+  \frac{4\lambda (\length -1) \rho}{\sqrt{(1+2\lambda^2)\pi}} 
 \left(\sqrt{\frac{d}{2}} \theta  \, \erf\bigg(\frac{\lambda \sqrt{d/2} \kappa}{\sqrt{1+2\lambda^2 \gamma^2}}\bigg) +  \frac{\lambda  \gamma^2 \rho}{\sqrt{1+2\lambda^2 \gamma^2}}\, \erf'\bigg( \frac{\lambda \sqrt{d/2} \kappa}{\sqrt{1+2\lambda^2 \gamma^2}} \bigg)\right)  \, .
\end{align*}
\end{proposition}

\begin{proof}
We first recall the notations for the five scalar products that are used throughout this proof.
\begin{align*}
    \nu = v^\top v^\star \, , \quad
    \kappa = k^\top k^\star \, , \quad
    \theta = v^\top k^\star \, , \quad
    \eta = k^\top v^\star \, , \quad
    \rho = k^\top v \, .
\end{align*}

\paragraph{A first decomposition.}
We start back from the expression \eqref{eq:intermediate-eq-risk} obtained for the risk. By expanding in $\xi$, then expanding the square, we obtain
\begin{align*}
\cR(k, v) &= \Esp \Big( \Big(X_1^\top v^\star -  \sum_{\ell=1}^\length  X_\ell^\top v \, \erf(\lambda X_\ell^\top k) \Big)^2 \Big)  + \varepsilon^2 \\
 &= \Esp \Big( \Big(X_1^\top v^\star -   X_1^\top v \, \erf(\lambda X_1^\top k) -   \sum_{\ell=2}^\length  X_\ell^\top v \, \erf(\lambda X_\ell^\top k) \Big)^2 \Big)  + \varepsilon^2 \\
 &= \underbrace{\vphantom{\sum_{\ell=2}^\length }\Esp \Big( \big(X_1^\top v^\star -   X_1^\top v \, \erf(\lambda X_1^\top k) \big)^2 \Big)}_{=:\color{RoyalBlue}R_1} +    \underbrace{\sum_{\ell=2}^\length  \Esp \Big( \big( X_\ell^\top v \, \erf(\lambda X_\ell^\top k) \big)^2 \Big)}_{=: \color{PineGreen} R_2} \\
 &\quad + \underbrace{\sum_{\ell \neq j \geq 2}^\length  \Esp \Big( X_\ell^\top v \, \erf(\lambda X_\ell^\top k) X_j^\top v \, \erf(\lambda X_j^\top k)  \Big)}_{=:\color{MediumOrchid}R_3} \\
 &\quad \underbrace{- 2 \sum_{\ell=2}^\length  \Esp \Big( \big(X_1^\top v^\star -   X_1^\top v \, \erf(\lambda X_1^\top k) \big) X_\ell^\top v \, \erf(\lambda X_\ell^\top k) \Big)}_{=: \color{Salmon}R_4} + \,  \varepsilon^2 \, .
\end{align*}

\paragraph{Computation of $R_1$.} By expanding the square,
\begin{align*}
    \Esp \Big( \big(X_1^\top v^\star &-   X_1^\top v \, \erf(\lambda X_1^\top k) \big)^2 \Big)  \\
    &= \Esp  \Big( \big(X_1^\top v^\star\big)^2 \Big)
    - 2   \Esp  \Big( X_1^\top v^\star X_1^\top v \, \erf(\lambda X_1^\top k)   \Big)
    +   \Esp \Big( \big( X_1^\top v \, \erf(\lambda X_1^\top k) \big)^2 \Big) \, .
\end{align*}
These three terms are computed hereafter. First we have
\[
\Esp  \Big( \big(X_1^\top v^\star\big)^2 \Big)
    = \Big(\Esp   \big(X_1^\top v^\star\big) \Big)^2 + \Var\big(X_1^\top v^\star\big)  = (\sqrt{\frac{d}{2}}(k^\star)^\top v^\star)^2 + \gamma^2 = {\color{RoyalBlue} \gamma^2} \, .
\]
Second, 
\begin{align*}
    \Esp \Big( &X_1^\top v^\star X_1^\top v \, \erf(\lambda X_1^\top k)   \Big) \\
    &= \Esp \left[ \left(\sqrt{\frac{d}{2}}(k^\star)^\top v^\star + Z_1\right) \left(\sqrt{\frac{d}{2}}(k^\star)^\top v + Z_2\right)\, \erf \left(\lambda \sqrt{\frac{d}{2}} (k^\star)^\top k + \lambda Z_3\right) \right] \, , \\
    &= \Esp \left[ Z_1 \left(\sqrt{\frac{d}{2}}\theta + Z_2\right)\, \erf \left(\lambda \sqrt{\frac{d}{2}} \kappa + \lambda Z_3\right) \right] \, ,
\end{align*}
with
\begin{align*}
    \begin{pmatrix}
        Z_1 \\
        Z_2 \\
        Z_3
    \end{pmatrix} &\sim 
    \mathcal{N} \left( 0 , \gamma^2 \begin{pmatrix}
        1 & v^\top v^\star & k^\top v^\star \\
        v^\top v^\star & 1 & v^\top k \\
        k^\top v^\star & v^\top k & 1
    \end{pmatrix} \right)
    =
    \mathcal{N} \left( 0 , \gamma^2 \begin{pmatrix}
        1 & \nu & \eta \\
        \nu & 1 & \rho \\
        \eta & \rho & 1
    \end{pmatrix} \right)
    \, .
\end{align*}
Recall the multivariate version of Stein's lemma \citep{stein1981estimation}, which states that, when $Z, G_1, \hdots , G_p$ are centered and jointly Gaussian, and $\sigma: \R^p \to \R$,
\begin{align*}
    \Esp \left[  Z \sigma (G_1, \hdots , G_p)\right] = \sum_{i=1}^p \Cov (Z,G_i) \Esp \left[ \partial_{i} \sigma (G_1, \hdots , G_p) \right] \, .
\end{align*}
Therefore, 
\begin{align*}
    &\Esp  \Big( X_1^\top v^\star X_1^\top v \erf(\lambda X_1^\top k)   \Big) \\
    &= \gamma^2 \nu \Esp\left[ \, \erf\left(\lambda \sqrt{\frac{d}{2}} \kappa+\lambda Z_3\right)\right] + \lambda \gamma^2 \eta \Esp \left[ \left(\sqrt{\frac{d}{2}} \theta +Z_2\right) \, \erf' \left(\lambda \sqrt{\frac{d}{2}} \kappa + \lambda Z_3\right)\right] \\
    &= \gamma^2 \nu \Esp\left[ \, \erf\left(\lambda \sqrt{\frac{d}{2}} \kappa+ \lambda Z_3\right)\right]
    + \lambda \gamma^2 \sqrt{\frac{d}{2}} \eta \theta \Esp\left[ \, \erf'\left(\lambda \sqrt{\frac{d}{2}}\kappa + \lambda Z_3\right)\right] \\
    &\qquad + \lambda^2 \gamma^4 \eta \rho \Esp\left[ \erf''\left(\lambda \sqrt{\frac{d}{2}} \kappa +\lambda Z_3\right)\right] \\
    &= {\color{RoyalBlue} 
    \gamma^2 \nu \erf \left(\frac{\lambda \sqrt{d/2} \kappa}{\sqrt{1+2\lambda^2 \gamma^2}} \right) + \sqrt{\frac{d}{2}}    \frac{\lambda \gamma^2 \eta \theta}{\sqrt{1+2\lambda^2\gamma^2}}\erf'\left(\frac{\lambda \sqrt{d/2} \kappa}{\sqrt{1+2\lambda^2 \gamma^2}}\right) } \\
    &\qquad {\color{RoyalBlue} + \frac{\lambda^2 \gamma^4 \eta \rho}{1+2\gamma^2 \lambda^2}\erf''\left(\frac{\lambda \sqrt{d/2} \kappa}{\sqrt{1+2\lambda^2 \gamma^2}}\right)}
\end{align*}
by using Lemma \ref{lem:technical_results}$(i)-(iii)$.
Finally, using again Stein's lemma and Lemma \ref{lem:technical_results}$(iv)-(vi)$, the computation of the last term is as follows:
\begin{align*}
&{\Esp \Big[ \big( X_1^\top v \, \erf(\lambda X_1^\top k) \big)^2 \Big]} \\
&= \Esp \left[ \bigg(\sqrt{\frac{d}{2}} (k^\star)^\top v + Z_2\bigg)^2 \, \erf\bigg(\lambda \sqrt{\frac{d}{2}} k^\top k^\star +\lambda Z_3\bigg)^2 \right] \\
&= \Esp \left[ \frac{d}{2} \theta^2 \, \erf^2\bigg(\lambda \sqrt{\frac{d}{2}} \kappa +\lambda Z_3\bigg) \right] + 2 \Esp \left[ \sqrt{\frac{d}{2}} \theta Z_2 \, \erf^2\bigg(\lambda \sqrt{\frac{d}{2}} \kappa + \lambda Z_3\bigg) \right] \\
&\; + \Esp \left[Z_2^2 \, \erf^2\bigg(\lambda \sqrt{\frac{d}{2}} \kappa + \lambda Z_3 \bigg) \right] \\
&= \frac{d}{2} \theta^2 \Esp \left[\erf^2\bigg(\lambda \sqrt{\frac{d}{2}} \kappa +\lambda Z_3\bigg)  \right] \\
&\; + 4 \lambda \gamma^2 \sqrt{\frac{d}{2}} \theta \rho \Esp \left[ \, \erf\bigg(\lambda \sqrt{\frac{d}{2}} \kappa + \lambda Z_3\bigg) \, \erf'\bigg(\lambda \sqrt{\frac{d}{2}} \kappa + \lambda Z_3\bigg) \right] +\gamma^2 \Esp \left[ \, \erf^2\bigg(\lambda \sqrt{\frac{d}{2}} \kappa +\lambda Z_3\bigg) \right] \\
&\; + 2 \lambda \gamma^2 \rho \Esp \left[Z_2 \, \erf\bigg(\lambda \sqrt{\frac{d}{2}} \kappa + \lambda Z_3\bigg) \erf'\bigg(\lambda \sqrt{\frac{d}{2}} \kappa + \lambda Z_3\bigg) \right] \\
&= (\frac{d}{2} \theta^2 + \gamma^2) \Esp \left[\erf^2\bigg(\lambda \sqrt{\frac{d}{2}} \kappa + \lambda Z_3\bigg)  \right] \\
&\; + 4 \lambda \gamma^2 \sqrt{\frac{d}{2}} \theta \rho \Esp \left[ \, \erf\bigg(\lambda \sqrt{\frac{d}{2}} \kappa + \lambda Z_3\bigg) \, \erf'\bigg(\lambda \sqrt{\frac{d}{2}} \kappa + \lambda Z_3\bigg) \right] \\
&\; + 2 \lambda^2 \gamma^4 \rho^2 \left(\Esp \left[\erf\bigg(\lambda \sqrt{\frac{d}{2}} \kappa + \lambda Z_3\bigg) \erf''\bigg(\lambda \sqrt{\frac{d}{2}} \kappa +\lambda Z_3\bigg) \right] + \Esp \left[(\erf')^2\bigg(\lambda \sqrt{\frac{d}{2}} \kappa +\lambda Z_3\bigg) \right]\right) \\
&= (\frac{d}{2} \theta^2 + \gamma^2) \zeta\Big(\lambda \sqrt{\frac{d}{2}} \, \kappa, \lambda^2 \gamma^2\Big) \\
&\; + 4 \lambda \gamma^2 \sqrt{\frac{d}{2}} \theta \rho \Esp \left[ \, \erf\left(\lambda \sqrt{\frac{d}{2}} \kappa + \lambda Z_3\right) \, \erf'\left(\lambda \sqrt{\frac{d}{2}} \kappa + \lambda Z_3\right) \right] \\
&\; + \frac{2 \lambda^2 \gamma^4 \rho^2}{1 + 2 \lambda^2 \gamma^2} \left(-2 \lambda \sqrt{\frac{d}{2}} \kappa\,  \Esp\left[\erf\left(\lambda \sqrt{\frac{d}{2}} \kappa+\lambda Z_3\right)\erf'\left(\lambda \sqrt{\frac{d}{2}} \kappa+ \lambda Z_3\right)\right] \right. \\
&\; \left. + \Esp\left[(\erf')^2\left(\lambda \sqrt{\frac{d}{2}} \kappa+ \lambda Z_3\right)\right] \right) \\
&= (\frac{d}{2} \theta^2 + \gamma^2) \zeta\Big(\lambda \sqrt{\frac{d}{2}} \, \kappa, \lambda^2 \gamma^2\Big) \\
&\; + 4 \lambda \gamma^2 \sqrt{\frac{d}{2}} \left(\theta \rho - \frac{\lambda^2 \gamma^2 \rho^2 \kappa}{1 + 2 \lambda^2 \gamma^2} \right) \Esp \left[ \, \erf\left(\lambda \sqrt{\frac{d}{2}} \kappa + \lambda Z_3\right) \, \erf'\left(\lambda \sqrt{\frac{d}{2}} \kappa + \lambda Z_3\right) \right] \\
&\; + \frac{2 \lambda^2 \gamma^4 \rho^2}{1 + 2 \lambda^2 \gamma^2} \Esp\left[(\erf')^2\left(\lambda \sqrt{\frac{d}{2}} \kappa+\lambda Z_3\right)\right] \\
&= {\color{RoyalBlue}(\frac{d}{2} \theta^2 + \gamma^2) \zeta\Big(\lambda \sqrt{\frac{d}{2}} \, \kappa, \lambda^2 \gamma^2\Big)} \\
&{\color{RoyalBlue}+ \sqrt{\frac{d}{2}} \left(\theta \rho - \frac{\lambda^2 \gamma^2 \rho^2 \kappa}{1 + 2 \lambda^2 \gamma^2} \right)  \frac{4 \lambda \gamma^2}{\sqrt{1+2\lambda^2\gamma^2}}    \erf\bigg( \frac{\lambda \sqrt{d/2} \, \kappa}{\sqrt{(1+4\lambda^2\gamma^2)(1+2\lambda^2\gamma^2)}}\bigg)   \erf'\bigg( \frac{\lambda \sqrt{d/2} \, \kappa}{\sqrt{1+2\lambda^2 \gamma^2}}\bigg)} \\
&{\color{RoyalBlue}+ \frac{2 \lambda^2 \gamma^4 \rho^2}{1 + 2 \lambda^2 \gamma^2}  \left(\frac{2}{\sqrt{\pi}\sqrt{1+4\lambda^2 \gamma^2}} \erf'\bigg(-\frac{\lambda \sqrt{d}\kappa}{\sqrt{1+4\lambda^2\gamma^2}} \bigg)\right)}
\end{align*}
by Lemma \ref{lem:technical_results}$(iv)-(vi)$.

\paragraph{Computation of $R_2$.}
We have
\begin{align*}
    R_2 = \sum_{\ell=2}^\length \Esp \Big( \big( X_\ell^\top v \, \erf(\lambda X_\ell^\top k) \big)^2 \Big) &=   (\length -1) \Esp \Big( \big( X_2^\top v \, \erf(\lambda X_2^\top k) \big)^2 \Big) \, . 
\end{align*}
Thus, using previous calculations with $\gamma^2=1$, $\theta=0$, and $\kappa=0$, we obtain
\begin{align*}
R_2 &= (\length -1) \left[\zeta(0, \lambda^2) + \frac{4 \lambda^2}{\sqrt{\pi}\sqrt{4\lambda^2+1} (1 + 2 \lambda^2)} \rho^2 \erf'\left(0\right) \right] \\
&= {\color{PineGreen}(\length -1) \left[\zeta(0, \lambda^2) + \frac{8 \lambda^2}{\pi\sqrt{4\lambda^2+1} (1 + 2 \lambda^2)} \rho^2 \right]} \, .
\end{align*}

\paragraph{Computation of $R_3$.} Regarding the cross-product terms, by independence of the $(X_\ell)$'s and Stein's lemma, one gets 
\begin{align*}
    \Esp \Big( X_\ell^\top v \, \erf(\lambda X_\ell^\top k) X_j^\top v \, \erf(\lambda X_j^\top k)  \Big) &= \Esp \Big( X_\ell^\top v \, \erf(\lambda X_\ell^\top k) \Big) \Esp \Big( X_j^\top v \, \erf(\lambda X_j^\top k)  \Big) = C^2 \rho^2 \, ,
\end{align*}
{with $C := \lambda \Esp(\erf'(\lambda X_\ell^\top k)) = 2\lambda/\sqrt{(1+2\lambda^2)\pi}$ by Lemma \ref{lem:technical_results}$(i)$}. This leads to $${\color{MediumOrchid}R_3 = \frac{4 \lambda^2}{(1+2\lambda^2)\pi} (\length -1)(\length -2) \rho^2}.$$

\paragraph{Computation  of $R_4$.} We have, again by independence and Stein's lemma,
\begin{align*}
    \Esp \Big( \big(X_1^\top v^\star -   X_1^\top v \,\erf(\lambda X_1^\top k) \big) &X_\ell^\top v \,\erf(\lambda X_\ell^\top k) \Big) \\
    &= \Esp \Big( X_1^\top v^\star -   X_1^\top v \,\erf(\lambda X_1^\top k) \Big) \Esp \Big( X_\ell^\top v \,\erf(\lambda X_\ell^\top k) \Big) \\
    &= \Big( \sqrt{\frac{d}{2}} (k^\star)^\top v^\star -   \Esp (X_1^\top v \,\erf(\lambda X_1^\top k))  \Big)\Esp \Big( X_\ell^\top v \,\erf(\lambda X_\ell^\top k) \Big)\\
    &=  - \Esp (X_1^\top v \,\erf(\lambda X_1^\top k))  \cdot C \rho \\
    &= - \frac{2\lambda \rho}{\sqrt{(1+2\lambda^2)\pi}} \Esp (X_1^\top v \,\erf(\lambda X_1^\top k)) \, .
\end{align*}
Note that, still using Stein's lemma,
\begin{align*}
  -   \Esp (&X_1^\top v \,\erf(\lambda X_1^\top k)  ) \\
 &= 
 -   \Esp (\sqrt{\frac{d}{2}} (k^\star)^\top v \, \erf(\lambda X_1^\top k)  ) 
 -   \Esp ((X_1^\top v - \sqrt{\frac{d}{2}} (k^\star)^\top v) \, \erf(\lambda X_1^\top k)  ) \\
 &=  -   \Esp (\sqrt{\frac{d}{2}} \theta \, \erf(\lambda X_1^\top k)  )
 -   \mathrm{Cov}\Big( X_1^\top v , \, \erf(\lambda X_1^\top k) \Big) \\
 &= -  \sqrt{\frac{d}{2}} \theta \, \Esp (\erf(\lambda X_1^\top k)  )
 -  \lambda \mathrm{Cov}\Big( X_1^\top v , X_1^\top k\Big)  \Esp\Big(\, \erf'(\lambda X_1^\top k) \Big) \\
 &= -  \sqrt{\frac{d}{2}} \theta \, \erf\left( \frac{\lambda \sqrt{d/2} \kappa}{\sqrt{1+2\lambda^2 \gamma^2}}\right)  - \lambda  \gamma^2 (k^\top v) {{\frac{1}{\sqrt{1+2\gamma^2\lambda^2}}\, \erf'\left( \frac{\lambda \sqrt{d/2} \kappa}{\sqrt{1+2\lambda^2 \gamma^2}} \right)}},
\end{align*}
where we used that $\lambda X_1^\top k \overset{\mathcal{L}}{=} \lambda \sqrt{d/2} \kappa + G$ with $G\sim \mathcal{N}(0,\lambda^2\gamma^2)$, in combination with Lemma~\ref{lem:technical_results}$(i)-(ii)$.  
Thus
\begin{align*}
{\color{Salmon}R_4 = \frac{4\lambda (\length -1) \rho}{\sqrt{(1+2\lambda^2)\pi}} 
 \left( \sqrt{\frac{d}{2}} \theta  \, \erf\Big( \frac{\lambda \sqrt{d/2} \kappa}{\sqrt{1+2\lambda^2 \gamma^2}}\Big)
 +  \frac{\lambda  \gamma^2 \rho}{\sqrt{1+2\gamma^2\lambda^2}}\, \erf'\Big( \frac{\lambda \sqrt{d/2} \kappa}{\sqrt{1+2\lambda^2 \gamma^2}} \Big)
 \right)} .
\end{align*}

\paragraph{All in all.} Putting everything together, we obtain
\begin{align*}
&{\cR(k, v) } = \varepsilon^2 \\
&+{\color{RoyalBlue}\gamma^2- 2 \gamma^2 \nu \erf \bigg(\frac{\lambda \sqrt{d/2} \kappa}{\sqrt{1+2\lambda^2 \gamma^2}}\bigg) -2 \lambda \gamma^2 \sqrt{\frac{d}{2}}  \eta \theta  \frac{1}{\sqrt{1+2\lambda^2\gamma^2}}\erf'\bigg(\frac{\lambda \sqrt{d/2} \kappa}{\sqrt{1+2\lambda^2 \gamma^2}}\bigg)  } \\
 & {\color{RoyalBlue}- \frac{2\lambda^2 \gamma^4 \eta \rho}{1+2\lambda^2\gamma^2}\erf''\bigg( \frac{\lambda \sqrt{d/2} \kappa}{\sqrt{1+2\lambda^2 \gamma^2}}\bigg)  + (\frac{d}{2} \theta^2 + \gamma^2) \zeta\Big(\lambda\sqrt{\frac{d}{2}} \, \kappa,\lambda^2\gamma^2\Big) } \\
 & {\color{RoyalBlue}  + \sqrt{\frac{d}{2}} \left(\theta \rho - \frac{\lambda^2 \gamma^2 \rho^2 \kappa}{1 + 2 \lambda^2 \gamma^2} \right)  \frac{4 \lambda \gamma^2}{\sqrt{1+2\lambda^2\gamma^2}}    \erf\bigg( \frac{\lambda \sqrt{d/2} \, \kappa}{\sqrt{(1+4\lambda^2\gamma^2)(1+2\lambda^2\gamma^2)}}\bigg)   \erf'\bigg( \frac{\lambda \sqrt{d/2} \, \kappa}{\sqrt{1+2\lambda^2\gamma^2}}\bigg)} \\
& {\color{RoyalBlue} +  \frac{4 \lambda^2 \gamma^4 \rho^2}{\sqrt{\pi} \sqrt{1+4\lambda^2\gamma^2} (1+2\lambda^2\gamma^2)} \erf'\bigg(-\frac{\lambda\sqrt{d}\kappa}{\sqrt{1+4\lambda^2\gamma^2}} \bigg)} \\
 &+  {\color{PineGreen} (\length -1) \left[\zeta(0, \lambda^2) + \frac{8 \lambda^2}{\pi\sqrt{1+4\lambda^2} (1 + 2 \lambda^2)}  \rho^2 \right]} + {\color{MediumOrchid} \frac{4 \lambda^2}{(1+2\lambda^2)\pi} (\length -1)(\length -2) \rho^2}\\
 &+ {\color{Salmon} \frac{4\lambda (\length -1) \rho}{\sqrt{(1+2\lambda^2)\pi}} 
 \left(\sqrt{\frac{d}{2}} \theta  \, \erf\bigg(\frac{\lambda \sqrt{d/2} \kappa}{\sqrt{1+2\lambda^2 \gamma^2}}\bigg) +  \frac{\lambda  \gamma^2 \rho}{\sqrt{1+2\lambda^2 \gamma^2}}\, \erf'\bigg(\frac{\lambda \sqrt{d/2} \kappa}{\sqrt{1+2\lambda^2 \gamma^2}}\right) } \, .
\end{align*}

This concludes the proof.
\end{proof}

\section{Technical results} \label{app:tech-results}

This section gathers formulas that are useful in the proofs, in particular regarding expectation of functions of Gaussian random variables involving $\erf$.

\begin{lemma}   \label{lemma:simple-bounds}
    For $u \geq 0$,
    \[
    \Big|\frac{1}{\sqrt{1+u}} - 1\Big| \leq u \, .
    \]
\end{lemma}
\begin{proof}
The argument of the absolute value is non-positive for $u \geq 0$, hence we need to show that
\[
f(u) := 1 - \frac{1}{\sqrt{1+u}} - u
\]
is non-positive for $u \geq 0$. Just note that
\[
f(0) = 0 \quad \textnormal{and} \quad f'(u) = \frac{1}{(1+u)^{3+2}} - 1 \leq 0 \, .
\]
\end{proof}
Recall that the $\erf$ function is defined on $\mathbb R$ as 
    \begin{equation*}
        \erf(u) = \frac{2}{\sqrt{\pi}} \int_0^u e^{-t^2} dt \, .
    \end{equation*}
\begin{lemma}[Properties of the $\erf$ function]
\label{lem:erf}
    We have 
    \begin{align*}
        \erf'(u) &= \frac{2}{\sqrt{\pi}} e^{-u^2}\, , \\
        \erf''(u) &= -\frac{4}{\sqrt{\pi}} u e^{-u^2} = -2u \erf'(u)  \, ,  \\
        |\erf(u)| &\leq \frac{2}{\sqrt{\pi}} |u| \, .
    \end{align*}
\end{lemma}

\begin{proof}
   The first two statements are clear by usual differentiation rules. Regarding the last statement, since $\erf$ is an odd function, it is sufficient to prove the statement for $u \geq 0$. Moreover, $\erf$ is concave on $[0, \infty)$, so we get, for $u \geq 0$,
   \[
   |\erf(u)| = |\erf(u) - \erf(0)| \leq \erf'(0) u = \frac{2}{\sqrt{\pi}} u \, , 
   \]
   which concludes the proof.
\end{proof}

\begin{lemma}
\label{lem:technical_results}
    Let $G \sim \mathcal{N}(0,\gamma^2)$. For $t\in\mathbb{R}$,
    \begin{enumerate}[label=(\roman*)]
        \item $\Esp \left[ \erf'(t+G) \right]= \frac{1}{\sqrt{1+2\gamma^2}}\erf'\Big( \frac{t}{\sqrt{1+2\gamma^2}}\Big)$.
        \item $\Esp \left[ \erf(t+G) \right] = \erf \Big( \frac{t}{\sqrt{1+2\gamma^2}} \Big)$.
        \item $\Esp \left[ \erf''(t+G) \right]= \frac{1}{1+2\gamma^2}\erf''\Big( \frac{t}{\sqrt{1+2\gamma^2}}\Big)$.
        \item $\Esp \left[ (\erf')^2(t+G) \right] =  \frac{2}{\sqrt{\pi}\sqrt{1+4\gamma^2}} \erf'\Big(-\frac{\sqrt{2}t}{\sqrt{1+4\gamma^2}} \Big)$.
        \item $(1+2\gamma^2) \Esp[\erf(t+G)\erf''(t+G)] = -2t\,  \Esp[\erf(t+G)\erf'(t+G)]-2\gamma^2 \Esp[(\erf'(t+G))^2]$.
        \item $\Esp \left[ \erf(t+G) \erf'(t+G) \right] = \frac{1}{\sqrt{1+2\gamma^2}}    \erf\Big( \frac{t}{\sqrt{(1+4\gamma^2)(1+2\gamma^2)}}\Big)   \erf'\Big( \frac{t}{\sqrt{1+2\gamma^2}}\Big)$.
    \end{enumerate}
\end{lemma}

This lemma reveals the importance of choosing the $\erf$ function as the component-wise nonlinearity: there are closed-form formulas for the expectation of $\erf$ and its derivatives applied to Gaussian random variables. Extending the results to any nonlinear, bounded, increasing, equal to 0 at 0, and differentiable activation function is an interesting next step.

\begin{proof}
    \begin{enumerate}[label=\textit{(\roman*)}]
    \item By Lemma \ref{lem:erf},
     \begin{align*}
        \Esp \left[ \erf'(t+G) \right]
        &= \frac{\sqrt{2}}{\pi \gamma} \int e^{-(t+g)^2} e^{-\frac{g^2}{2\gamma^2}} \diff g \\
        &= \frac{\sqrt{2}}{\pi \gamma} \int e^{-\frac{g^2}{c}} e^{-2 gt} e^{-t^2} \diff g \qquad  \text{for} \quad c:= \frac{2\gamma^2}{1+2\gamma^2} \\
        &= \frac{\sqrt{2}}{\pi \gamma} \int e^{-\frac{(g+ct)^2}{c} + ct^2-t^2}\diff g \\
        &= \frac{\sqrt{2}}{\pi \gamma} e^{-t^2(1-c)} \underbrace{\int e^{-\frac{(g+ct)^2}{c}} \diff g}_{=\sqrt{\pi c}} \\
        &= \frac{2}{\sqrt{\pi (1 + 2\gamma^2)}} \exp\left(-t^2 \left(1-\frac{2\gamma^2}{1+2\gamma^2} \right)\right) \\
        &= \frac{2}{\sqrt{\pi}\sqrt{1+2\gamma^2}} \exp\left(-\frac{t^2}{1+2\gamma^2} \right) \, .
    \end{align*}
    \item By $(i)$,
    \begin{align*}
        \Esp \left[ \erf(t+G) \right]
        &= \int_{-\infty}^t \Esp \left[ \erf'(s+G) \right] ds \\
        &= \int_{-\infty}^t \frac{2}{\sqrt{\pi}\sqrt{1+2\gamma^2}} \exp\left(-\frac{s^2}{1+2\gamma^2} \right) ds \\
        &= \int_{-\infty}^{t/\sqrt{1+2\gamma^2}} \frac{2}{\sqrt{\pi}} \exp\left(-u^2  \right) ds \\
         &= \erf\Big(\frac{t}{\sqrt{1+2\gamma^2}}\Big) \, .
    \end{align*}
    \item By Lemma \ref{lem:erf}, and following the same steps as in $(i)$,
    \begin{align*}
        \Esp \left[ \erf''(t+G) \right] &= - \frac{2\sqrt{2}}{\sqrt{\pi}\gamma} \int (t+g) e^{-(t+g)^2} e^{-\frac{g^2}{2\gamma^2}} dg \\
        &= -  \frac{2\sqrt{2}}{\pi \gamma} e^{-t^2(1-c)} \int (t+g) e^{-\frac{(g+ct)^2}{c}} dg \\
        &= -  \frac{2\sqrt{2}}{\pi \gamma} e^{-t^2(1-c)} \Big(t \sqrt{\pi c} + \sqrt{\pi c} \Esp(\cN(-ct, \frac{c}{2})) \Big) \\
        &= -  \frac{2\sqrt{2c}}{\sqrt{\pi} \gamma} e^{-t^2(1-c)} (t - ct) \\
        &= -  \frac{4}{\sqrt{\pi (1+2 \gamma^2)}} e^{-t^2(1-c)} \frac{1}{1+2 \gamma^2} t \\
        &= -  \frac{4 t}{\sqrt{\pi} (1+2 \gamma^2)^{3/2}} \exp\left(-\frac{t^2}{1+2\gamma^2} \right) \, .
       \end{align*}
    \item By Lemma \ref{lem:erf},
    \begin{align*}
        \Esp \left[ (\erf')^2(t+G) \right] &= \frac{1}{\sqrt{2\pi} \gamma} \int (\erf')^2(t+g) e^{-\frac{g^2}{2\gamma^2}} \, \diff g \\
        &= \frac{2\sqrt{2}}{\gamma \pi^{3/2}} \int e^{-2 (t+g)^2} e^{-\frac{g^2}{2\gamma^2}}\, \diff g \\
        &= \frac{2\sqrt{2}}{\gamma \pi^{3/2}} \int e^{-\frac{g^2}{2\Gamma^2}} e^{-4gt} e^{-2 t^2}\, \diff g \qquad \text{with} \quad \Gamma^2 := \gamma^2/(1+4\gamma^2)\\
        &= \frac{2\sqrt{2}}{\gamma \pi^{3/2}} \int  e^{-\frac{(g+ 4\Gamma^2 t)^2}{2\Gamma^2}} e^{8\Gamma^2 t^2} e^{-2 t^2}\, \diff g \\
        &= \frac{2\sqrt{2}}{\gamma \pi^{3/2}} e^{-2 t^2(1-4\Gamma^2)}\int  e^{-\frac{(g+ 4\Gamma^2 t)^2}{2\Gamma^2}}\, \diff g \\
        &= \frac{2\sqrt{2}}{\gamma \pi^{3/2}} e^{-2t^2(1-4\Gamma^2)} \sqrt{2\pi} \Gamma \\
        &= \frac{4}{\pi \sqrt{1+4\gamma^2}} \exp\left(-\frac{2 t^2}{1+4\gamma^2} \right) \, .
    \end{align*}
    \item We use Lemma \ref{lem:erf} and then Stein's lemma: 
    \begin{align*}
        \Esp[\erf(t+G)&\erf''(t+G)] \\
        &= -2\Esp\left[(t+G)\erf(t+G)\erf'(t+G)\right] \\
        &= -2t\Esp\left[\erf(t+G)\erf'(t+G)\right]  -2\Esp\left[G\erf(t+G)\erf'(t+G)\right] \\
        &= -2t\Esp\left[\erf(t+G)\erf'(t+G)\right]  \\
        &\qquad -2\gamma^2 \left(\Esp\left[\erf'(t+G)^2\right] + \Esp\left[\erf(t+G)\erf''(t+G)\right]\right) \, .
    \end{align*}
    Reordering terms, this gives the desired equation.
    \item We define the function 
    \begin{align*}
        f(t) = \Esp\left[\erf(t+G)\erf'(t+G)\right] \, .
    \end{align*}
    Then, using Lemma \ref{lem:technical_results}$(v)$, we have 
    \begin{align*}
        f'(t) &= \Esp\left[\erf'(t+G)^2\right] + \Esp\left[\erf(t+G)\erf''(t+G) \right] \\
        &= \Esp\left[\erf'(t+G)^2\right] -\frac{2t}{1+2\gamma^2}\,  \Esp\left[\erf(t+G)\erf'(t+G)\right] \\
        &\qquad-\frac{2\gamma^2}{1+2\gamma^2} \Esp\left[(\erf'(t+G))^2\right] \\
        &= \frac{1}{1+2\gamma^2} \Esp\left[(\erf'(t+G))^2\right] - \frac{2t}{1+2\gamma^2} f(t) \, .
    \end{align*}
    We solve this differential equation by the method of variation of parameters: we have
    \begin{align*}
        \frac{\diff}{\diff t}\left(f(t)e^{t^2/(1+2\gamma^2)}\right) = \frac{1}{1+2\gamma^2} \Esp\left[(\erf'(t+G))^2\right] e^{t^2/(1+2\gamma^2)} \, .
    \end{align*}
    We use Lemmas \ref{lem:erf} and \ref{lem:technical_results}$(iv)$: 
    \begin{align*}
        \frac{\diff}{\diff t}\left(f(t)e^{t^2/(1+2\gamma^2)}\right) &= \frac{2}{\sqrt{\pi}}\frac{1}{(1+2\gamma^2)\sqrt{1+4\gamma^2}} \erf'\left(-\frac{\sqrt{2}t}{\sqrt{1+4\gamma^2}}\right)  e^{t^2/(1+2\gamma^2)} \\
        &= \frac{4}{\pi} \frac{1}{(1+2\gamma^2)\sqrt{1+4\gamma^2}} e^{-2t^2/(1+4\gamma^2) } e^{t^2/(1+2\gamma^2)} \\
        &=  \frac{4}{\pi} \frac{1}{(1+2\gamma^2)\sqrt{1+4\gamma^2}} \exp\left(-\frac{t^2}{(1+2\gamma^2)(1+4\gamma^2)}\right) \\
        &= \frac{2}{\sqrt{\pi}} \frac{1}{(1+2\gamma^2)\sqrt{1+4\gamma^2}} \erf'\left(\frac{t}{\sqrt{(1+2\gamma^2)(1+4\gamma^2)}}\right) \, .
    \end{align*}
    As the distribution of $G$ is symmetric and $\erf$ is an odd function, we have that $f(0) = \Esp\left[\erf(G)\erf'(G)\right] = 0$. Thus integrating the above derivative, we obtain 
    \begin{align*}
        f(t)e^{t^2/(1+2\gamma^2)} &= \frac{2}{\sqrt{\pi}} \frac{1}{(1+2\gamma^2)\sqrt{1+4\gamma^2}} \int_0^t \diff s \,\erf'\left(\frac{s}{\sqrt{(1+2\gamma^2)(1+4\gamma^2)}}\right) \\
        &= \frac{2}{\sqrt{\pi}} \frac{1}{\sqrt{1+2\gamma^2}} \erf\left(\frac{t}{\sqrt{(1+2\gamma^2)(1+4\gamma^2)}}\right) \, .
    \end{align*}
    Using again Lemma \ref{lem:erf}, we obtain the claimed result:
    \begin{equation*}
        f(t) = \frac{1}{\sqrt{1+2\gamma^2}} \erf'\left(\frac{t}{\sqrt{1+2\gamma^2}}\right)\erf\left(\frac{t}{\sqrt{(1+2\gamma^2)(1+4\gamma^2)}}\right) \, .
    \end{equation*}
    \end{enumerate}
\end{proof}

\section{Experimental details and additional results}  \label{app:experimental-details}

Our code is available at 
\begin{quote}
\url{https://github.com/PierreMarion23/single-location-regression}
\end{quote} 

We use the Transformers \citep{wolf2020transformers} and scikit-learn \citep{pedregosa2011scikitlearn} libraries for the experiment of Section \ref{sec:regressiontask}, and JAX \citep{jax2018github} for the experiment of Section \ref{sec:optim}.
All experiments run in a short time (less than one hour) on a standard laptop.

\subsection{Experiment of Section \ref{sec:regressiontask} (NLP motivations)}

\paragraph{Data generation.} We use synthetically-generated data for this experiment. To create our train set, we generate sentences according to the patterns
\begin{center}
\texttt{The city is [SENTIMENT ADJ]. [PRONOUN] [COLOR ADJ] [ANIMAL] is [ADV] [SENTIMENT ADJ].}
\end{center}
and
\begin{center}
\texttt{The city is [SENTIMENT ADJ]. [PRONOUN] [SENTIMENT ADJ] [ANIMAL] is [ADV] [COLOR ADJ].}
\end{center}
where \texttt{ADJ} stands for adjective and \texttt{ADV} for adverb. 
Note that the difference between the two patterns is that the locations of the sentiment and of the color adjectives are swapped. Each element between brackets corresponds to a word, which can take a few different values that are chosen manually. For instance, some possible sentiment adjectives are \texttt{nice}, \texttt{clean}, \texttt{cute}, \texttt{delightful}, \texttt{mean}, \texttt{dirty}, or \texttt{nasty}. A possible value for some words is $\emptyset$, meaning that we remove the word from the sentence, which creates more variety in sentence length. By doing the Cartesian product over the possible values of each word in brackets, we generate in this way a large number of examples. Then, the label associated to each example depends solely on the sentiment adjective appearing in the \textit{second} sentence. For instance, the words \texttt{nice}, \texttt{clean}, \texttt{cute}, or \texttt{delightful} are associated to a label $+1$, while the words \texttt{delightful}, \texttt{mean}, and \texttt{dirty} are associated to a label $-1$. 

We now explain how the test sets are generated. We generate four test sets in order to assess the robustness of the model to various out-of-distribution changes. The baseline test set uses the same sentence patterns and the same sentiment adjectives as in the training set, but other words in the example (e.g., animals, adverb) are different. In particular, a given sentence cannot appear both in the train set and in the test set. Then, we generate another test set by using sentiment adjectives that are not present in the training set. We emphasize that the sentiment adjective fully determines the label, so using unseen adjectives at test time makes the task significantly harder. The third test set uses the same adjectives as in the train set, but another sentence pattern, namely
\begin{center}
\texttt{Hello, how are you? Good evening, [PRONOUN] [COLOR ADJ] [ANIMAL] is [ADV] [SENTIMENT ADJ].}
\end{center}
Finally, the fourth test set combines a different sentence pattern and unseen adjectives. The size of the datasets is given in the table below. All datasets have the same number of $+1$ and $-1$ labels.

\begin{table}[ht]
   \centering
   \begin{tabular}{cc}
   \toprule
   {\bf Name} & {\bf Number of examples} \\
   \midrule
   Train set & $15552$ \\
   Test set & $4608$ \\
   Test w.~OOD tokens & $3072$ \\
   Test w.~OOD structure & $144$ \\
   Test w.~OOD structure+tokens & $96$ \\
   \bottomrule
   \end{tabular}
   \medskip
   \caption{Size of the generated datasets.}
\end{table}

\paragraph{Model.} We recall that there exists several families of Transformer architectures, which in particular are not all best suited for sequence classification. An appropriate family is called encoder-only Transformer, and a foremost example is BERT \citep{devlin2019bert}. We refer to \citet{phuong2022formal} for an introductory discussion of Transformer architectures and associated algorithms. Here, we use a pretrained BERT model from the Hugging Face Transformers library \citep{wolf2020transformers}, with the default configuration, namely \texttt{bert-base-uncased}. The model has $110$M parameters, $12$ layers, the tokens have dimension $d=768$, and each attention layer has $12$ heads. It was pretrained by masked language modeling, namely some tokens in the input are hidden, and the model learns to predict the missing tokens. We refer to \citet{devlin2019bert} for details on the architecture and pretraining procedure. We do not perform any fine-tuning on the model.

\paragraph{Experiment design.} Our experiment consists in performing logistic regression on embeddings of [CLS] tokens in the hidden layers of the pretrained BERT model, where we recall that the [CLS] token is a special token added to the beginning of each input sequence. This is a particular case of the so-called \textit{linear probing}, which is a common technique in the field of LLMs interpretability. More precisely, let~$\ell$ denote a layer index between $0$ and $12$, where the index $0$ corresponds to the input to the model (after tokenization and embedding in $\R^d$). Then, for each value of $\ell \in \{0, \dots, 12\}$, we train a logistic regression classifier, where, for each example, the input to the classifier is the embedding of the [CLS] token at layer $\ell$ (that is, a $d$-dimensional vector), and the label is simply the label of the sentence as described above. 

\paragraph{Results.}
For $\ell=0$ (blue bar in Figure \ref{fig:exp-transformers}), the embedding of [CLS] is a fixed vector that does not depend on the rest of the sequence, so the  classifier has a pure-chance accuracy of $50\%$. However, as soon as $\ell>0$, thanks to the attention mechanism, the [CLS] token contains information about the sequence. We report in Figure \ref{fig:exp-transformers} the average accuracy over $\ell \in \{1, \dots, 12\}$ for the train set (in orange) and the test sets (in green). We observe that the information contained in the [CLS] token is actually very rich, since logistic regression achieves a perfect accuracy of $100\%$ in the train set. In other words, the data fed to the classifier is linearly separable. We emphasize that the size of the train set is significantly larger than the ambient dimension $d$, so it is far from trivial that this procedure would yield a linearly-separable dataset. Therefore, obtaining linearly-separable data demonstrates that \textit{the model constructs a linear representations of the input inside the [CLS] token}. Moving on to the test sets, the accuracy on the baseline test set is very good ($95\%$), which suggests some generalization abilities of the model. The accuracy on the out-of-distribution test sets degrades (between $64\%$ and $75\%$), but remains largely superior to pure-chance performance. This suggests that the internal representation built by the Transformer model is to some extent universal, in the sense that it is robust to the specifics of the sentence structure and of the word choice.

\subsection{Experiment of Section \ref{sec:optim} (Gradient descent recovers the oracle predictor)}

We begin by providing additional results before giving experimental details.

\paragraph{PGD with an initialization on the sphere and constant inverse temperature schedule.} 
As emphasized in Section \ref{sec:optim}, the dynamics of PGD with a general initialization on $(\mathbb{S}^{d-1})^2$ depend on the choice of the inverse temperature schedule $\lambda_t$. The experiment presented in the main text in Figure~\ref{fig:init-sphere} is for a decreasing schedule $\lambda_t = 1/(1 + 10^{-4}t)$. We report in Figure \ref{fig:experiment-gd-constant-lambda-sphere} results when taking a constant inverse temperature. We observe distinct patterns depending on the value of this parameter. With a large inverse temperature (Figure \ref{fig:init-sphere-constant-large-lambda}), we observe that the dynamics in $(\kappa, \nu)$ always escape the neighborhood of $0$. Furthermore, the direction $v^\star$ is almost perfectly recovered, i.e., $\nu \approx 1$. However, the value of $k^\star$ is only partially recovered: the dynamics stabilize around $\kappa \approx 0.3$. Moreover, the excess risk plateaus at a high value, while the dynamics stay far away from the manifold $\cM$. In the case of a smaller inverse temperature (Figure \ref{fig:init-sphere-constant-small-lambda}), the situation is different. We observe that some initializations lead to a convergence to the point $(\kappa, \nu)=(0, 0)$, in which case the dynamics stay far from the manifold $\cM$. In other words, there is no recovery of $k^\star$ and $v^\star$. Other initializations lead to perfect recovery of $k^\star$ and $v^\star$. In all cases, the final excess risk is low. 
Theoretical study of these observations is left for future work.

\begin{figure}[ht]
    \begin{subfigure}[b]{0.99\textwidth}
    \centering
    \includegraphics[width=\textwidth]{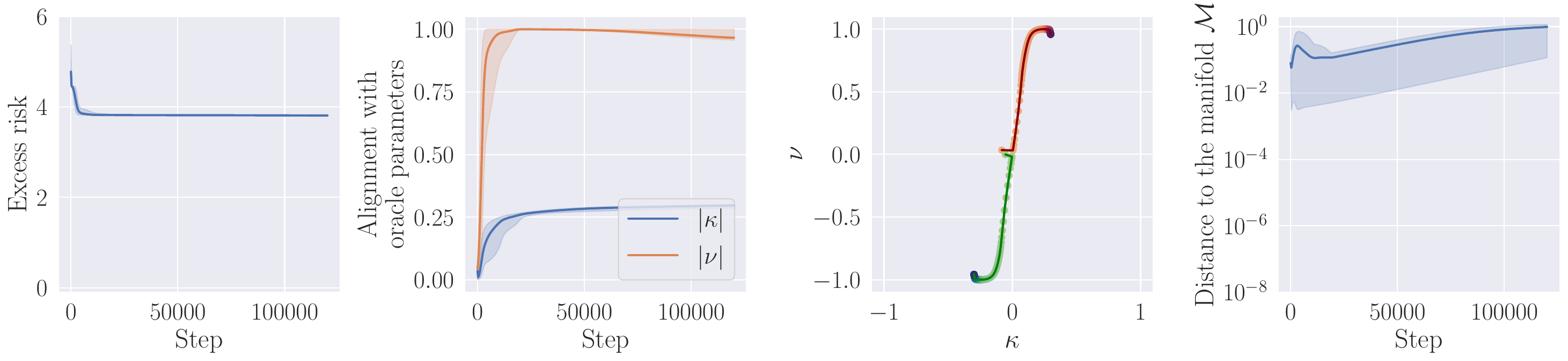}
    \caption{For $\lambda_t = 0.9$.}
    \label{fig:init-sphere-constant-large-lambda}
    \end{subfigure}
     \hfill
    \begin{subfigure}[b]{0.99\textwidth}
     \centering
    \includegraphics[width=\textwidth]{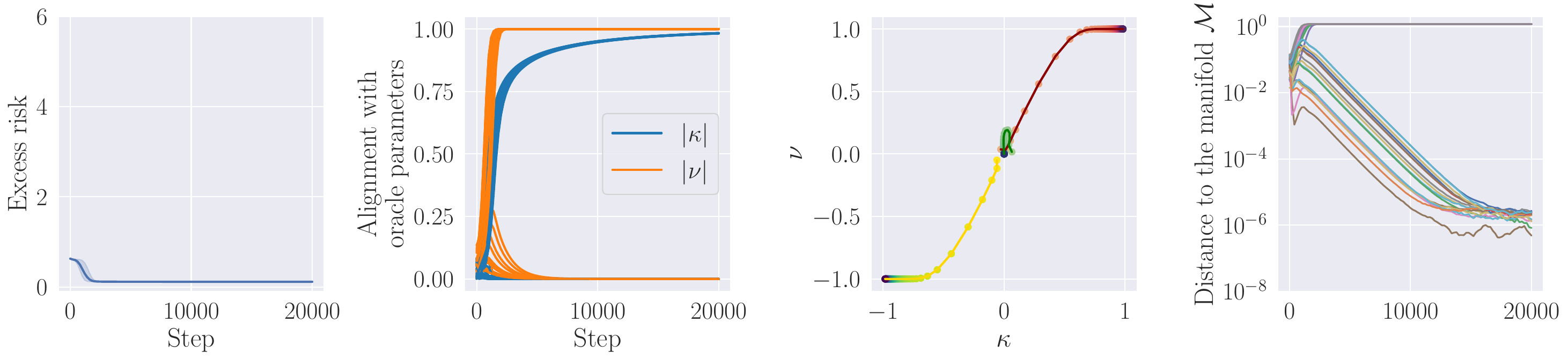}
    \caption{For $\lambda_t = 0.1$.}
    \label{fig:init-sphere-constant-small-lambda}
    \end{subfigure}
    \caption{Dynamics of PGD from a random initialization on $(\mathbb{S}^{d-1})^2$, for two iteration-independent values of $\lambda_t$. \textbf{Left:} Excess risk as a function of the number of steps. \textbf{Middle left:} Alignment $|\kappa| = |k^\top k^\star|$ and $|\nu| = |v^\top v^\star|$ with the oracle parameters. \textbf{Middle right:} Trajectories of $\kappa$ and $\nu$ in a few repetitions of the experiments. Each repetition corresponds to a color, the end point of each trajectory is in blue. \textbf{Right:} Distance to the invariant manifold $\cM$. In all plots except the middle right ones, the experiment is repeated $30$ times with independent random initializations, and either $95\%$ percentile intervals are plotted or all the curves are plotted. Parameters are $d=400$, $L=10$, and $\gamma = \sqrt{1/2}$.}
    \label{fig:experiment-gd-constant-lambda-sphere}
\end{figure}

\paragraph{Implementation details.} The implementation of the PGD algorithm \eqref{eq:pgd-iterations} requires to compute the gradient of the risk. To this aim, we use the formula for the risk given by Proposition~\ref{prop:full-risk-expression}. Note that all quantities appearing in this expression have explicit derivatives. The only quantity for which this is not directly clear is the function $\zeta$, which needs to be differentiated with respect to its first variable to compute the derivative of the risk with respect to $\kappa$. However, recall that $\zeta(t,\gamma^2) := \Esp \left[ \erf^2(t+G)\right]$. Then, by Lemma \eqref{lem:technical_results},
\begin{align*}
\partial_t\zeta(t,\gamma^2) &= 2 \Esp \left[ (\erf \, \erf')(t+G) \right] \\
&= \frac{2}{\sqrt{1+2\gamma^2}}    \erf\left( \frac{t}{\sqrt{(1+4\gamma^2)(1+2\gamma^2)}}\right)   \erf'\left( \frac{t}{\sqrt{1+2\gamma^2}}\right) \, .
\end{align*}
Evaluating $\zeta$ itself (and not its derivative) is not required to simulate the dynamics, but is useful for reporting the value of the risk. For this, we also use the formula above, and use numerical quadrature to compute the value of 
\[
\zeta(t,\gamma^2) = \int_{-\infty}^t \partial_s\zeta(s,\gamma^2) ds \, .
\]
We report in the figures the value of the excess risk, i.e., the risk $\cR_\lambda(k,v) - \varepsilon^2$. 
To compute the distance to the manifold $\cM$, recall that it is defined by
\[
\mathcal{M} = \{(k,v) \in \mathbb{S}^{d-1} \times \mathbb{S}^{d-1}, k^\top v^\star = 0, v^\top k^\star = 0, k^\top v = 0\} \, .
\]
For a point $(k,v) \in \mathbb{S}^{d-1} \times \mathbb{S}^{d-1}$, its distance to $\cM$ is therefore computed as
\[
d_\cM((k,v)) = \sqrt{(k^\top v^\star)^2 + (v^\top k^\star)^2 + (k^\top v)^2} \, .
\]

\paragraph{Parameter values.} The following table summarizes the value of the parameters in our experiments.

\begin{table}[ht]
   \centering
   \begin{tabular}{cccccc}
   \toprule
   {\bf Name} & {\bf Figure \ref{fig:init-sphere}} & {\bf Figure \ref{fig:init-manifold}} & {\bf Figure \ref{fig:experiment-sgd}} & {\bf Figure \ref{fig:init-sphere-constant-large-lambda}} & {\bf Figure \ref{fig:init-sphere-constant-small-lambda}} \\
   \midrule
   $d$ & $400$ & $400$ & $80$ & $400$ & $400$ \\
   $L$ & $10$ & $10$ & $10$ & $10$ & $10$ \\
   $\gamma$ & $1/\sqrt{2}$ & $1/\sqrt{2}$ & $1/\sqrt{2}$ & $1/\sqrt{2}$ & $1/\sqrt{2}$ \\
   $\lambda_t$ & $1/(1 + 10^{-4}t)$ & $0.1$ & $2/(1 + 10^{-4}t)$ & $0.9$ & $0.1$ \\
   $\alpha$ & $4 \cdot 10^{-3}$ & $4 \cdot 10^{-3}$ & $10^{-3}$ & $10^{-3}$ & $4 \cdot 10^{-3}$ \\
   Number of steps & $120$k & $20$k & $200$k & $120$k & $20$k \\
   N.~of repetitions & $30$ & $30$ & $30$ & $30$ & $30$ \\
   Batch size & - & - & $5$ & - & - \\
   $\varepsilon$ & $0$ & $0$ & $0.1$ & $0$ & $0$ \\
   \bottomrule
   \end{tabular}
   \medskip
   \caption{Parameter values for the experiments on recovery of the oracle predictor by gradient descent.}
\label{tab:params1}
\end{table}

\subsection{Additional experiments}
\paragraph{Transformer layer.} The most general formulation of the Transformer layer we consider writes, for $\mathbb{X} \in \R^{L \times d}$,
\begin{align}   \label{eq:true-transformer}
\begin{split}
    \tilde{\mathbb{X}} &= \mathrm{concat}(r, \mathbb{X}) \\
    \hat{\mathbb{X}} &= \tilde{\mathbb{X}} + \sum_{h=1}^H \mathrm{softmax}\bigg(\frac{1}{\sqrt{p}} \underbrace{\mathrm{LN}(\tilde{\mathbb{X}}) Q_h}_{(L+1)\times p} \underbrace{\vphantom{q}K_h^\top \mathrm{LN}(\tilde{\mathbb{X}})^\top}_{p\times (L+1)} \bigg) \underbrace{\mathrm{LN}(\tilde{\mathbb{X}})V_h}_{(L+1) \times p} \underbrace{O_h^\top}_{p \times d} \, , \\
    T(\mathbb{X}) &= \hat{\mathbb{X}} + \mathrm{ReLU}(\hat{\mathbb{X}} W_1^\top  + \mathbf{1} b_1^\top) W_2^\top + \mathbf{1}  b_2^\top \, ,
\end{split}
\end{align}
where
\begin{itemize}
    \item $\mathrm{concat}(r, \mathbb{X}) \in \R^{(L+1) \times d}$ adds a new token at the beginning of the sequence by concatenating $r \in \R^d$ to $X \in \R^{L \times d}$. This token corresponds to the [CLS] or register token (see Section \ref{sec:solving} for discussion and references). In all our experiments, $r \in \R^d$ is a vector with i.i.d.~Gaussian entries of variance $1/d$, which is not trained;
    \item $\mathrm{LN}$ denotes layer normalization, $\mathrm{softmax}$ denotes row-wise softmax, and $\mathbf{1} \in \R^{L+1}$ is the vector filled with $1$;
    \item the parameters are $Q_h, K_h, V_h, O_h \in \R^{d \times p}$, $W_1 \in \R^{d \times m}$, $b_1 \in \R^m$, $W_2 \in \R^{m \times d}$, and $b_2 \in \R^d$.
\end{itemize}
\paragraph{Experiment with single-head Transformer layer on single-location regression.} We first consider the case of single-head attention, where $H=1$ and $p=d$. For ease of notation, we drop the subscripts~$h$ in the parameters of the attention layer. We also set $O$ to be the identity matrix. We aim at training the Transformer layer on the single-location regression task, to check that our simplified model is a good description of the Transformer layer. First note that the output of the Transformer layer~\eqref{eq:true-transformer} is a matrix in $\R^{(L+1) \times d}$ while the target of single-location regression is a scalar. Thus, we consider only the first row of $T(\mathbb{X})$, corresponding to the register token, and learn a linear projection of this row to $\mathbb{R}$. In other words, the Transformer layer should learn to store in the register token global information about the sequence, as described in Sections \ref{sec:regressiontask} and \ref{sec:solving}. Overall, letting $\theta \in \R^d$, our risk writes
\begin{align*}
\cR(Q, K, V, W_1, b_1, W_2, b_2, \theta) &= \Esp \Big[ \big(Y - T(\mathbb{X})_1 \theta \big)^2 \Big] \, ,
\end{align*}
where $(\mathbb{X}, Y)$ are distributed according to the single-location task as described in Section \ref{sec:regressiontask}. We train using single-pass stochastic gradient descent (meaning that fresh samples are used at each step), for $8,000$ steps with a batch size of $128$ and a learning rate of $0.01$. The experiment is repeated $20$ times with independent random initializations, and $95\%$ percentile intervals are plotted (but are not visible when the variance is too small). Parameters $K$, $V$, $W_1$, $W_2$ are initialized with Gaussian entries of variance $2/(d_\textnormal{in} + d_\textnormal{out})$. The bias terms are initialized to $0$, as well as the query matrix $Q$, following a standard recommendation in the literature on signal propagation in Transformer \citep{yang2021tuning,he2023deep,he2024simplifying}. The output weights $\theta$ are initialized with Gaussian entries of variance $1/d^2$, following the mean-field regime \citep{chizat2019lazy}. 
Parameters are $L=10$, $d=p=80$, $m=200$, $\varepsilon^2=0.01$, $\gamma^2=0.5$.

\begin{figure}[ht]
    \begin{subfigure}[t]{0.49\textwidth}
    \centering
    \includegraphics[width=0.93\textwidth]{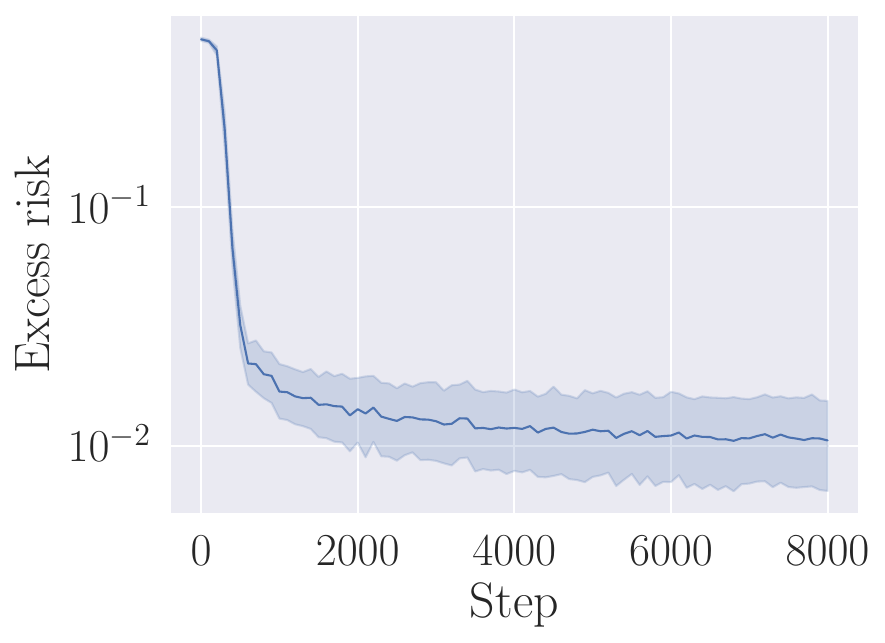}
    \caption{Excess test risk as a function of the number of steps.}
    \label{fig:single-head-left}
    \end{subfigure}
     \hfill
         \begin{subfigure}[t]{0.49\textwidth}
     \centering
    \includegraphics[width=\textwidth]{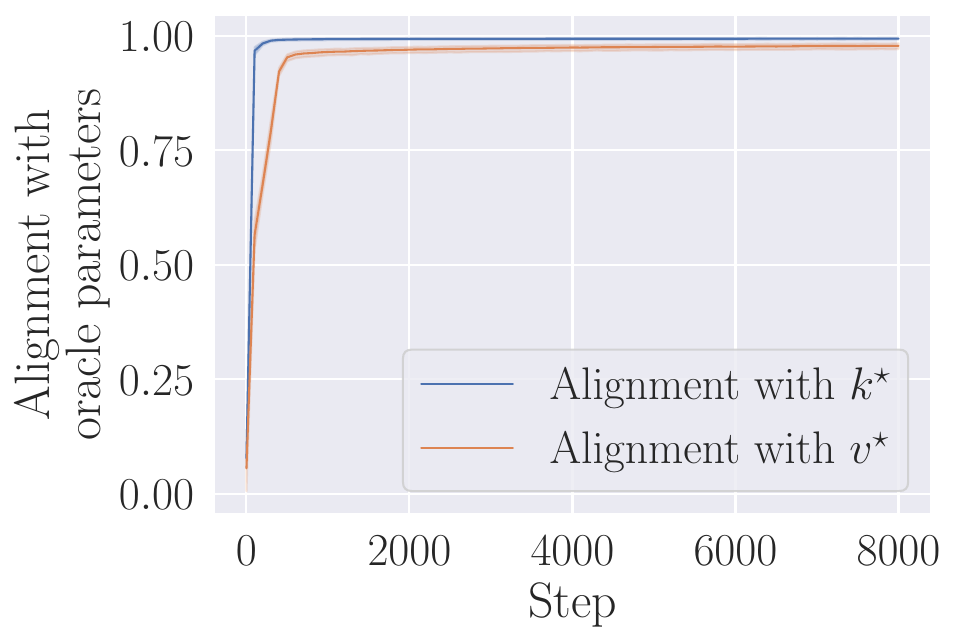}
    \caption{Alignment between Transformer parameters and oracle parameters $k^\star$ and $v^\star$. We plot $|k^\top k^\star|$ and $v^\top v^\star$ as a function of the number of steps, where $k$ is the first left singular vector of $K$, and $v := V (I + W_1 W_2) \theta / \|V (I + W_1 W_2) \theta\|$.}
    \label{fig:single-head-right}
    \end{subfigure}
    \caption{Training a full Transformer layer on single-location regression. The Transformer layer solves the task, and encodes the structure of the problem in its weights.}
    \label{fig:single-head}
\end{figure}

Results are given in Figure \ref{fig:single-head}. We observe in Figure \ref{fig:single-head-left} that the Transformer layer is able to solve single-layer regression. Furthermore, as shown by Figure \ref{fig:single-head-right}, it does so by encoding in its weights the underlying structure of the problem, namely the oracle parameters $k^\star$ and $v^\star$, as in our simplified model (see Section \ref{sec:optim}). More precisely, in the case of our model, we showed that the two parameters $k, v \in (\R^d)^2$ converge to $(k^\star, v^\star)$. To make appear the equivalent of $k$ and~$v$ in the more complex parametrization~\eqref{eq:true-transformer}, we let $k$ be the first left singular vector of $K$, and $v = V (I + W_1 W_2) \theta / \|V (I + W_1 W_2) \theta\|$. We check numerically that the weight matrix $QK^\top$ is nearly rank-one after training\footnote{The ratio between its first and second singular value is of the order of $10^6$ at the end of training.}, which validates taking $k$ as the first singular vector of $K$ in the present experiment. It also validates considering vector-valued parameters in our simplified model. The role of the vector~$k$ is to select the relevant token among all input tokens, while the vector $v$ describes how successive transformations (the value matrix of the attention layer, the MLP with skip connection, and the final linear projection) map this token to the output of the model. We observe that these two vectors align perfectly with $k^\star$ and $v^\star$. This confirms that our simplified model is a good description of how the Transformer layer solves single-location regression.

\paragraph{Multiple-location regression.} A natural extension of single-location regression is when the output depends on $s>1$ tokens instead of just one. This task, which we name multiple-location regression, can be written as
\begin{align}   \label{eq:multiple-location-reg}
   Y &= \sum_{h=1}^s X_{J(h)}^\top v_h^\star + \xi,  
\end{align}
where $J(1), \dots, J(s)$ are latent discrete random variables on $\{1, \hdots, L\}$, all different, and such that,
conditionally on $J(1), \dots, J(s)$,
\[
\left\{
\begin{array}{lll}
    X_{J(h)} &\sim&\cN\Big(\sqrt{\frac{d}{2}}k_h^\star, \gamma^2 I_d\Big)   \\
    X_\ell &\sim& \cN(0, I_d) \quad \textnormal{for} \quad \ell \notin \{J(1), \dots, J(s)\} \, .
\end{array}
\right.
\]
\paragraph{Experiment with simplified predictor on multiple-location regression.} In accordance with the above, a natural extension of the model presented in the main text is the multi-head predictor 
\begin{align}   \label{eq:multihead-predictor}
T_\lambda^{(k_1,v_1, \dots, k_h, v_h)}(\mathbb{X}) = \sum_{h=1}^s \erf\big(\lambda \mathbb{X} k_h \big)^\top \mathbb{X}v_h \, . 
\end{align}
The hope is that each head $(k_h, v_h)$ should align with one of the oracle directions $(k_h^\star, v_h^\star)$. As a first attempt in investigating this question, we run stochastic PGD in a setup similar to the one presented in Figure~\ref{fig:experiment-sgd}. We take $s=2$, the pair $(J(1), J(2))$ takes uniform values among disjoint pairs of indices in $\{1, \dots, L\}$. The directions $(k_1^\star, v_1^\star)$ and $(k_2^\star, v_2^\star)$ are sampled independently uniformly on the sphere, such that $(k_i^\star)^\top v_i^\star = 0$. Parameter values are the same as in Figure~\ref{fig:experiment-sgd}, except that the number of steps is set to $10^5$, the number of repetitions is set to $20$, and the inverse temperature $\lambda_t$ is constant after $2.5\cdot10^4$ steps. Results are given in Figure \ref{fig:simple-multihead}. We observe (Figure \ref{fig:simple-multihead-left}) that our predictor is able to solve the task. However, the recovery of oracle parameters is only partial, as shown in Figures \ref{fig:simple-multihead-middle} and~\ref{fig:simple-multihead-right}: each head partially aligns with the oracle parameters, but the alignment is not perfect. In other words, the model is not well able to separate the signal coming from the different $X_{P(h)}$. This calls for additional research in understanding how attention heads differentiate from each other in order to attend to various signals, and why in our setup the heads are not well-separated.

\begin{figure}[ht]
    \begin{subfigure}[t]{0.32\textwidth}
    \centering
    \includegraphics[width=\textwidth]{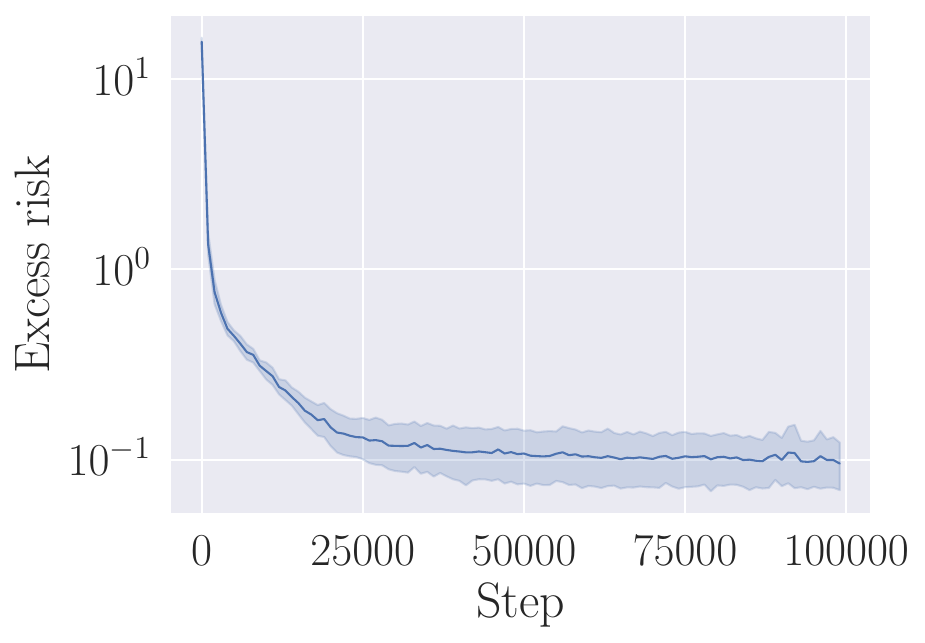}
    \caption{Excess test risk as a function of the number of steps.}
    \label{fig:simple-multihead-left}
    \end{subfigure}
     \hfill
    \begin{subfigure}[t]{0.32\textwidth}
     \centering
    \includegraphics[width=\textwidth]{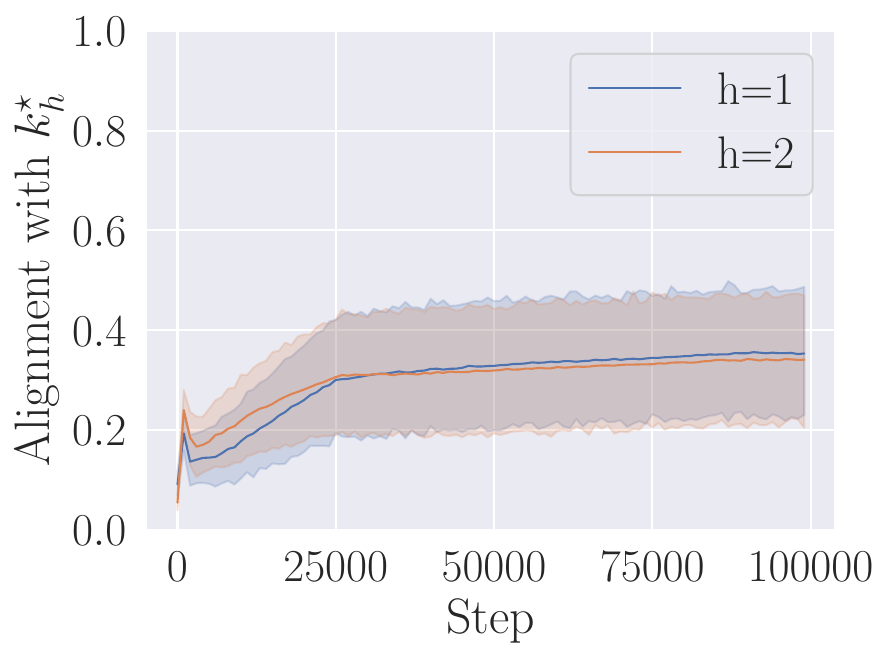}
    \caption{Alignment $|k^\top k_h^\star|$ with the oracle parameters.}
    \label{fig:simple-multihead-middle}
    \end{subfigure}
    \hfill
    \begin{subfigure}[t]{0.32\textwidth}
     \centering
    \includegraphics[width=\textwidth]{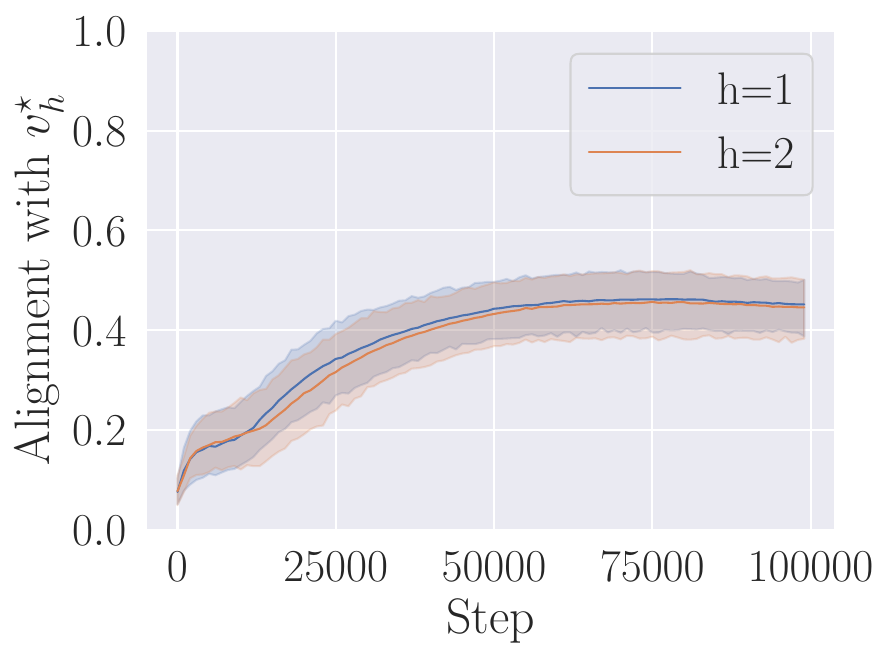}
    \caption{Alignment $|v^\top v_h^\star|$ with the oracle parameters.}
    \label{fig:simple-multihead-right}
    \end{subfigure}
    \caption{Training the multi-head predictor \eqref{eq:multihead-predictor} on the multiple-location regression task \eqref{eq:multiple-location-reg}. The predictor is able to reach a low-risk region. The recovery of oracle parameters by the predictor is partial. In the middle plot, for each repetition and each oracle parameter $k_h^\star$, we look at the end of training which head among $k_1$ and $k_2$ is closer to $k_h^\star$, and report the alignment between $k_h^\star$ and that head along training. If the alignment were perfect, this quantity would be close to $1$. The same holds for the right plot.}
    \label{fig:simple-multihead}
\end{figure}

\paragraph{Experiment with multi-head Transformer layer on multiple-location regression.} We train a multi-head Transformer layer on the multiple-location regression task \eqref{eq:multiple-location-reg}, taking $H=s=2$. The data is generated as in the previous experiment. Parameters are as in the experiment for single-head Transformer, except the dimension $p=d/H=40$, the number of repetitions set to $10$, and the learning rate set to $0.02$. Mimicking the single-head experiment, we let 
$k_h$ be the first left singular vector of $K_h$, and $v_h = V_h O_h^\top (I + W_1 W_2) \theta / \|V_h O_h^\top (I + W_1 W_2) \theta\|$. We also check numerically that all weight matrices $Q_h K_h^\top$ are nearly rank-one after training. Results are reported in Figure \ref{fig:transformer-multihead}. The conclusions are similar to the previous experiment: the excess risk is low at the end of training, but we observe partial recovery of the oracle parameters (although the recovery is somewhat better than with the simplified predictor, especially for $k_h^\star$). This suggests that our simplified predictor might be a first good testbed to understand the training dynamics of multi-head Transformer for this task.

\begin{figure}[ht]
    \begin{subfigure}[t]{0.32\textwidth}
    \centering
    \includegraphics[width=\textwidth]{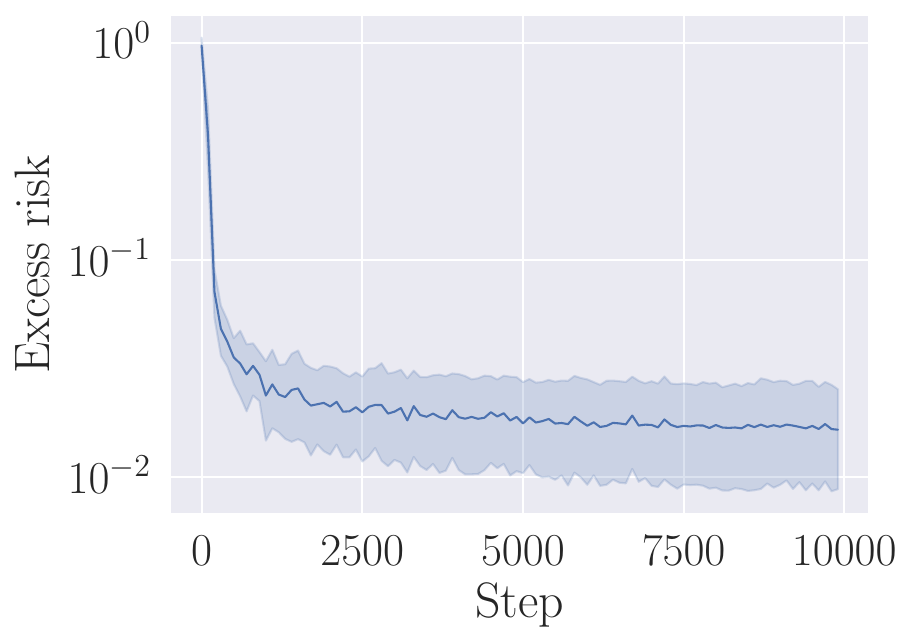}
    \caption{Excess test risk as a function of the number of steps.}
    \label{fig:transformer-multihead-left}
    \end{subfigure}
     \hfill
    \begin{subfigure}[t]{0.32\textwidth}
     \centering
    \includegraphics[width=\textwidth]{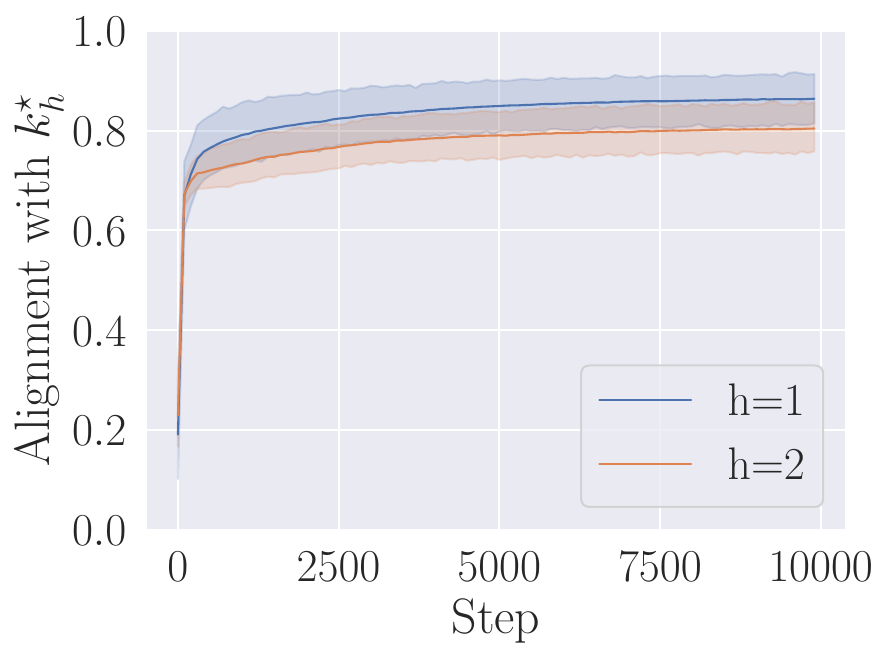}
    \caption{Alignment $|k^\top k_h^\star|$ with the oracle parameters.}
    \label{fig:transformer-multihead-middle}
    \end{subfigure}
    \hfill
    \begin{subfigure}[t]{0.32\textwidth}
     \centering
    \includegraphics[width=\textwidth]{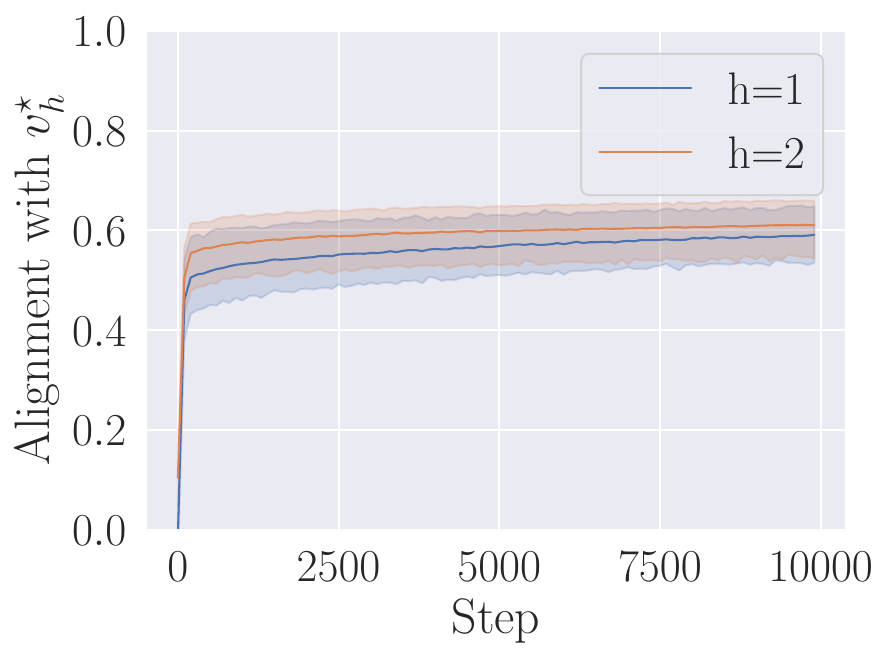}
    \caption{Alignment $v^\top v_h^\star$ with the oracle parameters.}
    \label{fig:transformer-multihead-right}
    \end{subfigure}
    \caption{Training the multi-head Transformer layer \eqref{eq:true-transformer} on the multiple-location regression task~\eqref{eq:multiple-location-reg}. The predictor is able to reach a low-risk region. The recovery of oracle parameters by the predictor is partial. For each $h \in \{1, 2\}$, we let $k_h$ be the first left singular vector of $K_h$, and $v_h = V_h O_h^\top (I + W_1 W_2) \theta / \|V_h O_h^\top (I + W_1 W_2) \theta\|$. In the middle plot, for each repetition and each oracle parameter $k_h^\star$, we look at the end of training which head among $k_1$ and $k_2$ is closer to $k_h^\star$, and report the alignment between $k_h^\star$ and that head along training. If the alignment were perfect, this quantity would be close to $1$. The same holds for the right plot.}
    \label{fig:transformer-multihead}
\end{figure}

\section{Further discussion of related models}   \label{app:related-works}
We begin by discussing some related works on training dynamics of Transformers \citep{jelassi2022vision,nichani2024how,wang2024transformers}, to illustrate the originality of our task and predictor.
\citet{jelassi2022vision} study how (vision) Transformers learn spatial patterns in the data by relying on positional encodings. This differs significantly from our task that is invariant by token permutation. Further, in their model, the argument of softmax (i.e., a matrix $A \in \R^{L \times L}$) is directly a parameter of the model. This is a radically different structure from the usual attention, and from our setup, where the data appear in the nonlinearity $\sigma(X_\ell^\top k)$.
Next, \citet{nichani2024how} explore a task involving a fixed latent causal graph over the positions of the tokens. Here again, positional encodings play a critical role in their analysis, whereas our task is invariant under permutations of the tokens. Moreover, in \citet{nichani2024how}, the output is expressed as a function of the last token, with the previous tokens providing the necessary context for this computation. In our setup, however, the output depends on a token whose position varies and must be identified within the context.
Closer to our approach is the recent paper by \citet{wang2024transformers}, which also incorporates a notion of token-wise sparsity: the output is computed as the average of a small subset of tokens, where the subset is identified by comparing the positional encodings of each token with that of a reference token. We outline two key differences with our setting. First, we do not make use of a reference token, but instead learn the latent direction $k^\star$ to identify the informative token. Second, in our setting, the tokens also encode an output projection direction $v^\star$ on top of $k^\star$. In other words, our task involves learning a linear regression in addition to identifying the relevant token, which is not the case in \citet{wang2024transformers}.

Besides, we also note that our task shares similarities with multi-index models \citep{mccullagh83} and mixtures of linear regressions \citep{deveaux1989mixtures}. %
However, our task $\eqref{pb:learning_pb}$ has a more structured nature, involving sequence-valued inputs and incorporating a single-location pattern.

Finally, one could imagine a multi-layer perceptron (MLP) designed specifically for single-location regression, where the weights have a diagonal structure with respect to the sequence index, namely
\[
\textnormal{MLP}(X_1, \dots, X_L) = \sum_{\ell=1}^L W_2 \sigma(W_1 X_\ell + b_1) + b_2.
\]
In such a setup, the first layer could learn the projections along $k^*$ and $v^*$, while the subsequent layer could learn to map these projections to the ouput $Y$ (in a somewhat similar spirit to multi-index models). However, this architecture is far from resembling those used in practice. If we do not assume a diagonal structure and instead use traditional MLPs, the number of parameters must scale at least linearly with the sequence length, which is highly suboptimal and may lead to very slow training. This highlights the efficiency of attention layers, which perform single-location regression with a fixed number of learnable parameters, independent of the input length. We leave a rigorous study of the learning abilities of MLPs in single-location regression for future work.

\end{document}